\documentclass[10pt]{article}
\pdfoutput=1

\usepackage[table,xcdraw]{xcolor}  %
\usepackage{color}

\usepackage[margin=1in,includefoot,footskip=30pt]{geometry}
\usepackage[width=474.18663pt]{caption}

\usepackage[T1]{fontenc}
\usepackage[utf8]{inputenc}
\usepackage{lmodern}

\usepackage{graphicx}
\usepackage{subcaption}
\usepackage{tikz}
\usetikzlibrary{arrows, positioning, shapes.geometric}

\usepackage{csquotes}
\usepackage{lipsum}
\usepackage{framed}
\usepackage[most]{tcolorbox}       %
\tcbset{
  myquote/.style={
    colback=gray!10,      %
    colframe=gray!50,     %
    width=\linewidth,     %
    sharp corners,        %
    boxrule=0.5pt,        %
    left=2em,             %
    right=2em,            %
    top=1em,              %
    bottom=1em,           %
    enhanced jigsaw,      %
  }
}

\usepackage{url}
\usepackage{hyperref}
\usepackage[numbers]{natbib}

\usepackage{booktabs}
\usepackage{multirow}
\usepackage{tabulary}
\usepackage{tabularx}

\usepackage{soul}
\usepackage{xspace}

\usepackage{algorithm}
\usepackage{algorithmicx}
\usepackage[noend]{algpseudocode}

\usepackage{amsmath,amssymb,amsfonts,amsthm}
\usepackage{thmtools,thm-restate}
\usepackage{mathtools}
\usepackage{nicefrac}
\usepackage{bm}
\usepackage{bbm}
\usepackage{dsfont}
\usepackage{upgreek}

\usepackage{enumitem}
\usepackage{array}

\usepackage{float}
\usepackage[bottom]{footmisc}
\usepackage{authblk}
\usepackage{microtype}
\usepackage{multicol}

\def\arxiv{1}
\def\litearxiv{0}
\definecolor{lightergray}{rgb}{0.9, 0.9, 0.9}
\definecolor{evenlightergray}{rgb}{0.95, 0.95, 0.95}
\definecolor{Darkblue}{rgb}{0,0,0.4}
\definecolor{Brown}{cmyk}{0,0.81,1.,0.60}
\definecolor{Purple}{cmyk}{0.45,0.86,0,0}
\hypersetup{colorlinks=true,
            citebordercolor={.6 .6 .6},
            linkbordercolor={.6 .6 .6},
            citecolor=Darkblue,
            urlcolor=blue,
            linkcolor=black}

\newcommand\blfootnote[1]{%
  \begingroup
  \renewcommand\thefootnote{}\footnote{#1}%
  \addtocounter{footnote}{-1}%
  \endgroup
}
\counterwithin{figure}{section}

\let\citet\cite

\newcommand{\sattarget}{\tsv{\sat}{\text{target}}}
\newcommand{\satimp}{\tsv{\sat}{\text{imp}}}
\newcommand{\satrange}{\bsat}
\newcommand{\maxarm}{K}

\newcommand{\reward}{R}
\newcommand{\effect}{E}
\newcommand{\engagement}{f}

\newcommand{\variancebound}{K}

\newcommand{\sdep}{\tilde{f}}

\def\arms{\mathcal{I}}
\def\bsat{\mathbf{x}}
\def\sat{x}

\theoremstyle{plain}
\newtheorem{theorem}{Theorem}[section]
\newtheorem{lemma}[theorem]{Lemma}

\newtheorem{proposition}[theorem]{Proposition}
\newtheorem{fact}[theorem]{Fact}

\theoremstyle{definition}
\newtheorem{definition}[theorem]{Definition}
\newtheorem{example}[theorem]{Example}

\theoremstyle{remark}

\renewcommand{\epsilon}{\varepsilon}
\renewcommand{\hat}{\widehat}
\renewcommand{\tilde}{\widetilde}

\newcommand{\under}[1]{\underline{#1}}
\newcommand{\ov}[1]{\overline{#1}}

\newcommand{\para}[1]{\left(#1\right)}
\newcommand{\parantheses}[1]{\left(#1\right)}

\newcommand{\curlybrackets}[1]{\left\{#1\right\}}

\newcommand{\bset}[1]{\curlybrackets{#1}}
\newcommand{\bsetf}[1]{\{#1\}}

\newcommand{\abs}[1]{\left|#1\right|}
\newcommand{\setsize}[1]{\left| #1 \right|}

\DeclareMathOperator*{\Exp}{\mathbb{E}}
\newcommand{\E}{\Exp}
\newcommand{\EE}[1]{\Exp\left[#1\right]}

\newcommand{\EEs}[2]{\Exp_{#1}\left[#2\right]}

\newcommand{\EEc}[2]{\Exp\left[#1\left|#2\right.\right]}

\renewcommand{\cite}[1]{\citep{#1}}

\newcommand{\ceil}[1]{\left\lceil#1\right\rceil}
\newcommand{\floor}[1]{\left\lfloor#1\right\rfloor}

\newcommand{\bigO}[1]{O\parantheses{#1}}

\newcommand{\R}{\mathbb{R}}

\newcommand{\reals}{\mathbb{R}}

\newcommand{\naturals}{{\mathbb N}}
\newcommand{\integers}{\mathbb{Z}}

\DeclareMathOperator*{\argmax}{arg\,max}

\DeclareMathOperator*{\sgn}{sgn}

\newcommand{\assignequals}{\coloneqq}

\newcommand{\asseq}{\coloneqq}

\newcommand{\loss}{\ell}

\newcommand{\ite}[2]{#1^{\parantheses{#2}}}
\newcommand{\tsv}[2]{#1_{#2}}

\newcommand{\cA}{\mathcal{A}}

\newcommand{\cI}{\mathcal{I}}

\newcommand{\cP}{\mathcal{P}}

\newcommand{\cR}{\mathcal{R}}

\newcommand{\cS}{\mathcal{S}}

\newcommand{\bi}{{\mathbf{i}}}

\usepackage[capitalize,noabbrev]{cleveref}
\Crefname{fact}{Fact}{Facts}

\title{\textbf{Algorithmic Content Selection and the Impact of User Disengagement}}
\author[1]{Emilio Calvano}

\author[2]{Nika Haghtalab}

\author[3]{Ellen Vitercik}

\author[2]{Eric Zhao}

\affil[1]{University of Rome}
\affil[2]{Stanford University}
\affil[3]{University of California, Berkeley}

\date{}  %

\begin{document}
\maketitle

\addtocontents{toc}{\protect\setcounter{tocdepth}{-1}} 

\begin{abstract}
Digital services face a fundamental trade-off in content selection: they must balance the immediate revenue gained from high-reward content against the long-term benefits of maintaining user engagement. Traditional multi-armed bandit models assume that users remain perpetually engaged, failing to capture the possibility that users may disengage when dissatisfied, thereby reducing future revenue potential.

In this work, we introduce a model for the content selection problem that explicitly accounts for variable user engagement and disengagement. In our framework, content that maximizes immediate reward is not necessarily optimal in terms of fostering sustained user engagement. 

Our contributions are twofold. First, we develop computational and statistical methods for offline optimization and online learning of content selection policies. For users whose engagement patterns are defined by $k$ distinct levels, we design a dynamic programming algorithm that computes the exact optimal policy in $O(k^2)$ time. Moreover, we derive no-regret learning guarantees for an online learning setting in which the platform serves a series of users with unknown and potentially adversarial engagement patterns.

Second, we introduce the concept of \emph{modified demand elasticity} which captures how small changes in a user's overall satisfaction affect the platform's ability to secure long-term revenue. This notion generalizes classical demand elasticity by incorporating the dynamics of user re-engagement, thereby revealing key insights into the interplay between engagement and revenue. Notably, our analysis uncovers a counterintuitive phenomenon: although higher friction (i.e., a reduced likelihood of re-engagement) typically lowers overall revenue, it can simultaneously lead to higher user engagement under optimal content selection policies.
\end{abstract}
\blfootnote{Correspondence to:  \href{mailto:eric.zh@berkeley.edu}{eric.zh@berkeley.edu}.}

\setcounter{tocdepth}{2}
\addtocontents{toc}{\protect\setcounter{tocdepth}{2}} %

\section{Introduction}
\label{sec:intro}
Content selection is an important problem for online platforms, including social media apps that employ recommendation algorithms.
In the content selection problem, a system (the app) repeatedly selects pieces of content to provide to a user. The content is chosen from a diverse pool with the goal of maximizing long-term payoff, such as cumulative ad revenue.
Classical models, such as multi-armed bandits (MAB), abstract each piece of content as an arm and each user interaction as a reward-generating pull.
These models, however, assume that user engagement is fixed: given a captive user, the only objective is revenue maximization.
In reality, user engagement is dynamic and deeply influenced by the user's previous experiences with the app and patterns of engagement. For instance, if previous interactions have led to dissatisfaction, a user may open an app less frequently.\footnote{The meaning of `reduced engagement' depends on the application. It could be defined as uninstalling the app, becoming inactive, or merely closing an active session. These are equivalent for our purposes.}

User engagement, despite its critical role, has often been overlooked in algorithmic models of content selection, largely because incorporating engagement introduces intricate trade-offs between competing objectives---namely, revenue and user satisfaction---that traditional online and multi-objective optimization frameworks are not equipped to handle.
Moreover, user engagement is inherently \emph{stochastic}, with future interactions heavily shaped by previous interactions: if a user becomes dissatisfied with an app, the app has fewer opportunities to rebuild engagement.
Prior research on variable engagement in user interactions
~\cite{ben-porat_modeling_2022,cao_fatigue-aware_2020,yang_exploration_2022} relies on simplifying assumptions to bypass these challenges. These include assuming a direct alignment between engagement and revenue maximization, modeling users as only capable of permanent disengagement, or allowing platforms to retain influence over disengaged users.
Without such assumptions, cultivating user engagement is a \emph{stateful, noisy,} and \emph{long-horizon} problem---the type of adaptive decision process that is typically approached with reinforcement learning.

In this paper, we directly tackle these complexities by 
    \emph{introducing a content selection model that captures the trade-offs between engagement and revenue, the inherent stochasticity of user behavior, and the necessity of long-term planning.}
    This model allows us to examine a fundamental question:
\begin{quote} \emph{How should firms account for variable user engagement rates in their decision-making?}
\end{quote}
In our model, every time a user engages with the app, the app selects content to display, generating a pair of instantaneous stochastic rewards---one for revenue and one for user experience---that are not necessarily aligned. The user's overall satisfaction with the app naturally represents the system's \emph{state}.
When disengaged, the app cannot interact with the user, display content, generate revenue, or influence this state. 

The probability a user engages at a given time step naturally depends on their satisfaction with prior experiences and inherent inertia.
For users who are currently engaged, the likelihood of continued engagement follows a function $f$ of the current state---called the user \emph{engagement pattern} or \emph{demand function}, reflecting its economic roots.
For disengaged users, a multiplicative factor capturing ``friction'' reduces the probability of re-engagement.
Thus, once disengaged, only the user's state prior to disengagement---together with chance---impacts their future decision to return. 

These modeling choices highlight key aspects of content selection overlooked by existing models. Namely, our model captures the trade-off between investing in user experience and generating revenue by accommodating scenarios in which user experience is inversely correlated with instantaneous revenue gains.
Moreover, our model accounts for key real-world dynamics, including the platform's inability to influence disengaged users, the possibility of user re-engagement, and the role of \emph{friction} in making re-engagement more challenging.

Using this model, we make the following technical contributions:

\paragraph{Efficient algorithms for optimal policy computation and  learning.}
We characterize structural properties of the optimal content selection policy. When the user's demand function is  $k$-piecewise constant (reflecting 
$k$ distinct engagement probabilities and archetypes), this leads to a dynamic programming algorithm that computes the exact optimal policy in $O(k^2)$ time (\cref{theorem:runtime}).

We also consider a learning variant of our model, where a sequence of $N$ users---potentially selected adversarially---arrives with unknown demand functions.
The app selects a content selection policy for each user and observes the realized cumulative revenue.
We present an algorithm with a sublinear regret of $O(\sqrt N)$ in environments with limited content (\cref{theorem:trionlinelearning}) and a $\tfrac 12$-approximate regret of $O(\sqrt N)$ (i.e., regret to a competitive ratio) in general (\cref{theorem:onlinelearning}).
These results rely on a mild linearity condition on the content landscape: each piece of content exhibits a (potentially negative) linear relationship between its expected revenue and its effect on user satisfaction.

The key insight enabling these efficient algorithms is that the content selection problem with variable user engagement can be reformulated as one with a captive audience, where the discount rate dynamically adjusts based on the user's state (\cref{theorem:cputime}). 
This allows us to express the trade-off between revenue and engagement as a trade-off between each interaction's reward and discount factor---a problem whose convexity restricts the space of optimal policies to a small, well-structured set (\cref{lemma:piecewiseconstant}).

\paragraph{Modified demand elasticity.}
Building on this insight, we introduce \emph{modified demand elasticity,} an adaptation of the classical notion of the price elasticity of demand~\cite{friedman, marshall2009principles}, tailored to our model. 
In classical economics, demand elasticity measures how sensitive consumer demand is to price changes, abstracting away external market forces such as consumer preferences and competition. Modified demand elasticity plays a similar role by capturing how a user's overall satisfaction with an app drives engagement and, by extension, long-term revenue. We show that modified demand elasticity fully captures the impact of our model's primitives---namely, the discount factor, content landscape, and engagement pattern---on an app's optimal content selection strategy.
 
\paragraph{Using modified demand elasticity to analyze friction and alignment.}
We further leverage modified demand elasticity to examine the role of \emph{friction} on user engagement and to determine conditions under which app creators are incentivized to invest more in user satisfaction.
Friction, which quantifies the diminished likelihood that a user will return once they have disengaged from the app,  naturally reduces an app's overall payoff. However, we show that, under optimal content selection policies, overall user-app engagement can paradoxically increase with friction (\cref{ex:first_example}).
This counterintuitive effect occurs because friction amplifies the app's perception of the user's modified demand elasticity, making user demand appear more sensitive to changes in satisfaction. When modified demand elasticity is high, even modest investments in enhancing user satisfaction yield substantial gains in retention, prompting the app to strategically choose content that drives higher overall satisfaction despite immediate revenue penalties.

\subsection{Related work}

\paragraph{Bandits and content selection.}
Recent works in bandit learning have also explored decision-making in settings that balance immediate rewards with future engagement.
\citet{ben-porat_modeling_2022,cao_fatigue-aware_2020}, and \citet{yang_exploration_2022} propose models where every arm pull yields a reward but also carries a probability of ending the episode, a structure designed to capture engagement-aware algorithms for recommendation systems.
These models differ from our work in two fundamental ways. First, they assume that user disengagement is permanent. Second, they enforce a direct alignment between revenue and engagement---either by defining revenue as the total number of user engagements or by positing that arms with higher rewards are less likely to cause disengagement.
In contrast, our model allows for user re-engagement and makes explicit the potential conflict between revenue and engagement maximization.
Other works have studied bandit settings with more general notions of state that can be affected by prior arm pulls~\citep[e.g.,][]{kleinberg_recharging_2018,laforgue_last_2022,Liu21:Rebounding}, but where the user's state affects future arm payoffs rather than engagement or horizon length.

Beyond bandit learning,
\citet{zhang_efficient_2022} study a decision-making problem where the sole objective is revenue maximization, but subject to a hard constraint that users maintain a positive expected future utility.
This constraint-based approach does not capture partial disengagement risk and assumes that users have perfect foresight about the app's intentions, making engagement independent of prior experiences. As a result, the app's previous actions have no bearing on its future engagement opportunities.
\citet{pacchiano_stochastic_2021} study a similar online learning framework with linear constraints.
In contrast to the models proposed by \citet{ben-porat_modeling_2022,cao_fatigue-aware_2020}, and \citet{yang_exploration_2022}, which assume direct alignment between engagement and revenue, the models of \citet{zhang_efficient_2022} and \citet{pacchiano_stochastic_2021} allow for misalignment between satisfying engagement constraints and reward maximization.

\paragraph{Engagement optimization.}
Optimizing user engagement is standard practice for online content platforms \cite{evans2018death,ban2021personalized} and, increasingly, firms in other sectors \cite{rafieian2023optimizing}.
A wide body of recent work has studied the downstream effects of engagement-driven optimization.
For instance, \citet{zhou2024higher} found that prioritizing revenue maximization can explicitly reduce engagement---highlighting the trade-offs our model aims to capture.
Similarly, \citet{kleinberg_challenge_2022} analyze how optimizing for user engagement, when misaligned with user utility, can negatively impact user welfare. 
Other studies explore the downstream effects of engagement optimization, including the spread of polarizing content \cite{munn2020angry,rathje2021out}---often recognized by firms themselves \cite{levy2020facebook,horwitz2020facebook}---as well as clickbait \cite{10.1145/3589334.3645353}.
Engagement optimization has also been studied in the contexts of information retrieval \cite{OBRIEN2022107109,o2008user,OBRIEN2010344}, marketing \cite{liu2022implications}, and mechanism design \cite{viljoen2021design}.
Finally, friction in user engagement is a well-documented phenomenon that can be interpreted as a consequence of brand loyalty \cite{krishnamurthi1991empirical} or as a discrete form of demand \cite{hanemann1984discrete}.

\section{Modeling Content Selection with Engagement-Revenue Tradeoffs}
\label{section:model}
We present our content selection model, building on the classical model where an app selects a piece of content at each timestep from a set of options, each generating a certain amount of revenue.

\begin{enumerate}
\item 
A user engages with the app at the first timestep, $t = 1$.

\item 
If the user interacts with the app at timestep $t$, the app selects a piece of content $i_t$ from a set $\cI$ to display.
Each content $i \in \cI$ generates revenue, represented by the random variable $R_i$, and provides a user experience reward, represented by the random variable $E_i$, both supported on $\reals$.
We denote realized revenue and user experience by $ r_t \sim R_{i_t}$ and $e_t \sim E_{i_t}$, respectively.

\item If the user does not interact with the app at timestep $t$, the app takes no action, earns no revenue, and does not influence the user; that is, $i_t = \emptyset$, $r_t = 0$, and $e_t = 0$.
The binary variable $s_t = \mathbbm{1}[\text{User Interacts}]$ indicates whether the user interacts with the app at timestep $t$.

\item The user decides whether to engage with the app in the next timestep based on a summary of their previous app experiences, defined as $x_{t+1} = \phi(e_1, \dots, e_t)$. For clarity, we set $\phi(e_1, \dots, e_t) = \sum_{\tau=1}^t e_\tau$, though our results also extend to other choices of $\phi$ (see the paragraph titled \emph{Generalizations}).
If the user is currently engaged $(s_t = 1)$, they continue using the app with probability $f(x_t)$, where $f$ is a monotonically non-decreasing function. That is, $s_{t} \sim \text{Bernoulli}(f(x_t))$.
If the user is not already engaged, the probability is scaled down by a constant factor $1 - c$ (with $c \in [0,1])$. That is, $s_{t} \sim \text{Bernoulli}((1 - c) \cdot f(x_t))$.
\end{enumerate}
We refer to $f$ as the user's \emph{demand function} and the constant $c$ as the \emph{friction parameter}, with larger $c$ corresponding to greater friction.
The summary $x_t$ of the user's interaction history is called the \emph{user's state}. Since $f(x_t)$ increases with $x_t$, a higher state indicates a stronger tendency to remain engaged.
For technical reasons, we assume the sets $\bset{\EE{E_i}}_{i \in \cI}$ and $\bset{\EE{R_i}}_{i \in \cI}$ are compact.

\paragraph{Content selection policy. }
A content selection policy $\pi$ maps a transcript of prior interactions, $H_t = \bset{(s_\tau, r_\tau, e_\tau, i_\tau)}_{\tau \in [t]}$, to a distribution over content. In other words, the app selects content for the next timestep according to $i_{t+1} \sim \pi(H_t)$.
The app's objective is to maximize long-term payoff
\begin{equation}
    J(\pi) \asseq \EEs{\bset{(s_t, r_t, e_t, i_t)}_{t}}{\sum_{t=1}^\infty \gamma^{t-1} r_t}, \label{eq:platformdiscountedobjective}
\end{equation}
where $\gamma \in (0, 1)$ is the \emph{discount factor}. This objective captures the time-discounted sum of revenues earned, with the expectation taken over the stochastic transcript.

Policies that rely on the entire interaction history are unwieldy.
Fortunately, it is sufficient for apps to use \emph{simple policies}, which select content at time $t$ solely based on the user's state $x_{t-1}$.
Formally, a policy $\pi$ is simple if there exists a function $g: \reals \to \arms$ such that for every transcript $H = \bset{(s_t, r_t, e_t, i_t)}_{t \in [T]}$, we have $\pi(H) = g(\sum_{\tau=1}^{t} \tsv{e}{\tau})$.
The following lemma establishes this fact.
\begin{restatable}{lemma}{simpleoptimal}
\label{lemma:simpleoptimal}
If there is an optimal policy for the app creator, there is also a simple optimal policy.
\end{restatable}
\noindent
We use $\ite{w}{t}$ to denote the value of a variable $w$ at the $t$-th user-app interaction. 
Notably, under any simple policy, the content selection $\ite{i}{t}$ depends only on the prior user state $\ite{x}{t-1}$ and does not otherwise depend on whether the user was engaged in the immediate past or their history of engagement.

\subsection{Discussion of Modeling Choices}
Our model naturally emerges from two key abstractions of user-platform interactions:
\begin{enumerate}[label=(\alph*)]
    \item The platform accrues revenue as a function of its content selections (arm pulls).
    \item A user's satisfaction with the platform, which governs their probability of continued engagement, accumulates over time as a function of the platform's content selections.
\end{enumerate}

Abstraction (a) is common in recommendation system models, where each arm pull yields direct monetary returns.
Our model extends this principle to user engagement (Abstraction (b)), where user satisfaction also accumulates over time.
This perspective---modeling engagement as the function of accumulated user experience---is motivated by prior work in operations research~\citep[e.g.,][]{baucells_satiation_2007} and recommendation systems~\citep[e.g.,][]{Liu21:Rebounding}. 
Related research on recommendation systems under participation constraints \cite{zhang_efficient_2022} and addiction \cite{kleinberg_challenge_2022} have similarly adopted this perspective. 

Additionally, the demand function $f$ in our model closely aligns with classical economic notions of demand in markets with prices.
Traditionally, consumer decision-making is modeled as a function mapping from the firm’s action (price) to the consumer’s action (amount to purchase).
Analogously, our model represents user decision-making as a demand function mapping satisfaction (determined by the platform) to user engagement probability; user satisfaction is analogous to the price set by a firm, while the user’s engagement probability corresponds to the quantity of goods purchased.
This analogy is especially clear in scenarios where, for instance, the platform is a grocery store app setting prices, and the user is a shopper deciding whether to return based on the prices they previously experienced---a scenario fully captured by our model.
More broadly, one can interpret the platform as offering an \emph{experience good}, where the user's state reflects their evolving perception of the good's value. 
From this perspective, our model is actually a simple extension of the classical model of the demand elasticity~\cite{marshall2013principles} that takes as its point of departure the modeling of demand as---rather than being a direct consequence of a firm's contemporaneous posted prices---evolving slowly as a function of the user's experience with the firm, including its prices, over a long time horizon.

\paragraph{Generalizations.}
Our model can be generalized in several ways.
First, user satisfaction can be aggregated differently---rather than the cumulative sum, it can be modeled as a discounted sum $\ite{x}{t+1} = \sum_{\tau=1}^t \gamma^{\tau} \ite{e}{\tau}$ or an average $\ite{x}{t+1} = \tfrac 1t \sum_{\tau=1}^t \ite{e}{\tau}$. For instance, if a user revisits a store based on the average price they have encountered at the business, this behavior naturally aligns with our model when satisfaction is defined as the (sign-flipped) mean of past prices.
Additionally, user experiences can be represented as multi-dimensional vectors, allowing for richer engagement dynamics. Specifically, we can define $\bset{E_i}_{i \in \arms}$ as vectors in $\reals^d$, with demand function $f: \reals^d \to [0, 1]$ that is monotonically non-decreasing in the product order of $\reals^d$.
For clarity, we present our results using the basic model, but our findings directly extend to these generalizations.

\paragraph{On friction.}
We model friction as a multiplicative factor
that adjusts engagement probabilities based on whether a user is currently engaged with the platform. Friction can vary significantly across platforms and depends on the environment in which the platform operates.
For example,
\begin{enumerate}
\item Friction may be rooted in the app design.
A native smartphone app capable of sending push notifications may experience lower friction than a web application without notification capabilities since the latter has limited ability to re-engage disengaged users.
\item Friction may be rooted in external factors, such as competition between apps. Between two competing services, if one is significantly more addictive, the less addictive one will experience higher friction, as users who switch to the more engaging platform are less likely to return. We explore this situation in more detail in \cref{ex:second_example}.
\item Friction may depend on the apps' usage patterns.
For example, a social network used primarily for work---where engagement is primarily driven by external obligations rather than personal desire---may be less susceptible to friction than a social network used for leisure, where engagement depends more on user preference. More generally, platforms with inelastic demand may experience lower friction than those with more elastic demand.
\end{enumerate}

While our basic model presents friction as a fixed multiplicative factor, our results readily generalize to other notions of friction, including additive models or cases where re-engagement probability depends on the time elapsed since the user was last on the platform. More broadly, friction can be represented as a family of demand functions $\bset{f_n}_{n \in \integers}$, where the probability of engagement is given by $f_n(x_t)$, conditioned on the user's last engagement occurring $n$ timesteps prior. All technical results we present directly extend to this generalization.

\subsection{From Variable Engagement to Variable Discount Rates}
\label{sec:var-rate}
To efficiently compute optimal policies, we make a key observation: from the app's perspective, content selection with variable user disengagement is equivalent to content selection without disengagement but with dynamically varying discount rates.
This equivalence arises because if a user does not interact with the app for $k$ timesteps, the app's discount rate effectively decreases from $\gamma$ to $\gamma^{k+1}$. Viewing the problem through this lens, we redefine the app creator's objective by allowing the discount factor to vary as a function $\tilde f$ of the user's state. Since the time between user-app interactions follows a geometric distribution, the following theorem leverages the moment-generating function to formally define these variable discount rates.
\if\arxiv0
The full proof is in \cref{section:cputimeproof}.
\fi

\begin{restatable}{theorem}{cputime}
\label{theorem:cputime}
For any simple policy $\pi$, the app creator's objective value can be written as $$J(\pi) = \EEs{\bset{(\ite{s}{t}, \ite{r}{t}, \ite{e}{t}, \ite{i}{t})}_t}{\sum_{t=1}^\infty \ite{r}{t} \prod_{\tau=2}^t \gamma \tilde f\left(\ite{x}{\tau}\right)},$$ 
where the variable discount rate $\gamma \tilde f(x)$ scales the original discount factor $\gamma$ by a modified demand function:
$$\tilde f(x) \asseq f(x) + \frac{1 - f(x)}{1 - \gamma (1 - (1-c) f(x))} (1-c) f(x) \gamma.$$
\end{restatable}

\if\arxiv1

\begin{proof}
    First, we define a function that maps each timestep to the number of user-app interactions that have occurred up to and including that timestep: $\mathsf{NumInteractions}(t) = \sum_{\tau=1}^t s_\tau$.
    We also use $\mathbb{T} = \bset{t \in \integers_+ \mid s_t = 1}$ to denote the set of timesteps where user-app interactions took place.
    Note that $\mathsf{NumInteractions}$ is bijective when restricted to $\mathbb{T}$.
    Since no revenue is generated on timesteps without user engagement (i.e. $r_t = 0$ when $t \not\in \mathbb{T}$), we can write \cref{eq:platformdiscountedobjective} as a sum over $\mathbb{T}$:
    \begin{align*}
        J(\pi) = \EEs{\bset{s_t, r_t, e_t, i_t}_t(\pi)}{\sum_{t=1}^\infty s_t \gamma^{t-1} r_t} = \EEs{\bset{s_t, r_t, e_t, i_t}_t(\pi)}{\sum_{t \in \mathbb{T}} \gamma^{t-1} r_t}.
    \end{align*}
    Since the mapping $\mathsf{NumInteractions}$ is bijective on $\mathbb{T}$, its inverse $\mathsf{NumInteractions}^{-1}$ is well-defined in the above series, and we have $\bset{\mathsf{NumInteractions}^{-1}(t) \mid t \in \mathbb{T}} = \integers_+$.    
    We can therefore re-index the sum over $\mathbb{T}$ in terms of the number of interactions rather than the original timesteps:
    \begin{align*}
        J(\pi) = \EEs{\bset{s_t, r_t, e_t, i_t}_t(\pi)}{\sum_{t \in \mathbb{T}} \gamma^{t-1} r_t} = \EEs{\bset{s_t, r_t, e_t, i_t}_t(\pi)}{\sum_{t = 1}^\infty \gamma^{\mathsf{NumInteractions}^{-1}(t) - 1} \ite{r}{t}},
    \end{align*}
    where $\ite{r}{t}$ denotes the realized revenue at the $t$-th interaction rather than the $t$-th timestep.
    By linearity, we can use telescoping to simplify further:
    \begin{align*}
        J(\pi) = \sum_{t = 1}^\infty \EEs{s_t, r_t, e_t, i_t}{\ite{r}{t} \gamma^{\mathsf{NumInteractions}^{-1}(t)-1}} = \sum_{t = 1}^\infty \EEs{s_t, r_t, e_t, i_t}{\ite{r}{t} \prod_{\tau=2}^t \gamma^{\ite{w}{\tau}} },
    \end{align*}
    where $\ite{w}{\tau} = \mathsf{NumInteractions}^{-1}(\tau) - \mathsf{NumInteractions}^{-1}(\tau - 1)$ represents the number of timesteps between the $(\tau-1)$-th and $\tau$-th ineteractions.
    Here, we have used the fact that the first timestep always corresponds to a user-app interaction, so $\mathsf{NumInteractions}^{-1}(1) = 1$.
    
    The random variable $\ite{w}{\tau}$ represents the amount of time that passes between the user's $(\tau-1)$th and $\tau$th interactions.
    Under a simple policy, $\ite{w}{\tau}$ is conditionally independent of $\ite{w}{\tau-1}$ given the user's state after the $(\tau - 1)$-th interaction, i.e. $\ite{x}{\tau }$.
    
    If there were no friction, $\ite{w}{\tau}$ would follow a standard geometric distribution with parameter $f(\ite{x}{\tau })$.
    Recall that a geometric distribution with parameter $p$ models the number of coin flips (with heads probability $p$) required to achieve the first head.
    When friction $c$ is present, $\ite{w}{\tau}$ follows a non-homogeneous geometric distribution with the probability mass function $\Pr(\ite{w}{\tau} = k) = (1-f_k(\ite x \tau)) \prod_{j=1}^{k-1} f_j(\ite x\tau)$ defined on the support $k \in \naturals$, where we define $f_1 = f$ and $f_n = c f$ for all $n > 1$.
    We denote this distribution by $\text{Geo}(p, c)$.
    One can compute the moment generating function of this non-homogeneous geometric distribution as
    $$\EEs{X \sim \text{Geo}(p, c)}{\exp(tX)} =  p \cdot \exp(t)  + \frac{1 - p}{1 -  \exp(t) (1 - (1-c) p)} (1-c) \cdot p \cdot \exp(t)^2$$
    for $t < - \ln(1 - (1-c)p)$. This result enables us to simplify the expectation $$\EEc{\gamma^{w_t}}{x_t} =   f(x_t) \gamma + \frac{1 - f(x_t)}{1 - \gamma (1 - (1-c) f(x_t))} (1-c) f(x_t) \gamma^2.$$
    Notice we can express this as $\EEc{\gamma^{w_t}}{x_t} = \gamma \cdot \tilde f(x_t)$.
    Finally, by the law of total expectation, we recover
    \begin{align*}
		  J(\pi) = \sum_{t = 1}^\infty \EEs{s_t, r_t, e_t, i_t}{\ite{r}{t} \prod_{\tau=2}^t \EEc{\gamma^{\ite{w}{\tau}}\, }{\, \ite{x}{\tau}} } = \EEs{\bset{s_t, r_t, e_t, i_t}_t}{\sum_{t = 1}^\infty  \ite{r}{t}{\prod_{\tau=2}^t \gamma \tilde f\left(\ite{x}{\tau}\right)}}.
    \end{align*}
\end{proof}
\else
\begin{proof}[Proof sketch]
    We re-index the app's objective to sum over the number of user-app interactions. Define $\mathsf{NumInteractions}(t) = \sum_{\tau=1}^t s_\tau$ as the number of interactions by time $t$ and $\ite{w}{\tau} = \mathsf{NumInteractions}^{-1}(\tau) - \mathsf{NumInteractions}^{-1}(\tau - 1)$ as the number of timesteps between interactions $\tau$ and $\tau-1$. Since revenue is only obtained when the user engages, we have $J(\pi) = \E\left[\gamma^{\mathsf{NumInteractions}^{-1}(t) - 1} \ite{r}{t}\right]$. Rewriting this sum, we can express the effective discount factor as $\gamma^{\ite{w}{\tau}}.$ We show that under a simple policy, $\ite{w}{\tau}$ follows a non-homogeneous geometric distribution whose moment generating function yields $\EEc{\gamma^{w_t}}{x_t} = \gamma \cdot \tilde f(x_t)$. The law of total expectation then implies the theorem.
\end{proof}
\fi

This theorem shows that incorporating a variable discount rate that reflects user state and friction simplifies the long-term revenue objective in a manner conducive to its optimization.

\section{Optimal Content Selection in Linear Settings}
\label{section:optimal}
Solving the content selection problem in our model involves tackling a long-term optimization problem. At each user-app interaction, the app must balance trade-offs between generating immediate revenue and fostering future engagement.
Moreover, user engagement is an adaptive stochastic process heavily influenced by content choices and user-app interactions.
These dependencies complicate optimizing and learning effective content selection policies.
Nonetheless, we show that if an app's content landscape exhibits a linear relationship between revenue and user experiences,
 the optimal content policy can be characterized in a surprisingly simple manner.

\paragraph{Linear setting.}
We say that a content decision problem is \emph{linear} if the following conditions hold:
\begin{itemize}
\item The set of available content is represented by an interval $\cI = [-K_1, K_2]$ for some $K_1, K_2 > 0$. We generally let $K \asseq K_1 = K_2$, though all results trivially extend to where $K_1 \neq K_2$.
\item Each content $i \in \cI$ has a deterministic impact on user experience, given by $E_i = C_E - i$, and generates deterministic revenue $R_i = C_R + i$ where $C_E \in [0, K)$ and $C_R \geq 0$ are constants representing drift.
\end{itemize}

Importantly, this assumption does not imply that the overall dynamics of the model are linear.
First, the user demand function can be arbitrarily complex. For instance, non-linear diminishing returns to increasing user satisfaction can be captured by a sigmoidal demand function $f$. Second, the linearity assumption aligns with classical settings such as when a platform sets direct prices for goods and users decide whether to visit the platform based on its historical prices.

We define a user's demand function to have \emph{complexity $k$} if it is piecewise constant with $k$ pieces. That is, there are breakpoints $b_1, \dots, b_{k} \in \reals$ such that $\engagement$ is constant on intervals $(-\infty, b_1), [b_1, b_2), \dots,$ $[b_{k-1}, b_{k}), [b_{k}, \infty)$.
In other words, the user's demand function has $k$ discrete levels, each corresponding to a distinct probability of engagement.

In linear settings, we show that under an optimal policy, the frequency of app-user interactions stabilizes quickly. Moreover, the app's investment in user engagement---and consequently its long-term behavior---is largely determined by the point at which user demand stabilizes, i.e., $f(x_\infty)$.
\begin{restatable}{lemma}{piecewiseconstant}
    \label{lemma:piecewiseconstant}
    In a linear setting where the user demand function has a complexity of $k < \infty$, there is a simple optimal policy that satisfies the following properties:
    \begin{itemize}[itemsep=0pt, topsep=0pt, left=2pt]
        \item The sequence of user states $x_1, x_2, \dots$ is monotonic.
        \item The limit $x_\infty = \lim_{t \to \infty} x_t$ exists and is either at a discontinuity of $f$ or negative infinity.
        \item If $x_\infty = -\infty$, the app always shows the highest-revenue content, i.e. $i = K$. 
        \item If $x_\infty$ is a discontinuity of $f$, the user state will reach $x_\infty$ within $k + n + 1$ interactions (i.e., $\tsv{x}{n+k+1} = x_\infty$), where $n$ is the minimal number of interactions required to reach $x_\infty$ under some policy (i.e.,
        there exists a policy where $\tsv{x}{n} = x_\infty$).
        Moreover, if $n=1$, then $x_2 = x_3 = \dots = x_\infty$.
    \end{itemize}
\end{restatable}
This result highlights the structured nature of optimal policies in linear settings, showing that the user state either stabilizes at a demand function discontinuity or diminishes indefinitely.
A key step in proving \cref{lemma:piecewiseconstant} is showing that an optimal policy either consistently increases or maximally decreases user state in a structured manner before reaching a discontinuity.
\begin{restatable}{lemma}{trajectory}
    \label{lemma:trajectory}
    Consider a simple optimal policy $\pi$.
    For any subsequence of the user state trajectory $\ite{x}1, \ite{x}2, \dots, \ite{x}T$, where $\ite{x}2, \dots, \ite{x}T$ are all not discontinuities of $f$, the following hold:
    \begin{enumerate}
        \item If the policy $\pi$ does not maximally increase the user state at the last step of the subsequence, i.e., $\pi(\ite{x}T) > -K$, then all previous actions in the subsequence must maximally decrease the user state, i.e., $\pi(\ite{x}1) =  \dots =  \pi(\ite{x}{T-1}) = K$.
        \item Conversely, if the policy $\pi$ does not maximally decrease the user state at the first step, i.e., $\pi(\ite{x}1) < K$, then all later actions in the subsequence must maximally increase the user state, i.e., $\pi(\ite{x}2) =\dots = \pi(\ite{x}T) = -K$.
    \end{enumerate}
\end{restatable}
This lemma builds on the intuitive observation that, if engagement probabilities were fixed, an app would have a stronger incentive to prioritize user satisfaction later rather than sooner since future revenue is discounted. The lemma proves that when engagement probabilities vary, this intuition still holds.
In particular, if the app does not find it beneficial to fully cater to the user in the final interaction before reaching a discontinuity, then it is suboptimal to refrain from fully exploiting the user earlier---even when accounting for the negative impact of user dissatisfaction on engagement.
Conversely, if the app chooses not to fully exploit the user initially, then there is no reason to avoid prioritizing user satisfaction later.

\subsection{Computing Optimal Content Selection Policies}
\begin{algorithm}[t]
\caption{Dynamic programming algorithm for computing app policies (\cref{theorem:runtime}).}
\label{alg:linear}
\begin{algorithmic}[1]
    \State \textbf{Input:} demand function $f$, set of discontinuities $D \subset \reals$ of $f$;
    \State \textbf{Initialize:} Dictionaries $\text{Trajs}, V$ whose keys are $D \cup \{0\}$;\\
    \Comment{\textcolor{gray}{$\text{Trajs}[x]$ and $V[x]$ are the optimal trajectory and payoff that can be realized with an initial user state of $x$, assuming user state always moves in the $\sgn(x)$ direction.}}
    \For{discontinuity $d \in D \cup \bset{0}$ where $d \leq 0$, in ascending order}\label{line:negative}
        \State Compute discount and payoff of decreasing user state from $d$ past $\min D$: \[\hat v, \hat \gamma, \text{Traj} \leftarrow \mathrm{GetPayoff}\left(d, d - (K - C_E) \cdot \big(1 + \floor{\tfrac{d - \min D}{K - C_E}}\big), f\right)\]\label{line:below_D}
        \State $V[d] = \hat v + \tfrac{\hat \gamma (C_R + K)}{1 - \gamma \tilde f(-\infty)}$, $\text{Trajs}[d] =  [K]^\infty$; \Comment{\textcolor{gray}{Evaluate pure exploitation, i.e. $\ite it = K$ for all $t$}}
        \For{$d' \in D$ where $d' < d$, in ascending order}
            \Comment{\textcolor{gray}{Evaluate decreasing $d$ to $d'$}}
            \State Compute potential payoff  $v' = \hat v + \hat \gamma \cdot V[d']$ with $\hat v, \hat \gamma, \text{Traj} \leftarrow \mathrm{GetPayoff}(d, d', f)$;\label{line:evaluate_d'} %
            \State If $v' > V[d]$, update $V[d] = v'$ and $\text{Trajs}[d] = \text{Traj} + \text{Trajs}[d']$ ($+$ is concatenation);
        \EndFor
        \State If $\tfrac{C_R + C_E}{1 - \gamma \tilde f(d)} > V[d]$, let $V[d] =  \tfrac{C_R + C_E}{1 - \gamma \tilde f(d)}$ and $\text{Trajs}[d] = [C_E]^\infty$; \Comment{\textcolor{gray}{Evaluate staying at $d$}}\label{line:stay}
    \EndFor
    \For{discontinuity $d \in D \cup \{0\}$ where $d \geq 0$, in descending order}\label{line:positive}
        \State If $d \neq 0$, let $V[d] = \tfrac{C_R + C_E}{1 - \gamma \tilde f(d)}$ and $\text{Trajs}[d] = [C_E]^\infty$; \label{line:pos_stay} \Comment{\textcolor{gray}{Evaluate staying at $d$}}
        \For{$d' \in D$ where $d' > d$, in descending order} \Comment{\textcolor{gray}{Evaluate increasing $d$ to $d'$}}
            \State Compute potential payoff $v' = \hat v + \hat \gamma \cdot V[d']$ with $\hat v, \hat \gamma, \text{Traj} \leftarrow \mathrm{GetPayoff}(d, d', f)$;
            \State If $v' > V[d]$, let $V[d] = v'$ and $\text{Trajs}[d] = \text{Traj} + \text{Trajs}[d']$; \label{line:pos_d'}
        \EndFor
    \EndFor
    \State Return $\text{Trajs}[0]$;
\end{algorithmic}
\end{algorithm}
\begin{algorithm}[t]
\caption{Algorithm $\mathrm{GetPayoff}(x, x^*, f)$.}
\label{alg:payoff}
\begin{algorithmic}[1]
\State \textbf{Input:}  starting point $x$, ending point $x^*$, and demand function $f$;
\State \textbf{Initialize:}  payoff $v_0 = 0$, discount $\gamma_0 = 1$, and position $x_0 = x$;
\For{$t \in [T']$ where $T' = \ceil{\frac{\abs{x^* - x}}{K + C_E}}$ is the number of timesteps to reach $x^*$ from $x$}
    \If{$x^* < x$}
        \State Take action $i_t = \min \bset{K, x_{t-1} - x^* + C_E}$;
    \ElsIf{$x^* > x$}
        \State Take action $i_t = \begin{cases}\delta & \delta > 0 \wedge t =1 \\ -K & \text{otherwise} \end{cases}$ where $\delta = (x^* - x) \mathrm{\;mod\;} (K + C_E)$; 
    \EndIf
    \State Update $v_t = v_{t-1} + \gamma_{t-1} \cdot (C_R + i_t)$, $x_t = x_{t-1} + C_E - i_t$, and $\gamma_t \leftarrow \gamma_{t-1} \cdot \gamma \tilde f(x_t)$;
\EndFor
\State \Return Payoff $v_{T'}$, discount $\gamma_{T'}$, and trajectory $i_1, \dots, i_{T'}$;
\end{algorithmic}
\end{algorithm}
\cref{alg:linear} uses dynamic programming to efficiently compute the optimal content selection policy by leveraging the structure identified in \cref{lemma:piecewiseconstant}.
This lemma establishes that under an optimal policy, user satisfaction moves monotonically until it reaches a discontinuity in the demand function, or it declines indefinitely.
Along this trajectory, the policy always either maximally shifts user satisfaction or adjusts the user state precisely to a discontinuity.

In order to determine the best trajectory for transitioning between states, \cref{alg:linear} uses the $\mathrm{GetPayoff}(x, x^*, f)$ subroutine (\cref{alg:payoff}), which computes the highest-revenue sequence of actions that moves the user state from $x$ to $x^*$, assuming such a transition is optimal.
\cref{alg:payoff} uses the characterization of the optimal policy from \cref{lemma:piecewiseconstant} to identify this trajectory.
If $|x^* - x|$ is an integer multiple of $(K+C_E)$, the subroutine repeatedly chooses the action $\sgn(x- x^*)K$ until $x^*$ is reached.
Otherwise, it adjusts the first or the last step (depending on whether $x^* > x$) so that $x^*$ is reached exactly.
In addition to computing the optimal trajectory, $\mathrm{GetPayoff}(x, x^*, f)$ returns the total discounted revenue and the variable discount factor, as defined in \cref{theorem:cputime}.

\cref{alg:linear} takes as input the demand function $f$ and the set $D$ of its discontinuities. It initializes two dictionaries, $V$ and $\text{Trajs}$, whose keys are the discontinuities in $D$ as well as $0$, the initial neutral state. The entry $V[x]$ stores the highest achievable payoff when starting at a user state of $x$, assuming the optimal policy moves the user state in the direction indicated by the sign of $x$ (this assumption is justified by the monotonic behavior observed in \cref{lemma:piecewiseconstant}). Correspondingly, $\text{Trajs}[x]$ stores the sequence of content choices that yield this optimal payoff.

The algorithm first evaluates states where user satisfaction is negative $(d \leq 0)$, processing them in ascending order (\cref{line:negative}). For each discontinuity, the algorithm evaluates three possible strategies:
\begin{enumerate}
    \item First, in \cref{line:below_D}, the algorithm evaluates the payoff of moving the user state from $d$ down past the smallest discontinuity in $D$, using the $\mathrm{GetPayoff}$ subroutine. This scenario represents the case where the best action is to continually drive user satisfaction downward---denoted by the trajectory $[K]^\infty$, meaning that the maximum action $i = K$ is applied indefinitely.
    \item For each lower discontinuity $d' < d$, the algorithm evaluates the payoff of transitioning from $d$ to $d'$, and then maintaining user satisfaction at $d'$ indefinitely (\cref{line:evaluate_d'}).
    \item Finally, the algorithm evaluates the payoff of maintaining the user state at $d$ (\cref{line:stay}).
\end{enumerate}
For each $d \leq 0$, $V[d]$ and $\text{Trajs}[d]$ are updated according to the best of these three strategies.

The algorithm then evaluates states where user satisfaction is positive $(d\geq 0)$ in a symmetric manner (\cref{line:positive}). Importantly, by \cref{lemma:piecewiseconstant}, the maximum user satisfaction achievable under the optimal policy is $\max D$, so no further transitions upward beyond this point are considered.
{In this step, $V[0]$ and $\text{Trajs}[0]$ are updated if increasing the user state provides a higher payoff than decreasing or keeping the user state constant.}

After processing all discontinuities, the algorithm returns $\text{Trajs}[0]$, which represents the optimal sequence of content decisions when starting from a neutral user state of $0$.  This trajectory is constructed by assembling the best strategies computed for each discontinuity.

\begin{restatable}{theorem}{runtime}
    \label{theorem:runtime}
    In the linear setting where the user demand function has a complexity of $k < \infty$, \cref{alg:linear} computes an optimal app policy in time $O(k^2)$.
\end{restatable}

\begin{proof}
We first note some technical lemmas about \cref{alg:linear}\if\litearxiv0{, which we prove in \cref{subsection:algorithmproofs}}\fi{}.

\begin{restatable}{lemma}{algorithmleft}
    \label{fact:algorithmleft}
For $d=0$ and every discontinuity $d \leq 0$ of $f$, if $\pi^*(d) \geq C_E$, then $\text{Trajs}[d]$ describes the trajectory  of the optimal policy $\pi^*$ of \cref{lemma:piecewiseconstant} starting at user state $d$.
\end{restatable}

\if\litearxiv1

\begin{proof}
Let $(\ite x1, \ite i1), (\ite x2, \ite i2), \dots$ denote the trajectory of $\pi^*$ starting at $d$.
We proceed inductively.
First consider the base case where $d = \min D$.
By \cref{lemma:piecewiseconstant}, $\ite xt$ must be monotonically non-increasing and, if $\ite it = C_E$, then $\ite xt \in D$.
By \cref{lemma:trajectory}, since there are no discontinuities below $D$, for all $t > 1$, either $\ite it = C_E$ for all $t \geq 1$ or $\ite it = K$ for all $t \geq 1$.
The payoff of the former trajectory is computed in \cref{line:evaluate_d'} and the payoff of the latter is computed in \cref{line:below_D} of \cref{alg:linear}, and the maximum is taken for $V[d]$ and $\text{Trajs}[d]$, tie-breaking in favor of \cref{line:below_D}.

For the inductive step, we fix a $d \in D$ where $d < 0$.
By \cref{lemma:piecewiseconstant}, $\ite xt$ must be monotonically non-increasing and, if $\ite it = C_E$, then $\ite xt \in D$.
Consider the set of discontinuities visited by the optimal policy $\pi^*$: $\bset{\ite xt \mid t \geq 2, \ite xt \in D}$.
If this set is empty, i.e. no discontinuities are visited, then by \cref{lemma:trajectory}, $\ite it = K$ for all $t \geq 1$ and \cref{line:below_D} computes the payoff of $\pi^*$.
If it is not, \cref{line:below_D} computes the payoff of a policy $\pi$ that does not visit any discontinuities.
If this set consists only of $d$, then $\ite it = C_E$ for all $t \geq 1$ and \cref{line:stay} computes the payoff of $\pi^*$.
If it does not, \cref{line:stay} computes the payoff of a policy $\pi$ that does stay at $d$.
If the smallest element of the set is $d' = \min \bset{\ite xt \mid t \geq 2, \ite xt \in D}$, suppose that $t'$ is the first timestep where $\ite x{t'} = d'$.
By \cref{lemma:piecewiseconstant}, $\ite it = \min \bset{K, x_{t-1} - x^* + C_E}$ for all $t < t'$.
Thus, $J_{d}(\pi^*) = \hat v + \hat \gamma \cdot J_{d'}(\pi^*)$ where $\hat v, \hat \gamma, \text{Traj} \leftarrow \mathrm{GetPayoff}(d, d', f)$. 
Since $d' \leq d$ and $\ite x{t'} \geq \ite x{t'+1}$, we have by inductive hypothesis that $J_{d'}(\pi^*) = V_{d'}^*$.
Thus, \cref{line:evaluate_d'} computes the payoff of $\pi^*$.

In any of the three possible cases, one of \cref{line:below_D}, \cref{line:evaluate_d'} or \cref{line:stay} must have computed $J_{d'}(\pi^*)$.
Moreover, they will do so before computing the payoffs of any other optimal policies.
Noting that \cref{line:below_D}, \cref{line:evaluate_d'} and \cref{line:stay} only ever compute the payoffs of valid policies, the optimality of $\pi^*$ implies that $\text{Trajs}[d]$ describes the action trajectory unrolled by $\pi^*$.
\end{proof}
\fi

\begin{restatable}{lemma}{algorithmright}
    \label{fact:algorithmright}
For $d=0$ and every discontinuity $d \geq  0$ of $f$, if $\pi^*(d) \leq C_E$, then $\text{Trajs}[d]$ describes the trajectory of the optimal policy $\pi^*$ of \cref{lemma:piecewiseconstant} starting at user state $d$.
\end{restatable}

\if\litearxiv1
\begin{proof}
Let $(\ite x1, \ite i1), (\ite x2, \ite i2), \dots$ denote the trajectory of $\pi^*$ starting at $d$.
We proceed inductively.
First consider the base case where $d = \max D$.
By \cref{lemma:piecewiseconstant}, $\ite xt$ must be monotonically non-increasing and, if $\ite it = C_E$, then $\ite xt \in D$.
Thus, $\ite it = C_E$ for all $t \geq 1$, the payoff of which is computed in \cref{line:pos_stay} of \cref{alg:linear}.

For the inductive step, we fix a $d \in D$ where $d > 0$.
By \cref{lemma:piecewiseconstant}, $\ite xt$ must be monotonically non-decreasing and, if $\ite it = C_E$, then $\ite xt \in D$.
Consider the set of discontinuities visited by the optimal policy $\pi^*$: $\bset{\ite xt \mid t \geq 2, \ite xt \in D}$.
This set cannot be empty, as by \cref{lemma:piecewiseconstant}, some discontinuity must be reached by $\pi^*$ in finite time.
If this set consists only of $d$, then $\ite it = C_E$ for all $t \geq 1$ and \cref{line:pos_stay} computes the payoff of $\pi^*$.
If it does not, \cref{line:pos_stay} computes the payoff of a policy $\pi$ that does stay at $d$.
If the largest element of the set is $d' = \max \bset{\ite xt \mid t \geq 2, \ite xt \in D}$, suppose that $t'$ is the first timestep where $\ite x{t'} = d'$.
By \cref{lemma:piecewiseconstant}, letting $\delta = (d' - d) \mathrm{\;mod\;} (K + C_E)$ for all $t < t'$, we know that action $i_t = \delta$ if $\delta > 0$ and $t = 1$, while $i_t = -K$ otherwise.
Thus, $J_{d}(\pi^*) = \hat v + \hat \gamma J_{d'}(\pi^*)$ where $\hat v, \hat \gamma, \text{Traj} \leftarrow \mathrm{GetPayoff}(d, d', f)$. 
Since $d' \geq d$ and $\ite x{t'} \geq \ite x{t'+1}$, we have by inductive hypothesis that $J_{d'}(\pi^*) = V[d']$.
Thus, \cref{line:pos_d'} computes the payoff of $\pi^*$.
In either of these cases, either \cref{line:pos_stay} or \cref{line:pos_d'} must have computed $J_{d'}(\pi^*)$.
Moreover, they will do so before computing the payoffs of any other optimal policies.
Noting that \cref{line:pos_stay} and \cref{line:pos_d'} only ever compute the payoffs of valid policies, the optimality of $\pi^*$ implies that $\text{Trajs}[d]$ describes the action trajectory unrolled by $\pi^*$.
\end{proof}
\fi

Consequently, whether $\pi^*$ initially increases or decreases the user state, \cref{alg:linear} ensures that $V[0] = J(\pi^*)$ and $\text{Traj}[0]$ is the trajectory unrolled by $\pi^*$.
This proves the correctness of \cref{alg:linear}.

To analyze runtime, note that each call to \cref{alg:payoff} takes $O(1)$ time per unit distance $\abs{x^* - x} / (K - C_E)$.
Since \cref{alg:linear} only calls \cref{alg:payoff} for transitions between discontinuities--and treating the gap between discontinuities as a constant---each call runs in $O(1)$ time.
The outer loops of \cref{alg:linear} iterate over at most $k+1$ discontinuities, while the inner loop iterates over at most $k$ discontinuities.
Therefore, the total runtime is $O(k^2)$.
\end{proof}

\subsection{Online Learning of Content Selection Policies}
We now consider an online learning setting in which an app interacts with a sequence of users one at a time and must choose a content selection policy for each.
Upon encountering a new user---whose demand function is unknown to the app and may be adversarially selected---the app commits to a content selection policy for the user.
The user's interactions with the app then unfold over a long time horizon of $T \to \infty$ timesteps.
In this section, we explore a more challenging setting than the original model (where the app receives feedback at every timestep): here, the app obtains only a single terminal feedback signal---the final episodic revenue---at the end of each user session. This limited feedback framework complicates learning, as the app must adjust its policy based solely on the overall outcome of an entire user interaction without access to incremental, per-timestep observations. Moreover, we impose no restrictions on the complexity of the user demand functions beyond assuming that they eventually stabilize, meaning that for each user, their demand function converges to a constant value outside a bounded domain $[-m, m]$. In other words, for $x \notin [-m ,m]$, the derivative of the demand function is zero.

Although we do not assume the demand functions are piecewise-constant, we leverage \cref{lemma:piecewiseconstant} and \cref{lemma:trajectory}, which show that under such assumptions, optimal content selection policies exhibit a structured form, significantly reducing their complexity. While these lemmas specifically apply to piecewise-constant demand functions, we use them to construct approximate coverings over the space of optimal policies, facilitating the direct application of no-regret learning algorithms.

\paragraph{Three content choices.}
We first examine a setting where an app selects from three distinct content options: one that increases user engagement, one that decreases user engagement, and one that has no effect.
In this scenario, \cref{lemma:piecewiseconstant} constrains the space of possible optimal policies to a small, finite set. This limited policy space allows for the application of the Exp3-IX \cite{neu15} bandit algorithm, which guarantees a sublinear regret bound. \if\litearxiv0 The proof is in \cref{subsection:trionlinelearningproof}.\fi
\begin{restatable}{theorem}{trionlinelearning}
    \label{theorem:trionlinelearning}
    Consider a linear content selection problem in an online learning setting where users arrive sequentially with unknown demand functions $f_1, \dots, f_N$ and the platform must fix a policy $\pi_j$ for each user $j \in [N]$ while only observing the realized episodic revenues $\hat J_{f_1}(\pi_1), \dots, \hat J_{f_{j-1}}(\pi_{j-1})$ of prior users.
    Suppose that $f_1, \dots, f_N$ take constant values outside the domain $[-m, m]$, the episodic revenue from each user is bounded by $U$,
    and that the available content options are limited to $\cI = \bset{2C_E-K, C_E, K}$.
    There is an online learning algorithm for choosing $\pi_1, \dots, \pi_N$ such that, with probability at least $1 - \delta$,
    \begin{align*}
        \sum_{j=1}^N \hat J_{f_j}(\pi_j)
        \geq \max_{\pi^* \in \Pi} \sum_{j=1}^N J_{f_j}(\pi^*) - O\bigg(U \cdot \sqrt{\tfrac{Nm }{K} \log\tfrac{m }{K\delta}}\bigg).
    \end{align*}
    Here, $\Pi$ denotes the set of simple policies.
\end{restatable}

The dependence on $U$ in \cref{theorem:trionlinelearning} arises from the need to control the noise introduced by only observing each user's episodic revenue.
A similar bound can be obtained by substituting $U$ with a term that depends on the variance of each user's episodic revenue.
{
We also note that the same upper bound holds if we were to replace the realized payoffs $\sum_{j=1}^N \hat J_{f_j}(\pi_j)$ with the in-expectation payoffs $\sum_{j=1}^N J_{f_j}(\pi_j)$.
}

\if\litearxiv1
\begin{proof}
For each $f_j$, we construct its \emph{rounded counterpart} $f'_j$ defined as follows:
\begin{equation}
f'_j(x)
\asseq
 f_j(\floor{x / (K - C_E)} (K - C_E))%
\label{eq:round}\end{equation}
This rounding process transforms each demand function into a piecewise constant function.
Since each demand function $f_j$ is a constant function on the domain $(-\infty, -m) \cup (m, \infty)$, %
we know that the discontinuities of its rounded counterpart $f'_j$ lie in the set $[-m, m]$.
Moreover, by construction, $f'_j(x) \leq f_j(x)$ for all $x \in \reals$, meaning that the rounded demand functions correspond to a user who is less likely to engage and, therefore, provides less value to the app.

Importantly, we can verify that the optimal payoff is not affected by this rounding. To see why, observe that $(C_E - i) \mod (K - C_E) = 0$ for all $i \in \{2C_E - K, C_E, K\}$.
As a result, the user state always lies on a multiple of $K - C_E$: $\ite x t \mod (K - C_E) = 0$.
Rounding thus does not impact the optimal payoff:
\begin{equation*}
    \max_{\pi^* \in \Pi} \sum_{j=1}^N J_{f_j}(\pi^*) = \max_{\pi^* \in \Pi} \sum_{j=1}^N J_{f_j'}(\pi^*).
\end{equation*}

The key advantage of working with rounded demand functions is that they are piecewise constant with strategically spaced discontinuities, as summarized by the following lemma:
\begin{restatable}{lemma}{otwo}
\label{lemma:o2}
   Given a set of rounded demand functions $f_1, \dots, f_N$, define the set of simple policies $\Pi' = \bset{\pi_{i, v} \mid v \in \pm 1, i \in [0, \dots, 2 m / K] \cup \bset{-\infty}}$ where each policy $\pi_{i,v}$ is defined by setting $\pi_{i, v}(i \cdot K - m) = C_E$ and $\pi_{i, v}(x) = C_E + v (K - C_E)$ for all $x \neq i \cdot K - m$.
   There is an optimal policy in this set, i.e. \[
\max_{\pi^* \in \Pi} \sum_{j=1}^N J_{f_j'}(\pi^*)
= \max_{\pi^* \in \Pi'} \sum_{j=1}^N J_{f_j'}(\pi^*),
\]
\end{restatable}

Since our reduced policy space $\Pi'$ from \cref{lemma:o2} is small with $\setsize{\Pi'} \in O(\frac m K)$, we can directly apply a bandit learning algorithm to select among the policies in $\Pi'$.
Moreover, since each user's episodic feedback is bounded by $U$, we can apply a folklore stochastic approximation argument~\cite{nemirovski2009robust} to the high-probability regret bound of Exp3-IX \cite{neu15}, as summarized by the following lemma.

\begin{lemma}
\label{lemma:exp3}
Consider an agent who repeatedly uses the (random) Exp3-IX algorithm \cite{neu15} to select an action $i_t$ from a finite set $\cA$.
An adversary, who observes $i_1, \dots, i_t$, chooses a loss $\loss_t \in [0, 1]^{\cA}$. However, the agent receives only a noisy, unbiased estimate $\hat \loss_t \in [0, 1]^{\cA}$, satisfying $\mathbb{E}[\hat \loss_t] = \loss_t$.
The agent then chooses its next action $i_{t+1}$ according to the Exp3-IX algorithm, which is computed from the history $\bsetf{(i_\tau, \hat \loss_\tau(i_\tau))}_{\tau \in [t]})$.
Then, with probability at least $1 - \delta$ in the randomness of Exp3-IX and the observed losses $\hat \loss$, the cumulative loss incurred by the agent satisfies
\[
\sum_{t=1}^T \hat \loss_t(i_t) \geq \max_{i^* \in \cA} \sum_{t=1}^T \loss_t(i^*) - O(\sqrt{T \setsize{A} \log( \setsize{A} /\delta)}).
\]
\end{lemma}
Applying Exp3-IX results in the regret bound
\begin{align*}
\sum_{j=1}^N \hat J_{f_j}(\pi_j)
&\geq \max_{\pi^* \in \Pi'} \sum_{j=1}^N J_{f_j}(\pi^*) - O(U \cdot \sqrt{({Nm }/K)  \log(m /K\delta)}) \\
&\geq \max_{\pi^* \in \Pi'} \sum_{j=1}^N J_{f_j'}(\pi^*) - O(U \cdot \sqrt{({Nm }/K) \log(m /K\delta)}) \\
&= \max_{\pi^* \in \Pi} \sum_{j=1}^N J_{f_j}(\pi^*) - O(U \cdot \sqrt{({Nm }/K) \log(m /K\delta)}).
\end{align*}
\end{proof}
\else
\begin{proof}[Proof sketch of \cref{theorem:trionlinelearning}]
For each user $j$, we construct a ``rounded'', piecewise-constant demand function $f_j'$ with discontinuities in $[-m,m]$, where
\begin{equation*}
f'_j(x)
\asseq
 f_j(\floor{x / (K - C_E)} (K - C_E))  %
\end{equation*}
This rounding process transforms each demand function into a piecewise constant function.
Since each demand function $f_j$ is a constant function on the domain $(-\infty, -m) \cup (m, \infty)$, we know that the discontinuities of its rounded counterpart $f'_j$ lie in the set $[-m, m]$.
Moreover, by construction, $f'_j(x) \leq f_j(x)$ for all $x \in \reals$, meaning that the rounded demand functions correspond to users who are less likely to engage and, therefore, provide less value to the app.

We also note that, for any sequence of actions $\ite i 1, \ite i2, \dots$, the user state is always a factor of $K- C_E$ away from origin; this is because $i = C_E$ keeps user state constant, $i = K$ decreases user state by $K - C_E$, and $i = 2C_E - K$ increases user state by $K - C_E$.
This means that the optimal payoff the app can attain on the rounded demand functions is the same as the optimal payoff on the original demand functions:
\begin{equation*}
    \max_{\pi^* \in \Pi} \sum_{j=1}^N J_{f_j}(\pi^*) = \max_{\pi^* \in \Pi} \sum_{j=1}^N J_{f_j'}(\pi^*).
\end{equation*}

Therefore, we can restrict our attention to a small set of simple policies $\Pi'$ of size $O(m/k)$, which we prove contains an optimal policy. With this reduced policy space, we apply Exp3-IX \cite{neu15} to select a policy for each user based solely on their episodic revenue, which is bounded by $U$. A standard stochastic approximation argument then yields the final regret bound.
\end{proof}
\fi

\paragraph{Spectrum of content choices.}
We next study online learning when the app can select from a spectrum of content choices.
Given the expanded decision space, we construct an \emph{approximate} covering of the space of optimal policies, which introduces an approximation factor of $\tfrac 12$.
\if\litearxiv0
The full proof is in \cref{subsection:onlinelearningproof}.
\fi

\begin{restatable}{theorem}{onlinelearning}
    \label{theorem:onlinelearning}
    Consider a linear content selection problem with content $\cI = [2C_E - K, K]$ in an online learning setting where users arrive sequentially with unknown demand functions $f_1, \dots, f_N$ and the platform must fix a policy $\pi_j$ for each user $j \in [N]$ while only observing the realized episodic revenues $\hat J_{f_1}(\pi_1), \dots, \hat J_{f_{j-1}}(\pi_{j-1})$ of prior users. 
    Suppose that $f_1, \dots, f_N$ take constant values outside the domain $[-m, m]$ and the episodic revenue from each user is bounded by $U$.
    There is an online learning algorithm for choosing $\pi_1, \dots, \pi_N$ such that,  with probability at least $1 - \delta$,
    \begin{align*}
        \sum_{j=1}^N \hat J_{f_j}(\pi_j)
        \geq  \max\left\{\tfrac {C_R + K}{C_R - C_E + 2K}, \tfrac {1}{2}  \right\} \cdot \max_{\pi^* \in \Pi} \sum_{j=1}^N J_{f_j}(\pi^*) - O\bigg(U \cdot \sqrt{\tfrac{Nm }{K} \log\tfrac {m }{K\delta}}\bigg).
    \end{align*}
\end{restatable}
\if\litearxiv1

\begin{proof}
For each $f_j$, we construct its rounded counterpart $f'_j$ as in Equation~\eqref{eq:round}.
We can verify that the optimal payoff is not significantly affected by this rounding. 
\begin{restatable}{lemma}{oone}
\label{lemma:o1}
   The optimal payoff that can be realized with the rounded demand functions $f'_1, \dots, f'_N$ is at least a constant fraction of that of the original demand functions $f_1, \dots, f_N$:
   \[\max_{\pi^* \in \Pi} \sum_{j=1}^N J_{f_j}(\pi^*) \leq  \left(1 + \tfrac{C_E + K}{C_R + K}\right) \max_{\pi^* \in \Pi} \sum_{j=1}^N J_{f_j'}(\pi^*).\]
\end{restatable}

\cref{lemma:o1} and \cref{lemma:o2} give
\begin{align*}
\max_{\pi^* \in \Pi} \sum_{j=1}^N J_{f_j}(\pi^*)
\leq \left(1 + \tfrac{C_E + K}{C_R + K}\right) \max_{\pi^* \in \Pi'} \sum_{j=1}^N J_{f_j'}(\pi^*) \leq \left(1 + \tfrac{C_E + K}{C_R + K}\right) \max_{\pi^* \in \Pi'} \sum_{j=1}^N J_{f_j}(\pi^*),
\end{align*}
where the second inequality follows from the fact that $f_j'(x) \leq f_j(x)$ for all $x \in \reals$.

As in our proof of \cref{theorem:trionlinelearning}, we can again apply a bandit learning algorithm to choose from $\Pi'$.
We can apply a standard stochastic approximation argument~\cite{nemirovski2009robust} to the high-probability regret bound of Exp3-IX \cite{neu15} on the reduced policy space $\Pi'$ to get:
\begin{align*}
\sum_{j=1}^N \hat J_{f_j}(\pi_j)
&\geq \max_{\pi^* \in \Pi'} \sum_{j=1}^N J_{f_j}(\pi^*) - O\left(U \cdot \sqrt{\tfrac{Nm \gamma}{K}  \log\tfrac {m }{K\delta}}\right) \\
&\geq \max_{\pi^* \in \Pi'} \sum_{j=1}^N J_{f_j'}(\pi^*) - O\left(U \cdot \sqrt{\tfrac{Nm \gamma}{K}  \log\tfrac {m }{K\delta}}\right) \\
&\geq \tfrac {C_R + K}{C_R + C_E + 2K} \max_{\pi^* \in \Pi} \sum_{j=1}^N J_{f_j}(\pi^*) - O\left(U \cdot \sqrt{\tfrac{Nm \gamma}{K}  \log\tfrac {m }{K\delta}}\right).
\end{align*}
Finally, since $C_E \leq K$ and $C_R \geq 0$, $ \tfrac {C_R + K}{C_R - C_E + 2K} \geq \tfrac 12$.
\end{proof}
\else
\begin{proof}[Proof sketch]
    As in \cref{theorem:trionlinelearning}, we first construct rounded demand functions $f_j'$ for each $f_j$. However, in this case, restricting the policy space to $\Pi'$ introduces a multiplicative loss of $\tfrac {C_R + K}{C_R - C_E + 2K} \leq \tfrac{1}{2}.$ Consequently, when we apply Exp3-IX \cite{neu15} over the reduced policy space, our regret guarantee holds up to a factor of $\tfrac{1}{3}$ relative to the optimal cumulative reward.
\end{proof}
\fi

Therefore, even with limited feedback---observing only the episodic revenue from each user---an online learning algorithm can achieve near-optimal performance.

\section{Analyzing User Disengagement with Modified Demand Elasticity}
\label{section:primitives}
In this section, we build on \cref{theorem:cputime} to introduce \emph{modified demand elasticity}, a key primitive that encapsulates how a user's demand function, friction, and the discount factor affect an app's optimal content selection policy. We demonstrate that modified demand elasticity 
serves as a powerful tool for analyzing how fundamental model primitives ---such as friction---affect optimal decision-making. In particular, we leverage this concept to reveal a counterintuitive phenomenon: increased friction can lead to higher user engagement while simultaneously reducing revenue.
We also show that greater modified demand elasticity implies better alignment between an app creator and their user's satisfaction.

\subsection{Modified Demand Elasticity}
Demand elasticity is a fundamental concept in economics that measures how sensitive consumer demand is to a firm's decisions, such as pricing. It plays a central role in shaping firm behavior---when demand is highly responsive to changes in the firm's actions, firms are incentivized to take actions that boost demand.
In the context of repeated app-user interactions, the natural analogue of demand elasticity is the sensitivity of user demand to the content an app provides.
To capture this idea, we introduce \emph{modified demand elasticity}, defined as:
$$\frac{\partial}{\partial x} \log \tilde f(x) = \frac{\partial}{\partial x} \log \para{f(x) + \frac{1 - f(x)}{1 - \gamma (1 - (1-c) f(x))} (1-c) f(x) \gamma}.$$
This quantity measures how responsive the modified demand function $\tilde f$ (as defined in \cref{theorem:cputime}) is to changes the user state $x$, which reflects variations in the content selection.
Moreover, modified demand elasticity captures key information about both the user's demand function $f$ and the friction parameter $c$, making it a crucial determinant of an app creator's long-term utility. 
As \cref{theorem:cputime} demonstrates, the friction parameter $c$ influences the app's objective solely through the modified demand function $\tilde f(x)$. In other words, the sensitivity of the effective discount factor to the user state---captured by modified demand elasticity---fully characterizes how both the user demand function $f(x)$ and friction $c$ together shape the app creator's long-term utility.

\subsection{User Engagement is Not Monotone in Friction}
\label{section:implicit}

We now show that while an app's optimal payoff decreases as user friction increases (since higher friction generally reduces long-term rewards), the relationship between friction and user engagement under the optimal policy is more nuanced. Counterintuitively, increasing friction---which lowers the probability that users return---can, in some cases, lead to higher user engagement.

\paragraph{A comparative statics analysis.}
This paradoxical phenomenon arises from app creators' strategic incentives, which we can analyze using comparative statics.
Consider a simplified model where (1) every piece of content $i \in \cI \subseteq \reals$ has a deterministic effect on user experiences, (2) the demand function $f$ is strictly increasing, and (3) the optimal policy begins by showing some specific content $i = \bi$ in the first interaction and holds the user's state steady thereafter by consistently showing content $i'$ so that $x_2 = x_3 = \cdots$.
In this setting, the app's investment in user engagement is captured by its {initial} choice of content  $\bi$. If the optimal $\bi$ increases with friction $c$, then higher friction encourages the app to select content that promotes greater engagement.

We can formalize this intuition by formulating the content selection problem as a Markov decision process (\cref{fact:edpmdp}) where the user's state $\tsv{x}{t}$ is the process state. 
Accordingly, we define the Q-function $Q: \cI \times \reals \to\reals$ where $$Q(i, x) = \EE{R_i + \gamma \tilde f(x + E_i) \max_{i \in \cI} Q(i, x+ E_i)}$$ represents the optimal discounted payoff achievable by selecting content $i$ when the user is in state $x$.
Since the app follows an optimal policy, we have $\bi \in \argmax_i Q(i, 0)$, implying that $Q(i,x)$ attains a local maximum at $(\bi, 0)$, i.e., $\smash{\tfrac{\partial}{\partial i} Q(\bi, 0) = 0}$ and $\smash{\frac{\partial^2}{\partial i^2 } Q(\bi, 0)}<0$.
We assume that $Q$ is twice differentiable and that the revenue function $i \mapsto R_i$ is differentiable.
To make the dependence on the friction parameter $c$ explicit, we denote the Q-function as $Q_c$.

To examine how the optimal content choice $\bi$ evolves with friction $c$, we apply the implicit function theorem~\citep{topkis1998supermodularity} to analyze how the solution to the first-order condition $\smash{\tfrac{d}{di} Q_c(\bi, 0) = 0}$ changes as friction $c$ varies.
This theorem yields:
\begin{align*}
    \frac{\partial \bi}{\partial c} = - \left[\frac{\partial^2}{\partial c \partial i} Q_c(\bi, 0)\right]\left[\frac{\partial^2}{\partial i^2 } Q_c(\bi, 0)\right]^{-1}.
\end{align*}
Since the second derivative $\frac{\partial^2}{\partial i^2 } Q_c(\bi, 0)$ is negative by assumption, the sign of $\tfrac{\partial \bi}{\partial c}$ depends on the cross-partial derivative $\frac{\partial^2}{\partial c \partial i} Q_c(\bi, 0)$.
If the cross-partial derivative is positive, increasing friction $c$ incentivizes the app to select a higher value of $\bi$, reflecting greater investment in engagement.

To better understand this cross-partial derivative, we express the Q-function explicitly in terms of friction by leveraging \cref{theorem:cputime}, which implies that
\begin{align*}
    Q_c(\bi, 0) %
    &= R_i + \gamma \tilde f(E_\bi) \sum_{t=1}^\infty \left(\gamma \tilde f (E_\bi)\right)^{t-1} R_{i'} = R_{\bi} + R_{i'} \tfrac{\gamma \tilde f_c(E_{\bi})}{1-\gamma \tilde f_c(E_{\bi})}.
\end{align*}
Since $\tilde f(E_\bi)$ is a function of $c$ and $f(E_\bi)$, differentiating with respect to $c$ and $i$ yields
\begin{align*}
    \frac{\partial^2}{\partial c \partial i} Q_c(\bi, 0)
    = 
    \frac{\partial^2}{\partial c \partial i} \para{   R_{\bi} + R_{i'} \tfrac{\gamma \tilde f_c(E_{\bi})}{1-\gamma \tilde f_c(E_{\bi})}}
    =
    \gamma^2 R_{i'} \cdot \frac{\partial f(E_{\bi})}{\partial i} \cdot \underbrace{\frac{ (2- c \gamma) f(E_{\bi}) - 1}{(1-\gamma)(1 - f(E_{\bi})\gamma c)^3}}_{A}.
\end{align*}
Note that $\tfrac {\partial}{\partial c} R_i = 0$ since an arm's payoff is independent of friction $c$.
For the second term, we expanded the definition of $\smash{\tilde f(E_\bi)}$ in terms of $c$ and $f(E_\bi)$ (see \cref{theorem:cputime}) and computed the derivative.
Since all terms except $A$ are strictly positive, $A$ determines the sign of $\tfrac{\partial^2}{\partial c \partial i} Q_c(\bi, x)$.
In particular, if $(2 - c \gamma) f(E_i) \geq 1$ (for example if $c \gamma = 0.5$ and $f(E_i) \geq 2/3$), then $A$ is positive.
This occurs when the app's optimal policy already {ensures a relatively} high probability of engagement.
In such cases, a marginal increase in friction incentivizes the app to select content that boosts engagement.
Next, we present a concrete example where this effect arises.
\begin{example}\label{ex:first_example}
Suppose a user's satisfaction with an app can be categorized into three levels:
\begin{itemize}[itemsep=0pt, topsep=0pt, left=2pt]
    \item If their cumulative satisfaction is below a threshold $a \in \R$, they dislike the app and stop using it.
    \item If their cumulative satisfaction falls within an interval $[a,b)$, they moderately enjoy the app, continuing to use it with a probability of 60\%.
    \item If their cumulative satisfaction exceeds $b$, they are enthusiastic about the app and remain engaged with a 99\% probability.
\end{itemize}
\noindent
Now, suppose the app creator has a discount factor of $\gamma=0.9$ and selects content from a linear landscape parameterized by  $i \in [-b, b]$. In this setting, displaying content $i$ yields revenue $R_i = 1 + i$ and affects user experience according to $E_i = -i$.

A key observation in this example is that increasing the friction parameter $c$---which reduces the probability of users returning---can paradoxically result in a strictly better user experience. In other words, when facing higher friction, the app is incentivized to provide content that enhances user satisfaction more than it would under lower friction.
\begin{restatable}{proposition}{basiceng}
\label{prop:basic_1}
In \cref{ex:first_example}, for appropriate choices of $a, b$, user satisfaction is strictly lower under the optimal policy when friction is lower compared to when friction is higher.
Formally, for any $c' > c$, let $x^c_t$ and $x^{c'}_t$ denote the user states at time $t$ under the optimal policies for friction parameters $c$  and $c'$, respectively. Then, for all $t$, $x^c_t \leq x^{c'}_t$, with strict inequality for $t \geq 2$. 
\end{restatable}
This result illustrates the counterintuitive effect of friction: although higher friction generally reduces the app’s long-term payoff, it can simultaneously compel the app to adopt a more engagement-focused content strategy, ultimately leading to improved user satisfaction.
\end{example}

\paragraph{The role of modified demand elasticity.}
    Modified demand elasticity offers a clear explanation for why friction incentivizes investment in user engagement.
    In particular, $\tfrac{\partial}{\partial x} \log \tilde f$ always increases in friction $c$, making user demand more sensitive to changes in user state and amplifying the impact of content choices on engagement. 
    Moreover, when comparing complete friction ($c=1$) and no friction ($c=0$), the ratio of modified demand elasticity $\tfrac{\partial}{\partial x} \log \tilde f$ grows linearly in user demand $f(x)$.
    This implies that the effect of increasing friction on demand elasticity is most pronounced when user demand is already high. In other words, when engagement levels are high, even small increases in friction can significantly heighten the sensitivity of demand to content selection, further motivating the app to optimize for retention.

\paragraph{A high-engagement regime for friction.}
 {Intuitively,}
when user engagement is already low, the app may struggle to increase engagement enough for friction to become a less significant concern. If users are already disengaged, small changes in content selection are unlikely to meaningfully alter their behavior. However, when engagement is moderately high, friction can play a more nuanced role, potentially encouraging the app to invest in strategies that further enhance user satisfaction.

We 
formalize this intuition by extending our construction in \cref{ex:first_example}.
To simplify the analysis, we compare two extreme cases: complete friction ($c=1$) and no friction ($c=0$). However, the results naturally extend to intermediate levels of friction.

\begin{example}[Generalization of \cref{ex:first_example}]
\label{ex:generalized_example}
Suppose a user's satisfaction with an app falls into one of three categories:
\begin{itemize}
    \item If their cumulative satisfaction is below a threshold $a$, they will stop using the app.
    \item If their cumulative satisfaction is within an interval $(a,b)$, they continue using the app with probability $p_1$.
    \item If their cumulative satisfaction exceeds $b$, they continue using the app with probability $p_2$.
\end{itemize}
Now consider an app with discount factor $\gamma$ and a linear content landscape, as in \cref{ex:first_example}. The app selects content $i \in [-b, b]$, where displaying content $i$ yields revenue $R_i = 1 + i$ and affects user experience according to $E_i = -i$. We compare the optimal policy and its effect on user experience under two scenarios: no friction $(c = 0)$ versus complete friction $(c = 1).$

\begin{restatable}{proposition}{regime}
\label{prop:regime_1}
In \cref{ex:generalized_example}, there is an appropriate choice of $a$ and $b$ such that user satisfaction is strictly lower when there is no friction ($c = 0$) compared to when there is complete friction ($c = 1$). This holds if the probabilities $p_1$ and $p_2$ satisfy the following conditions:
\begin{enumerate}
    \item The difference between $p_1$ and $p_2$ is not too large: $p_2 - p_1 < \min\bset{\tfrac{1-\gamma}{\gamma(1-\gamma p_1)}, \tfrac{1-\gamma}{2\gamma(1-\gamma p_2)}}$,
    \item The function $h(p) \asseq \frac{1}{1 - \gamma p} - \frac{\gamma}{1-\gamma}p$ satisfies $h(p_2)> h(p_1)$. This function is visualized in \cref{fig:regime}.
\end{enumerate}
\end{restatable}

Intuitively, $h(p)$ quantifies the additional effective discounting introduced by complete friction---representing the additional loss in expected future revenue due to the greater risk of disengagement---relative to the uniform discounting seen in a frictionless environment at a given engagement level $p$. The specific form of $h$ and its implications are explored further later in this section. The proof of this proposition is in \cref{subsection:regimeproof}.

\begin{proof}[Proof sketch of \cref{prop:regime_1}]
    We begin by setting $a = 0$ and $b = \frac {\gamma}{1-\gamma} (\epsilon+p_2-p_1)$ for a small $\epsilon > 0$. By \cref{lemma:piecewiseconstant}, the optimal policy must follow one of three strategies: (i) continuously decrease the user state by repeatedly selecting content $i = b$, (ii) keep the user state at $0$ by always choosing $i = 0$, or (iii) initially increase the state (by selecting $i = -b)$ and then stabilize it at $x = b$ by switching to $i = 0$. In the absence of friction, we show that maintaining the state (policy (ii)) outperforms the mixed strategy (policy (iii)). However, under complete friction, the modified discounting causes the mixed strategy to yield a higher payoff, as the effective discount factor for future revenue increases when the user is more likely to disengage; this is captured by the condition $h(p_2) > h(p_1)$ (condition (2)) together with the requirement that $p_2 - p_1$ is not too large (condition (1)). Finally, we show that the continuously decreasing strategy (policy (i)) is always suboptimal under complete friction, thus establishing that when friction is present, the optimal policy results in a higher level of user satisfaction compared to the frictionless case.
\end{proof}

\begin{figure}
    \centering
    \includegraphics[width=0.55\linewidth]{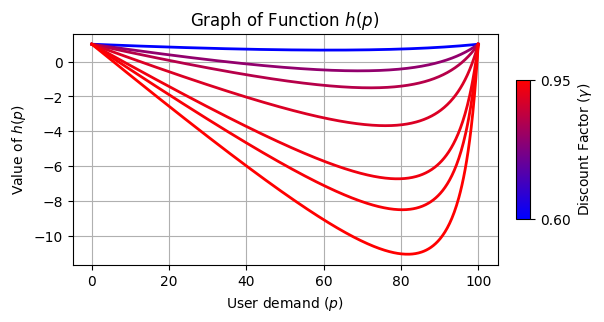}
    \caption{Plot of the function $h(p) \asseq \frac{1}{1 - \gamma p} - \frac{\gamma}{1-\gamma}p$ on the domain $p \in [0, 1]$ for various choices of $\gamma$.}
    \if\arxiv0
    \Description{Plot of the function $h(p) \asseq \frac{1}{1 - \gamma p} - \frac{\gamma}{1-\gamma}p$ on the domain $p \in [0, 1]$ for various choices of $\gamma$.}
    \fi
    \label{fig:regime}
\end{figure}
The second criterion of \cref{prop:regime_1} is typically the more restrictive of the two. \cref{fig:regime} illustrates the function $h$, which exhibits a characteristic U-shape.
This structure implies that 
condition 2 
holds---and thus, the app is incentivized to invest more in engagement under friction---when both $p_1$ and $p_2$ are close to 1, since $h(p)$ is increasing in this region. Conversely, if both $p_1$ and $p_2$ are relatively small (for example, if $p_2$ falls below the minimum of the curve for the given discount factor $\gamma$), the condition is not satisfied.

To provide further intuition, $h$ quantifies the difference between two modifications of the geometric series $\sum_{t=1}^\infty \gamma^{t-1}$. In one, the discount factor $\gamma$ is reduced by a factor of $p$, capturing the effect of disengagement under complete friction---in each timestep, there is a $1 - p$ chance of eliminating future revenue. In the second modification, the entire series is scaled by $p \gamma$, modeling the impact of disengagement with no friction---each timestep's payoff is uniformly reduced by a fixed probability.
Hence, if $h(p_2) > h(p_1)$, it indicates that at the higher demand level $p_2$, the app's utility is more sensitive to the shift from no friction $(c=0)$ to complete friction $(c=1)$.
\end{example}

\subsection{Attributing Friction Phenomena to Modified Demand Elasticity}
\label{section:attribution}
In the previous section, we showed that friction can unexpectedly boost user engagement, even as it typically reduces overall revenue.
Building on those findings, this section further analyzes \emph{how} friction shapes the user's equilibrium under the optimal policy. 
Indeed, {under the optimal policy, the user's state will quickly reach an equilibrium level $x_{\infty},$ as formalized in \cref{lemma:piecewiseconstant}.
This section uses modified demand elasticity to analyze how friction impacts $x_{\infty}.$ 

To maintain the user state at equilibrium, the app repeatedly displays a specific type of content. Given a user state $x$, let $U(x,c)$ be the asymptotic utility that the app derives from repeatedly showing the user content $i \in \cI$ when they are in state $x$. Formally,}
 $$U(x, c) \asseq \tfrac{\EE{R_i}}{1-\gamma \tilde f(x)},$$ where $\tilde f$ is the modified demand function from \cref{theorem:cputime}. {Said another way, $U(x, c)$ represents the long-term utility that awaits an app if it quickly drives the user's state to an equilibrium $x_\infty=x$.}
The difference in asymptotic utility between two states $x, x'$, $U(x, c) - U(x', c)$, characterizes how much revenue the app creator should be willing to sacrifice to shift the user state from $x'$ to $x$.

\begin{figure}
    \centering
    \includegraphics[width=0.75\textwidth]{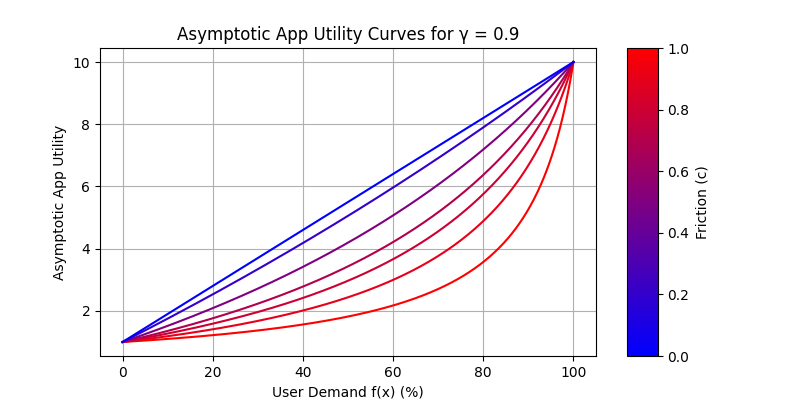}
    \caption{Asymptotic app utility at different levels of user demand and friction. Content revenue and the app creator's discount factor are fixed at $\EE{R_i} = 1$ and $\gamma = 0.9$ for clarity.}
    \if\arxiv0
    \Description{Asymptotic app utility at different levels of user demand and friction. Content revenue and the app creator's discount factor are fixed at $\EE{R_i} = 1$ and $\gamma = 0.9$ for clarity.}
    \fi
    \label{fig:asymp}
\end{figure}

\cref{fig:asymp} illustrates how asymptotic app utility varies with user demand $f(x)$ and friction $c$.
When user demand is high, i.e. $f(x) \to 1$, asymptotic app utility grows faster, with the rate of growth becoming steeper as friction increases.
To interpret this effect, consider the following example:
the difference between asymptotic utility between demand levels $f(x) = 0.9$ and $f(x) = 0.8$ is much larger when friction is high $(c = 1)$ than when friction is lower (e.g., $c = 0.5$). This implies that when friction is high, the app creator is willing to forgo more immediate revenue to increase user demand from $0.8$ to $0.9$ as the long-term benefits of retaining users become more pronounced.

To analyze the curvature of asymptotic app utility more formally, we examine its derivative, which we can decompose into three interpretable factors:
\begin{align}
\label{eq:factors}
    \frac{d}{dx} \frac{1}{1 - \gamma \tilde f(x)}
    = \underbrace{\para{\frac{1}{1-\gamma \tilde f(x)}}^2}_{{(A)}} \cdot \underbrace{\gamma \tilde f(x)}_{(B)}
    \cdot \underbrace{\frac{d}{dx} \log(\tilde f(x))}_{(C)},
\end{align}
Term $A$ is the asymptotic app utility squared.
Term $B$ is the app creator's effective discount factor, as defined in Section~\ref{sec:var-rate}.
Term $C$ is our definition of modified demand elasticity.
While terms $A$ and $B$ are strictly decreasing with friction, term $C$--modified demand elasticity--strictly increases in friction.
This leads to a key insight: \emph{the increase in modified demand elasticity is the sole mechanism behind friction-induced increases in user engagement}.

This decomposition also explains why friction only increases engagement when user demand is already high, i.e. when $f(x)$ is large.
\cref{fig:terms} plots the ratios of each term when there is complete friction ($c = 1$) compared to no friction ($c = 0$).
While the asymptotic app utility ratio is symmetric around $f(x) = 0.5$, the ratios of both the effective discount factor and asymptotic app utility grow linearly with demand, highlighting their direct dependence on user engagement levels.

\begin{figure}
    \centering
\includegraphics[width=0.6\textwidth]{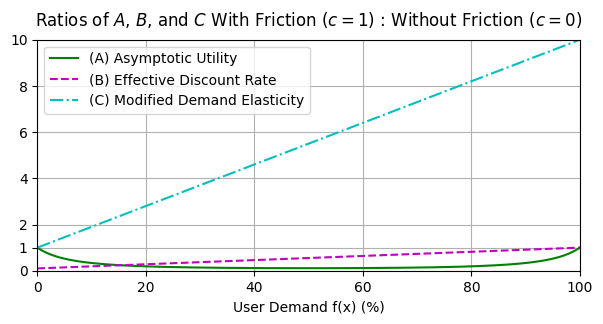}
    \caption{The ratio of the factors (A), (B), and (C) that compose $\tfrac \partial {\partial x} \tfrac 1 {1 - \gamma \tilde f(x)}$ as stated in \eqref{eq:factors}  when there is full friction ($c = 1$) against when there is no friction ($c = 0$). The app creator's discount factor is fixed at $\gamma = 0.9$.}
    \if\arxiv0
    \Description{The ratio of the factors (A), (B), and (C) that compose $\tfrac \partial {\partial x} \tfrac 1 {1 - \gamma \tilde f(x)}$ as stated in \eqref{eq:factors}  when there is full friction ($c = 1$) against when there is no friction ($c = 0$). The app creator's discount factor is fixed at $\gamma = 0.9$.}
    \fi
    \label{fig:terms}
\end{figure}

\subsection{Comparison to Classical Demand Elasticity}
\label{section:comparison}
Our model extends the classical supply-demand framework to a setting where demand evolves over repeated interactions rather than responding instantaneously to a firm's decisions.
In traditional economics, the demand function maps a firm's chosen prices to consumer demand, whereas in our model, demand is shaped by the cumulative effect of an app's content choices over time.
Given this parallel, one might naturally hope to analyze our model using the classical definition of demand elasticity, $\tfrac{\partial}{\partial x} \log f$, where $f$ is the demand function.
However, unlike modified demand elasticity, this classical measure does not account for two key factors: friction (which influences user re-engagement) and the discount factor (which shapes long-term optimization). As a result, classical demand elasticity alone is insufficient for understanding how engagement strategies evolve in dynamic environments with friction.

To better understand why friction must be accounted for in defining demand elasticity, we examine the ratio between modified demand elasticity and its classical counterpart, as shown in \cref{fig:elasticity}.
This ratio forms a convex curve, where the curvature decreases with friction $c$ and increases in the app creator's discount factor $\gamma$.
This pattern reflects an intuitive trade-off: when friction is low, and the app creator is patient ($\gamma$ is large), modified demand elasticity is significantly smaller than its classical counterpart.
In such cases, app creators are less concerned about friction preventing users from returning, as they can afford to patiently wait for users to return if they disengage.
In contrast, when friction is high, even small changes in user demand become more consequential, amplifying the impact of engagement strategies. This effect makes investments in increasing user demand more profitable as the risk of permanent disengagement grows. As a result, modified demand elasticity---unlike its classical counterpart---correctly captures the heightened sensitivity of demand in environments where user retention is more fragile.

This distinction also explains why modified demand elasticity is central to our model: an app creator's behavior is driven by perceived user demand rather than actual user demand.
In classical supply-demand curves, these two align since consumers respond to direct, instantaneous pricing changes.
In our setting, the gap between perceived demand $\tilde f(x)$ and actual demand $f(x)$ does not stem from uncertainty about user behavior but rather the dynamic nature of repeated interactions---where engagement evolves over time---and the adjustments app creators make to maximize long-term utility. Thus, modified demand elasticity better captures how engagement strategies should be optimized in environments where friction and delayed user responses matter.

\begin{figure}
    \centering
    \begin{subfigure}[b]{0.48\textwidth}
        \centering
        \includegraphics[width=\textwidth]{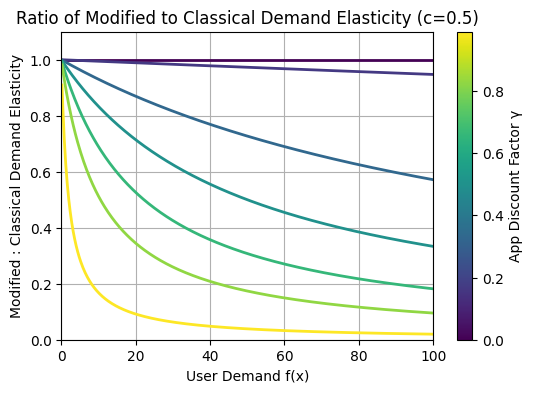}
    \end{subfigure}
    \hfill
    \begin{subfigure}[b]{0.48\textwidth}
        \centering
        \includegraphics[width=\textwidth]{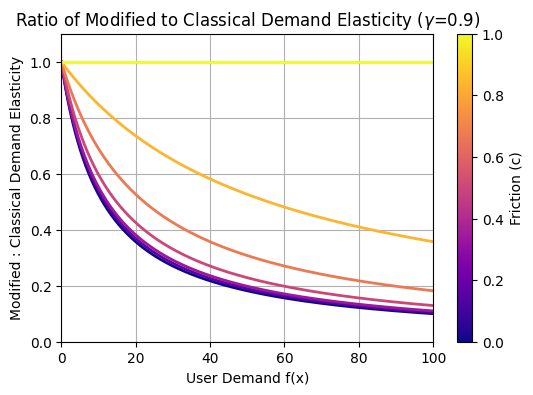}
    \end{subfigure}
    \caption{Ratio of modified demand elasticity $\frac{d}{dx} \log \tilde f(x)$ to classical demand elasticity $\frac {d}{dx} \log f(x)$.}
    \if\arxiv0
    \Description{Ratio of modified demand elasticity $\frac{d}{dx} \log \tilde f(x)$ to classical demand elasticity $\frac {d}{dx} \log f(x)$.}
    \fi
    \label{fig:elasticity}
\end{figure}

\subsection{Demand Elasticity and Alignment}
Modified demand elasticity also serves as a key tool for deriving general insights about app behavior in our model.
For example, it allows us to prove an asymptotic lower bound on the user demand an app should strive to maintain.
In particular, we analyze a form of partial alignment between user welfare and app revenue, quantified by determining the minimum revenue a user must generate to incentivize the app to actively invest in improving their experience.

\cref{theorem:simplealignment} establishes a fundamental relationship between user engagement, lifetime customer value, and modified demand elasticity. Specifically, for any given threshold $x^*$ and $\delta$, it shows that if the lifetime customer value exceeds $x^*$, the app's optimal policy must ensure that the user's state remains above $x^*$ for at least $1 - \delta$ of the time.
A key insight from this result is that the required lifetime customer value is inversely related to the user's modified demand elasticity.
Intuitively, while a user's value to the app represents the ``reward'' associated with increasing engagement, modified demand elasticity reflects the cost-efficiency of doing so. 
When demand elasticity is high, small investments in engagement yield substantial increases in retention, reducing the minimum revenue threshold required for the app to prioritize user experience improvements.

\begin{restatable}{theorem}{alignmenttt}
    \label{theorem:simplealignment}
    Consider any threshold user state $x^* \in \reals$ and small constant $\delta$. Suppose an app expects an achievable payoff of at least $$\max_{\pi} J(\pi) \in \Omega\para{\frac{\log(1/\gamma)}{\delta} \cdot \para{ \frac {\partial}{\partial x} \log \tilde f(x) \mid_{x = x^*}}^{-1}}.$$
    Then, any app policy that allows the user's state to remain below $\sat^*$ for at least a $\delta$-fraction of interactions---formally, $\frac 1 {T} \sum_{t=1}^{T} 1[\ite{\sat}t \leq \sat^*] \geq \delta$ for all large $T$---is suboptimal.
    Here, $\Omega$ treats content attributes $\bset{E_i}_{i \in \arms}$ and $\bset{R_i}_{i \in \arms}$ as constants.
\end{restatable}

This theorem highlights a key structural property of the app’s optimal strategy: when user demand is sufficiently elastic, even moderate lifetime customer value is enough to justify sustained investment in engagement. Conversely, when demand elasticity is low, the app must extract significantly higher revenue from users to make it worthwhile to invest in improving user engagement, reinforcing the critical role of modified demand elasticity in shaping long-term platform incentives.

\section{Discussion}
\label{section:discussion}
This paper developed a model for the algorithmic problem of content selection, capturing the trade-off between maximizing immediate revenue and increasing long-term user engagement.
To demonstrate the model's tractability, we showed that approximate no-regret learning guarantees can be achieved with bandit algorithms and that optimal policies are well-structured and efficiently computable.
We further applied our model as a microfoundation for recommendation systems, analyzing how content selection primitives affect alignment between platforms and users.
In particular, we identified modified demand elasticity as the key primitive determining whether a platform prioritizes engagement over revenue and used this primitive to demonstrate that increased friction can, counterintuitively, boost engagement.
Several open questions remain.
While we established a sufficient condition for app-user alignment (\cref{theorem:simplealignment}), deriving tight alignment guarantees appears challenging.
Tightly characterizing when apps are incentivized to invest in engagement would provide important insights into the strategic behavior of social media and other online platforms.
Moreover, it remains an open question under what conditions the counterintuitive friction phenomenon we observed (\cref{section:implicit}) arises in competitive multi-app settings.

\section{Acknowledgements}
This work was supported in part by the National Science Foundation under grants CCF-2145898 and CCF-2338226, by the Office of Naval Research under grant N00014-24-1-2159,
a C3.AI Digital Transformation Institute grant, the Mathematical Data Science program of the Office of Naval Research, and Alfred P. Sloan fellowship, and a Schmidt Science AI2050 fellowship.
This material is based upon work also supported by the National Science Foundation Graduate Research Fellowship Program under Grant No. DGE 2146752. Any opinions, findings, and conclusions or recommendations expressed in this material are those of the author(s) and do not necessarily reflect the views of the National Science Foundation.
Calvano acknowledges financial support from the ERC-ADV grant 101098332 and PRIN 2022 CUP E53D23006420001.
\bibliographystyle{alpha}
\newcommand{\etalchar}[1]{$^{#1}$}
\newcommand{\nips}[1]{Advances in Neural Information Processing Systems #1}

\newpage
\tableofcontents
\newpage
\appendix
\section{Omitted Proofs and Additional Results for \cref{section:model}}

\subsection{Proof of \cref{lemma:simpleoptimal}}

\cref{lemma:simpleoptimal} is a consequence of the fact that most reasonable Markov decision processes admit an optimal policy that is deterministic and stationary.

\simpleoptimal*

\begin{proof}
We will first prove that discounted payoff is well-defined (\cref{fact:discountedpayoffexists}).
We then write our model as a Markov decision process and demonstrate a bijection between simple policies of our model and stationary policies of the MDP (\cref{fact:edpmdp}).
Finally, we recall that there always exists an optimal stationary policy for an MDP with an optimal policy (\cref{fact:stationaryoptimal}).

We first can verify that discounted payoff \eqref{eq:platformdiscountedobjective} is well-defined (see, e.g., \cref{proposition:averagereward}).
\begin{fact}
\label{fact:discountedpayoffexists}
For any app policy $\pi$, its discounted payoff $J(\pi)$ as defined in \eqref{eq:platformdiscountedobjective} is finite for valid discount factors $\gamma \in (0, 1)$.
\end{fact}
\begin{proof}[{Proof of \cref{fact:discountedpayoffexists}}]
    Since we assume that the set of content revenue means $\bset{\EE{R_i}}_{i \in \cI}$ is a compact set, there exists a constant $K \in \reals$ that upper bounds $\abs{\EE{R_i}} \leq K$ for every content $i \in I$.
    Thus, we can upper bound the payoff by the series $\abs{J(\pi)} \leq \sum_{t=1}^\infty \gamma^{t-1} \abs{r_t} \leq \frac{K}{1 - \gamma}$.
\end{proof}

We next observe that there always exists a stationary Markov decision process (MDP) that is equivalent to our model.
Fix a user with demand function $f$ and friction $c$, and an app with content $\cI$ that returns revenues $\bset{R_i}_{i \in \cI}$ and user experiences $\bset{E_i}_{i \in \cI}$.
Let us construct an MDP $(\cS, \cA, \cP, \cR)$ with state space $\cS \asseq \reals \times \bset{0, 1}$, actions $\cA \asseq \cI \cup \bset{\emptyset}$, transition function $\cP$ and reward function $\cR$.
In the following, we will use $S_t$ to denote the MDP state and $a_t$ to denote the MDP action at timestep $t$.

At every state in the set $\reals \times \bset{0} \subset \cS$, only the actions in the set $\cI$ are available.
Each of these states, $(x, 0)$, corresponds to a timestep in which the user is interacting with the app and has a user state of $x$.
At every state in the set $\reals \times \bset{1} \subset \cS$, only a trivial action $\emptyset$ is available.
Each of these states, $(x, 1)$, corresponds to a timestep where the user declined to interact with the app.
We accordingly define the MDP's initial state to be $(0, 0)$.
We define the reward function as $\cR(a=\emptyset, S=(x, 1)) = 0$ and $\cR(a=i, S=(x, 0)) = R_i$ for all $x \in \reals, i \in \cI$.

We now define transition probabilities.
First, for all $x \in \reals$, we define $\mathbb P(S_{t+1} = (x, 0) \mid a_t = \emptyset, S_t = (x, 1)) = (1-c) f(x)$ and $\mathbb P(S_{t+1} = (x, 1) \mid a_t = \emptyset, S_t = (x, 1)) = 1 - (1-c) f(x)$ according to the probability that a user returns to an app after having left it.
Similarly, for all $x, y \in \reals$ and $i \in \cI$, we define $\mathbb P(S_{t+1} = (y, 0) \mid a_t = i, S_t = (x, 0)) = \Pr(E_i = y - x) f(y)$ and $\mathbb P(S_{t+1} = (y, 1) \mid a_t = i, S_t = (x, 0)) = \Pr(E_i = y - x) (1 - f(y))$ according to the probability that a user continues use of an app.
If $E_i$ is continuously valued, we can instead define the cumulative density function $\mathbb P(S_{t+1} \in \bset{ (y', 0),  y' \leq y} \mid a_t = i, S_t = (x, 0)) = \Pr(E_i \leq y - x) \EEc{f(y')}{y' \leq y}$ and $P(S_{t+1}  \in \bset{ (y', 1),  y' \leq y}  \mid  a_t = i, S_t = (x, 0)) = \Pr(E_i \leq y - x) \EEc{1 - f(y')}{y' \leq y}$.

There is a bijection between MDP policies and app policies that preserves their discounted payoff.
\begin{fact}
\label{fact:edpmdp}
There exists a bijective mapping $\phi$ from the MDP's policy space, i.e. functions of form $\pi: (\cS \times \cA \times \reals)^* \to \cA$, to the space of app policies in our model, i.e. functions of form $\pi: \reals^4 \to \cI$.
This map $\phi$ guarantees that for every MDP policy $\pi$, the MDP's discounted payoff $J'$ matches that of its counterpart in our model: $J'(\pi) = J(\phi(\pi))$.
$\phi$ also maps stationary MDP policies only to simple app policies in our model.
\end{fact}
\begin{proof}[Proof of \cref{fact:edpmdp}]
Fix a policy for the MDP: $\pi: (\cS \times \cA \times \reals)^* \to \cA$.
We can express $\pi$ as taking a transcript as input, which we can denote as $H = [(S_1, a_1, r_1), \dots, (S_{T-1}, a_{T-1}, r_{T-1})] \in (\cS \times \cA \times \reals)^*$.
We now define a mapping $\psi$ of MDP transcripts to transcripts in our model.
Fix any MDP transcript $H$, and define $(x_{t+1}, 1-s_t) = S_t$ for every $t \in [T]$.
We construct our mapping $\psi$ as $\psi(H) = \bset{(s_t, r_{t}, x_{t+1} - x_{t}, a_{t})}_{t \in [T-1]}$.

We can define our bijection $\phi$ as $\phi(\pi) = \pi(\psi^{-1}(\cdot))$; in other words, for every possible transcript of user-app interactions $H$, the policy $\phi(\pi)$ returns the action $\pi(\psi^{-1}(H))$.
Writing $J'$ to denote the discounted objective of the MDP and using $J$ as in \eqref{eq:platformdiscountedobjective}, by construction, we have $J'(\phi(\pi)) = J(\pi)$ for every policy $\pi$.
Similarly, we have $J(\phi^{-1}(\pi)) = J'(\pi)$ for every MDP policy $\pi$.

We can also confirm that stationary policies in the MDP indeed map to simple app policies.
This is because every policy in the MDP only plays non-trivial actions, i.e., $a_t \neq \emptyset$, at states that belong in the set $\reals \times \bset{0}$.
In those states, since a stationary policy's action at a timestep $t$ depends only on the current state $S_{t-1} = (x_t, 0)$, the stationary policy's action---and by extension its corresponding app policy's action---only depends on the user's state $x_t$.
\end{proof}

We thus have that the existence of an optimal app policy for in our model implies the existence of an optimal policy for the MDP.
Moreover, any deterministic and stationary policy for the MDP implies a simple app policy attaining the same objective value.
The following statement concerning policy iteration concludes our proof.

\begin{fact}[\citet{Bertsekas}]
    \label{fact:stationaryoptimal}
    In a stationary MDP, for any (potentially history dependent and non-deterministic) policy, there exists a stationary and deterministic policy with at least as high an expected discounted payoff, on every initial state distribution.
\end{fact}
\end{proof}

\if\arxiv0
\subsection{Proof of \cref{theorem:cputime}}
\label{section:cputimeproof}
\cputime*
\begin{proof}
    First, we define a function that maps each timestep to the number of user-app interactions that have occurred up to and including that timestep: $\mathsf{NumInteractions}(t) = \sum_{\tau=1}^t s_\tau$.
    We also use $\mathbb{T} = \bset{t \in \integers_+ \mid s_t = 1}$ to denote the set of timesteps where user-app interactions took place.
    Note that $\mathsf{NumInteractions}$ is bijective when restricted to $\mathbb{T}$.
    Since no revenue is generated on timesteps without user engagement (i.e. $r_t = 0$ when $t \not\in \mathbb{T}$), we can write \cref{eq:platformdiscountedobjective} as a sum over $\mathbb{T}$:
    \begin{align*}
        J(\pi) = \EEs{\bset{s_t, r_t, e_t, i_t}_t(\pi)}{\sum_{t=1}^\infty s_t \gamma^{t-1} r_t} = \EEs{\bset{s_t, r_t, e_t, i_t}_t(\pi)}{\sum_{t \in \mathbb{T}} \gamma^{t-1} r_t}.
    \end{align*}
    Since the mapping $\mathsf{NumInteractions}$ is bijective on $\mathbb{T}$, its inverse $\mathsf{NumInteractions}^{-1}$ is well-defined in the above series, and we have $\bset{\mathsf{NumInteractions}^{-1}(t) \mid t \in \mathbb{T}} = \integers_+$.    
    We can therefore re-index the sum over $\mathbb{T}$ in terms of the number of interactions rather than the original timesteps:
    \begin{align*}
        J(\pi) = \EEs{\bset{s_t, r_t, e_t, i_t}_t(\pi)}{\sum_{t \in \mathbb{T}} \gamma^{t-1} r_t} = \EEs{\bset{s_t, r_t, e_t, i_t}_t(\pi)}{\sum_{t = 1}^\infty \gamma^{\mathsf{NumInteractions}^{-1}(t) - 1} \ite{r}{t}},
    \end{align*}
    where $\ite{r}{t}$ denotes the realized revenue at the $t$-th interaction rather than the $t$-th timestep.
    By linearity, we can use telescoping to simplify further:
    \begin{align*}
        J(\pi) = \sum_{t = 1}^\infty \EEs{s_t, r_t, e_t, i_t}{\ite{r}{t} \gamma^{\mathsf{NumInteractions}^{-1}(t)-1}} = \sum_{t = 1}^\infty \EEs{s_t, r_t, e_t, i_t}{\ite{r}{t} \prod_{\tau=2}^t \gamma^{\ite{w}{\tau}} },
    \end{align*}
    where $\ite{w}{\tau} = \mathsf{NumInteractions}^{-1}(\tau) - \mathsf{NumInteractions}^{-1}(\tau - 1)$ represents the number of timesteps between the $(\tau-1)$-th and $\tau$-th ineteractions.
    Here, we have used the fact that the first timestep always corresponds to a user-app interaction, so $\mathsf{NumInteractions}^{-1}(1) = 1$.
    
    The random variable $\ite{w}{\tau}$ represents the amount of time that passes between the user's $(\tau-1)$th and $\tau$th interactions.
    Under a simple policy, $\ite{w}{\tau}$ is conditionally independent of $\ite{w}{\tau-1}$ given the user's state after the $(\tau - 1)$-th interaction, i.e. $\ite{x}{\tau }$.
    
    If there were no friction, $\ite{w}{\tau}$ would follow a standard geometric distribution with parameter $f(\ite{x}{\tau })$.
    Recall that a geometric distribution with parameter $p$ models the number of coin flips (with heads probability $p$) required to achieve the first head.
    When friction $c$ is present, $\ite{w}{\tau}$ follows a non-homogeneous geometric distribution with the probability mass function $\Pr(\ite{w}{\tau} = k) = (1-f_k(\ite x \tau)) \prod_{j=1}^{k-1} f_j(\ite x\tau)$ defined on the support $k \in \naturals$, where we define $f_1 = f$ and $f_n = c f$ for all $n > 1$.
    We denote this distribution by $\text{Geo}(p, c)$.
    One can compute the moment generating function of this non-homogeneous geometric distribution as
    $$\EEs{X \sim \text{Geo}(p, c)}{\exp(tX)} =  p \cdot \exp(t)  + \frac{1 - p}{1 -  \exp(t) (1 - (1-c) p)} (1-c) \cdot p \cdot \exp(t)^2$$
    for $t < - \ln(1 - (1-c)p)$. This result enables us to simplify the expectation $$\EEc{\gamma^{w_t}}{x_t} =   f(x_t) \gamma + \frac{1 - f(x_t)}{1 - \gamma (1 - (1-c) f(x_t))} (1-c) f(x_t) \gamma^2.$$
    Notice we can express this as $\EEc{\gamma^{w_t}}{x_t} = \gamma \cdot \tilde f(x_t)$.
    Finally, by the law of total expectation, we recover
    \begin{align*}
		  J(\pi) = \sum_{t = 1}^\infty \EEs{s_t, r_t, e_t, i_t}{\ite{r}{t} \prod_{\tau=2}^t \EEc{\gamma^{\ite{w}{\tau}}\, }{\, \ite{x}{\tau}} } = \EEs{\bset{s_t, r_t, e_t, i_t}_t}{\sum_{t = 1}^\infty  \ite{r}{t}{\prod_{\tau=2}^t \gamma \tilde f\left(\ite{x}{\tau}\right)}}.
    \end{align*}
\end{proof}
\fi

\subsection{Existence of Simple Optimal Policies}
An optimal app policy can be shown to exist under standard and mild assumptions inherited from the theory of Markov decision processes.
\begin{proposition}
\label{proposition:optimalexistence}
There is an optimal simple app policy in our model if at least one of the following conditions hold:
\begin{enumerate}
    \item All contents provide deterministic user experiences, i.e. $\bset{E_i}_{i \in \cI}$ are deterministic.
    \item The app chooses from a finite set of content.
    \item The set of app content is compact and 
    the CDF of content user experiences, $\Pr(E_i \leq z)$, is continuous in the content $i \in \cI$.
\end{enumerate}
\end{proposition}
\begin{proof}
    The MDP is trivial on states in the set $\reals \times \bset{1}$.
    It therefore suffices for us to define a Bellman operator exclusively on the state space $\reals \times \bset{0}$, which we will map onto the reals for convenience.
    For any function $W: \reals \to \reals$, we define the Bellman operator $T$ as
	\begin{align*}
		(TW)(\sat) & = \max_{i \in \arms} \EEs{e \sim E_i}{\EEs{\substack{n \sim N(\sat, e) \\ r \sim \cR(a=i, S=(x, 0))}}{{r + \gamma^n W(\sat + e)}}},
	\end{align*}
    where we define $N(x, y)$ as the random variable that is the number of timesteps that the MDP takes to reach the state $(x + y, 0)$ from the state $(x, 0)$, conditioned on the MDP reaching either the state $(x + y, 0)$ or $(x + y, 1)$ immediately after state $(x, 0)$.
    \cref{fact:nondecreasingbellman} and \cref{fact:inductivebellman} prove two important properties of this operator.
    \begin{fact}
    \label{fact:nondecreasingbellman}
    Under any of the listed assumptions, $TW(x)$ is well-defined for every monotonically non-decreasing $W$ and for every $x \in \reals$.
    \end{fact}
    \begin{proof}
    In Case 1, when user experiences are deterministic, the operator can be written as 
	\begin{align*}
		(TW)(\sat) & = \max_{i \in \arms} \EEs{\substack{n \sim N(\sat, E_i) \\ r \sim \cR(a=i, S=(x, 0))}}{{r + \gamma^n W(\sat + E_i)}},
	\end{align*}
    where we note that $N(\sat, E_i)$ is monotonically non-increasing in $E_i$ as the user demand functions are monotonically non-decreasing and $\gamma^n$ is monotonically non-increasing in $n$ as $\gamma < 1$.
    We also have that $W(\sat + E_i)$ is monotonically non-decreasing in $E_i$ by assumption.
    It thus follows that the expectation inside the maximum, namely $$\EEs{\substack{n \sim N(\sat, E_i) \\ r \sim \cR(a=i, S=(x, 0))}}{{r + \gamma^n W(\sat + E_i)}},$$ is non-decreasing in $E_i$.
    Since $\bset{E_i}_{i \in \arms}$ is compact, the maximum exists and hence $TW(x)$ exists.

    In Case 2, it is also obvious that $TW(x)$ exists when contents are finite.

    In Case 3, the below expectation is necessarily continuous $$g(i) = \EEs{e \sim E_i}{\EEs{\substack{n \sim N(\sat, e) \\ r \sim \cR(a=i, S=(x, 0))}}{{r + \gamma^n W(\sat + e)}}}.$$
    Since $\arms$ is a compact set, the existence of $TW(x)$ follows by the extreme value theorem.
    \end{proof}
    
    \begin{fact}
    \label{fact:inductivebellman}
    If $W$ is monotonically non-decreasing and bounded below by zero, then $TW$ is also monotonically non-decreasing and bounded below by zero.
    \end{fact}
    \begin{proof}
        By assumption $f(x)$ is monotonically non-decreasing in $x$.
        Thus, $N(x, E_i)$ is weakly stochastically dominated by $N(x', E_i)$ if $x > x'$.
        Since $\gamma^n$ is non-increasing in $n$ and the function $W \geq 0$ is bounded below by zero for all arguments, it also follows that $TW$ is monotonically non-decreasing.
        Since there always exists a content with non-negative revenue, i.e. $\exists i \in I$ such that $\EE{R_i} \geq 0$, it is also true that $TW$ remains bounded below by zero.
    \end{proof}

    By the previous facts, if we repeatedly apply the operator $T$ to a  monotonically non-decreasing function $W$ bounded below by zero, say the all zeros function $W(x) = 0$, we will always obtain a well-defined monotonically non-decreasing function bounded below by zero.
    A standard argument of the $\gamma$-contractiveness of $T$ and appeal to Brouwer's fixed point theorem \cite{puterman} then directly implies the existence of a monotonically non-decreasing non-negative function $W^*$ such that $TW^* = W^*$.
    There must therefore exist an optimal stationary policy $\pi$  defined as $$\pi(x) = \argmax_{i \in \arms} \EEs{e \sim E_i}{\EEs{\substack{n \sim N(\sat, e) \\ r \sim \cR(a=i, S=(x, 0))}}{{r + \gamma^n W(\sat + e)}}}.$$
    By \cref{lemma:simpleoptimal}, there therefore also exists a simple app policy that is optimal.
\end{proof}

\subsection{Non-Existence of Average Reward Objective}
A common alternative to studying discounted reward objectives is studying the average reward objective.
In our model, however, the average reward objective may not be well-defined for some stationary policies.
\begin{proposition}
    \label{proposition:averagereward}
    There exists an instance of our model in which a simple policy does not have a well-defined expected average reward:
    $$\lim_{T \to \infty} \EEs{\bset{s_t, r_t, e_t, i_t}_t}{\frac{1}{T} \sum_{t=1}^T r_t}.$$
\end{proposition}
\begin{proof}
Consider an instance of our model where the app creator chooses from two pieces of content $I = \bset{a, b}$, each of which give deterministic user experiences $E_a = E_b = 1$ and revenues $R_a = -1$ , $R_b = 1$.
Let $f(0) = 1$ for simplicity.
Consider the following simple policy $\pi: \reals \to \bset{a, b}$ where for each $q \in \integers$ and every $x \in (3^q, 3^{q+1}]$, we set $\pi(x) = a$ if $q$ is odd and $\pi(x) = b$ if $q$ is even.
That is, the policy alternates between showing content $a$ for increasing long periods to showing content $b$ for increasing long periods, such that the empirical content distribution has no limit.

The deterministic trajectory of this policy will be user states $1, 2, \dots$ and revenues $$-1, +1, +1, -1, -1, -1, -1, -1, -1, +1, \dots.$$
Direct computation gives that the average reward series alternates between values of $-0.5$ and $0.5$ with $\liminf_{T \rightarrow \infty} \frac{1}{T} \sum_{t=1}^T r_t \leq -0.5$ and $\limsup_{T \rightarrow \infty} \frac{1}{T} \sum_{t=1}^T r_t \geq 0.5$.
The limit that is expected average reward therefore does not exist.
\end{proof}

\section{Omitted Proofs and Additional Results for \cref{section:optimal}}
\subsection{Greedy Revenue Optimization}
Consider a user that already has a user state $x_t$ that is at an extreme value, whether it is extremely high (user demand is near maximum) or extremely low (user demand is near minimum).
We will show that app creators are always incentivized to act greedily with such users and show the highest revenue-earning content.

\begin{proposition}
    \label{proposition:extremesexploit}
    Suppose the user experiences provided by app content are of finite variance. 
    For any $\epsilon > 0$, there is a threshold $x^* \in \reals$ where for all higher user states $\ite{x}{t} \geq x^*$, the app creator has an $\epsilon$-optimal simple policy that shows the highest revenue earning content in the next timestep, i.e. $\ite{i}{t+1} = \argmax_{i \in \cI} \EE{R_i}$. 
    Similarly, there exists a threshold $x^* \in \reals$ where for all lower user states $\ite{x}{t} \leq x^*$, the app creator is again $\epsilon$-incentivized to show the highest revenue earning content.
\end{proposition}
\begin{proof}
We want to show that for any $\epsilon > 0$, for sufficiently large or small $x$,
$$Q(\argmax_{i \in \cI} \EE{R_i}, x) \geq \max_{i \in \cI} Q(i, x) - \epsilon.$$

First, we recall that, by the compactness of the revenue expectations $\bset{\EE{R_i}}_{i\in\arms}$, we can upper bound the amount of revenue that the app creator can attain in a single interaction by some constant value $K \in \reals$.
It therefore follows that the value function at any user state $x \in \reals$ is upper bounded by the geometric series $V(x) \leq \frac{K \gamma}{1 - \gamma}$.
We also know the value function is non-negative and monotonically non-decreasing at all user states $x \in \reals$.
Since the domain of the value function $V$ is $\reals$ yet the range of $V$ is bounded, we know that for every $\delta > 0$ and $\epsilon > 0$ there exists a threshold $x^*$ such that for all higher user states $x > x^*$, we have
\begin{align*}
    V(x) - V(x - \delta) \leq \epsilon.
\end{align*}
For convenience, we now define $i_{\mathrm{greedy}} \asseq \argmax_{i \in \arms} \EE{R_i}$ to be the content that greedily maximizes expected revenue, which must exist by the compactness of $\bset{\EE{R_i}}_{i \in \arms}$.
We now similarly define $i_{\mathrm{opt}} \asseq \argmax_{i \in \arms} Q(i, x)$ to be the content that the app shows if it continues executing some optimal policy.
We can write the difference in the Q-function values of showing the optimal content $i_\mathrm{opt}$ and showing the greedy content $i_\mathrm{greedy}$ as
\begin{align*}
    &Q(i_{\mathrm{opt}}, x) - Q(i_{\mathrm{greedy}}, x) \\
    &= \tilde f(x) \gamma \para{{\EE{R_{i_{\mathrm{opt}}}} -  \EE{R_{i_{\mathrm{greedy}}}} + \EE{V(x + E_{i_{\mathrm{opt}}})} - \EE{V(x + E_{i_{\mathrm{greedy}}})}}} \\
    &\leq \tilde f(x) \gamma \para{\EE{V(x + E_{i_{\mathrm{opt}}}) - V(x + E_{i_{\mathrm{opt}}} - (E_{i_{\mathrm{opt}}} - E_{i_{\mathrm{greedy}}}))}}\\
    &\leq {\EE{V(x + E_{i_{\mathrm{opt}}}) - V(x + E_{i_{\mathrm{opt}}} - (E_{i_{\mathrm{opt}}} - E_{i_{\mathrm{greedy}}}))}}.
\end{align*}
In the final inequality, we use the fact that $Q(i_{\mathrm{opt}}, x) \geq Q(i_{\mathrm{greedy}}, x)$ by definition of $i_\mathrm{opt}$ and the fact that $\tilde{f}(x) \leq 1$ and $\gamma \leq 1$.

We now use the fact that $E_{i_{\mathrm{opt}}}$ and $E_{i_{\mathrm{greedy}}}$ each have bounded variance and finite means to invoke Chebyshev's inequality.
Chebyshev's inequality gives that, for any probability $p \in (0, 1)$, there exists a $\delta'$ such that $\abs{E_{i_{\mathrm{opt}}}} \leq \delta'$ with probability at least $1-p/2$.
We can similarly apply Chebyshev's inequality to the random variable $E_{i_{\mathrm{opt}}} - E_{i_{\mathrm{greedy}}}$, which is also of bounded variance and finite mean.
By Chebyshev's inequality, for any $p > 0$, there exists a $\delta$ such that $\abs{E_{i_{\mathrm{opt}}} - E_{i_{\mathrm{greedy}}}} \leq \delta$ with probability at least $1-p/2$.
Suppose that, for some fixed choice of $\delta, \delta'$, we condition on the event $\abs{E_{i_{\mathrm{opt}}} - E_{i_{\mathrm{greedy}}}} \leq \delta$  and $\abs{E_{i_{\mathrm{opt}}}} \leq \delta'$.
It then immediately follows that there exists some user state $x^*$ such that for all greater user states $x > x^*$,
\begin{align*}
    V(x + E_{i_{\mathrm{opt}}}) - V(x + E_{i_{\mathrm{opt}}} - (E_{i_{\mathrm{opt}}} - E_{i_{\mathrm{greedy}}}))
    \leq V(x + E_{i_{\mathrm{opt}}}) - V(x + E_{i_{\mathrm{opt}}} - \delta)
    \leq \epsilon / 2.
\end{align*}
Applying a union bound, we know that there is some choice of $\delta, \delta'$ such that the aforementioned event $\abs{E_{i_{\mathrm{opt}}} - E_{i_{\mathrm{greedy}}}} \leq \delta$  and $\abs{E_{i_{\mathrm{opt}}}} \leq \delta'$ occurs with probability at least $1-p$.
When this event does not occur, we can still apply a deterministic upper bound to the value function of $$V(x + E_{i_{\mathrm{opt}}}) - V(x + E_{i_{\mathrm{opt}}} - (E_{i_{\mathrm{opt}}} - E_{i_{\mathrm{greedy}}})) \leq \frac{K \gamma}{1 - \gamma}.$$
Thus, if we choose the probability $p = \frac{\epsilon}{2} \frac{1 - \gamma}{K \gamma}$, we have that there exists a user state $x^*$ such that for all $x \geq x^*$
\begin{align*}
    Q(i_{\mathrm{opt}}, x) - Q(i_{\mathrm{greedy}}, x) &\leq {\EE{V(x + E_{i_{\mathrm{opt}}}) - V(x + E_{i_{\mathrm{opt}}} - (E_{i_{\mathrm{opt}}} - E_{i_{\mathrm{greedy}}}))}} \\
    &\leq (1-p) \frac{\epsilon}{2} + p \frac{K \gamma}{1-\gamma} \\
    &\leq \epsilon,
\end{align*}
where the second inequality applies the law of total expectation.

The second half of the claim---the existence of a threshold $x^*$ such that for all lower user states $x \leq x^*$ apps are also at least $\epsilon$-incentivized to show the highest revenue content---follows identically.
\end{proof}

\subsection{Proof of \cref{lemma:piecewiseconstant}}
\piecewiseconstant*

\begin{proof}
We refer to the sequence of user states $\ite{x}{1}, \ite{x}{2}, \dots$ that result from executing a simple policy $\pi$ as the \emph{user state trajectory} of $\pi$.
In the linear setting, this trajectory is deterministic.
We similarly refer to the sequence of content $\ite{i}{1}, \ite{i}{2}, \dots$ that result from executing a simple policy $\pi$ as the \emph{action trajectory} of $\pi$. 
We will use the shorthand $D$ to denote the set of discontinuities of the demand function $f$.
Our proof consists of three steps.
We will first establish that the set of content maximizing the Q-function at each state is compact.
We then accordingly construct a specific instance of a simple optimal policy $\pi^*$.
We conclude by proving our four claims.

We begin by remarking that, since the demand function $f$ is piecewise linear and right-continuous with finite discontinuities, $\tilde f$ is also piecewise linear and right-continuous with finite discontinuities.
\begin{fact}
    The function $\tilde f$ is piecewise linear and right-continuous with $\setsize{D} < \infty$ discontinuities.
\end{fact}
We next observe that any simple optimal policy should only ``hold'' user state constant at user states that correspond to discontinuities of the user demand function $f$.
In the sequel, we will use $V_\pi(x)$ to denote the discounted payoff that results from executing a simply policy $\pi$ on a user with initial state $x$.
Observe that the value function can be written as $V(x) = \max_{\pi} V_\pi(x)$.
We will also write $Q_\pi(i, x)$ to denote the discounted payoff that results from showing a content $i$ to a user with initial state $x$ and thereafter executing a simple policy $\pi$.
\begin{lemma}
    \label{lemma:holdingdiscontinuity}
    A simple optimal policy $\pi$ should only keep user states constant at discontinuities of $f$: that is, if for some $x \in \reals$ and simple policy $\pi$, the equality $\pi(x) = C_E$ holds and $\pi$ is optimal, then $x \in D$.
\end{lemma}
\begin{proof}
    Fix a user state $x \in \reals$ and suppose to the contrary that it is not a discontinuity: $x \notin D$.
    Since there are only finite discontinuities, there must exist a lower level of user state $x' < x$ such that user engagement is unchanged: $f(x) = f(x')$.
    Moreover, we can choose $x'$ so that we also have that $x'$ is reachable from $x$ within one timestep: $x' \geq x + C_E - K$.
    For example, we can define $x' = \max \bset{d \in D \mid d < x} \cup \bset{x + C_E - K}$.
    We will use $i' \in (0, K]$ to denote the content that one shows to lower user state from $x$ to $x'$ in a single timestep.
    
    Let $\pi$ be a simple policy that holds user state constant at $x$, i.e. set $\pi(x) = C_E$ as $E_{C_E} = 0$.
    We will design a policy $\pi'$ to be a witness to the suboptimality of $\pi$ by attaining a strictly larger discounted objective value at an initial user state of $\ite{x}0 = x$.
    We define this policy $\pi'$ as first showing the content $i'$ to lower user state from $x$ to $x'$, and then holding user state constant into perpetuity by repeatedly showing the content $i_{\mathrm{hold}} = C_E$.
    This is a valid construction as $i' \in [-K, K]$ by construction.
    Using the geometric series identity, we can explicitly write
    $$V_{\pi}(x) = \frac{C_R + C_E}{1 - \tilde{f}(x)\gamma }, \quad V_{\pi'}(x) =\frac{C_R + C_E}{1 - \tilde{f}(x)\gamma } +x - x',$$
  where we use the fact that $\tilde{f}(x') = \tilde{f}(x)$.
  Since $x > x'$ by construction, we have as desired that $V_{\pi'}(x) > V_{\pi}(x)$ and a witness to the suboptimality of $\pi$.
\end{proof}

The next lemma constrains the possible trajectories of simple optimal policies when the user state trajectory does not coincide with the discontinuities $D$.
\trajectory*
\begin{proof}
    We begin by proving the first claim.
    Suppose to the contrary that for some intermediate user state $\ite{x}0 \in \reals$ (which we re-index to timestep 0 for notational convenience), applying the policy $\pi$ results in a user state trajectory with the subsequence $\ite{x}1, \ite{x}2, \dots, \ite{x}T$ where the following conditions simultaneously hold:
    \begin{enumerate}
        \item The user states $\ite{x}2, \dots, \ite{x}T$ do not coincide with any of the discontinuities of $f$.
        \item $\pi$ does not maximally increase user state at the end of the subsequence: $\pi(\ite{x}T) > -K$.
        \item At some $t \in [T-1]$, the app does not maximally decrease user state: $\pi(\ite{x}t) < K$.
    \end{enumerate}
    There must exist some $\delta > 0$ so that $\pi(\ite{x}t) + \delta \leq K$, $\pi(\ite{x}T) - \delta \geq -K$ and $f(\ite{x}{\tau} - \delta) = f(\ite{x}{\tau})$ for all $\tau \in [t-1, T]$.
    In fact, we can simply choose $$\delta \asseq \min \bset{K - \pi(\ite{x}t), \pi(\ite{x}T) + K} \cup \bset{0.5 \cdot |{d - \ite{x}\tau}| \mid \tau \in [t+1, T], d \in D}.$$
    This $\delta$ is bounded away from zero since we assume that $\ite{x}2, \dots, \ite{x}T$ are not discontinuities of $f$ and well-defined since there are only finite discontinuities in $D$.

    We now construct a policy $\pi'$ which yields the trajectory $\ite{x}1, \dots, \ite{x}t, \ite{x}{t+1} - \delta, \dots, \ite{x}T - \delta$ and resumes the trajectory of $\pi$ at the $(T+1)$th timestep.
    This policy $\pi'$ is valid because we chose $\delta$ so as to ensure that $\pi(\ite{x}t) + \delta \leq K$ and thus $\pi'(\ite{x}t) \leq K$.
    Similarly, we chose $\delta$ so that $\pi(\ite{x}T) - \delta \geq -K$  and by extension $\pi'(\ite{x}T) \geq -K$.
    Thus, at the only two timesteps in which the actions of policy $\pi$ differ from $\pi'$, timesteps $t$ and $T$, we are still guaranteed the new actions are legal: $\pi'(\ite{x}t) \in [-K, K]$ and $\pi'(\ite{x}T) \in [-K, K]$.
    
    We can verify this policy $\pi'$ guarantees a strictly higher discounted payoff than $\pi$, leading to a contradiction with the optimality of $\pi$.
    Let us write the user state trajectory of the original policy $\pi$ as $\ite{x}1, \ite{x}2, \dots$ and that of the modified policy $\pi'$ as $\ite{x}1, \dots, \ite{x}{t-1}, \ite{\hat x}{t}, \dots, \ite{\hat x}{t}, \ite{x}{T+1}, \dots$.
    Let us also write the original actions trajectory of policy $\pi$ as $i_1, i_2, \dots$.
    Then we can express each of the payoffs of $\pi$ and $\pi'$ as
    \begin{align*}
        V_{\pi}(\tsv x0) &= \sum_{w=1}^{\infty} \para{\prod_{\tau=2}^w \gamma \tilde f(\ite{x}\tau)} (i_w + C_R),  \\
         V_{\pi'}(\tsv x0) &= \sum_{w=1}^{t-1} \para{\prod_{\tau=2}^w \gamma \tilde f(\ite{x}\tau)} (i_w + C_R)  + \para{\prod_{\tau=2}^{t} \gamma \tilde f(\ite{x}\tau)}  (\ite{i}t + \delta + C_R) \\
       & +  \sum_{w=t+1}^{T-1} \para{\prod_{\tau=2}^{t} \gamma \tilde f(\ite{x}\tau)} \para{\prod_{\tau=t+1}^w \gamma \tilde f(\ite{\hat x}{\tau})} (\ite{i}w + C_R)\\ 
        & + \para{\prod_{\tau=2}^{t} \gamma \tilde f(\ite{x}\tau)}  \para{\prod_{\tau=t+1}^T \gamma \tilde f(\ite{\hat x}{\tau})}  (\ite{i}T - \delta + C_R) \\
    &    +  \sum_{w=T+1}^{\infty} \para{\prod_{\tau=2}^t \gamma \tilde f(\ite{x}\tau)} \para{\prod_{\tau={t+1}}^T \gamma \tilde f(\ite{\hat x}{\tau})} \para{\prod_{\tau=T+1}^t \gamma \tilde f(\ite{x}\tau)} (\ite{i}w + C_R).
    \end{align*}

    Since we chose $\delta$ to guarantee that $\tilde{f}(\ite{\hat x}{\tau}) = \tilde{f}(\ite{x}\tau)$ at all timesteps $\tau \in [t, T]$, we can thus simplify the payoff of policy $\pi'$ as
    \begin{align*}
        V_{\pi'}(\tsv x0) &= \sum_{w=1}^{\infty} \para{\prod_{\tau=2}^w \gamma \tilde f(\ite{x}\tau)} (\ite{i}w + C_R) +  \delta\para{\prod_{\tau=2}^t \gamma \tilde f(\ite{x}\tau)} \para{1 - \para{\prod_{\tau \in [t+1, T]} \gamma \tilde f(\ite{x}\tau)}}  \\
        & =  V_{\pi}(\tsv x0) +  \delta\para{\prod_{\tau=2}^t \gamma \tilde f(\ite{x}\tau)} \para{1 - \para{\prod_{\tau \in [t+1, T]} \gamma \tilde f(\ite{x}\tau)}} \\
        & >  V_{\pi}(\tsv x0).
    \end{align*}
    Thus, $\pi'$ is thus a witness to the suboptimality of $\pi$.

    We can very similarly prove the second claim.
    Suppose to the contrary that for some choice of initial user state, applying the policy $\pi$ results in a user state trajectory with the initial subsequence $\ite{x}1, \ite{x}2, \dots, \ite{x}T$ where the following conditions simultaneously hold:
    \begin{enumerate}
        \item The user states $\ite{x}2, \dots, \ite{x}T$ do not coincide with any of the discontinuities of $f$.
        \item The policy $\pi$ does not maximally decrease user state as its first action: $\pi(\ite{x}1) < K$.
        \item At some $t \in [2, T]$, the app does not maximally increase user state: $\pi(\ite{x}t) > -K$.
    \end{enumerate}
    As before, there must exist some small value $\delta > 0$ so that $\pi(\ite{x}t) - \delta \geq -K$, $\pi(\ite{x}1) + \delta \leq K$ and $f(\ite{x}{2} - \delta) = f(\ite{x}{2}), \dots f(\ite{x}{t} - \delta) = f(\ite{x}{t})$.
    This time, we construct a (non-simple) policy $\pi'$ which executes the trajectory $\ite{x}1, \ite{x}2 - \delta, \dots, \ite{x}t - \delta, \ite{x}{t+1}, \dots, \ite{x}T$ and then resumes the trajectory of $\pi$ starting at the $(t+1)$st timestep.
    Again, we note that this policy is valid, since we chose $\delta$ so as to ensure that $\pi(\ite{x}1) + \delta \leq K$ and $\pi(\ite{x}t) - \delta \geq -K$.
    
    We also observe this policy $\pi'$ guarantees a strictly higher objective value than $\pi$, leading to a contradiction with the optimality of $\pi$.
    Let us write the user state trajectory of the original policy $\pi$ as $\ite{x}1, \ite{x}2, \dots$ and that of the modified policy $\pi'$ as $\ite{x}1, \ite{x}2, \dots, \ite{\hat x}{t}, \ite{x}{t+1}, \dots$.
    Let us also write the actions trajectory of policy $\pi$ as $\ite{i}1, \ite{i}2, \dots$.
    Then we can express each of the objectives of $\pi$ and $\pi'$ as
    \begin{align*}
        V_{\pi}(\tsv x0) &= \sum_{w=1}^{\infty} \para{\prod_{\tau=2}^w \gamma \tilde f(\ite{x}\tau)} (\ite{i}w + C_R),  \\
        V_{\pi'}(\tsv x0) &= 
        \ite{i}1 + \delta + C_R \\
        & + \sum_{w=2}^{t-1} \para{\prod_{\tau=2}^w \gamma \tilde f(\ite{x}{\tau})} (\ite{i}w + C_R) \\
        & +   
 \para{\prod_{\tau=2}^{t} \gamma \tilde f(\ite{x}{\tau})} (\ite{i}t - \delta + C_R) \\
       & +  \sum_{w=t+1}^{\infty}  \para{\prod_{\tau=2}^{t} \gamma \tilde f(\ite{x}{\tau})} \para{\prod_{\tau=t+1}^w \gamma \tilde f(\ite{x}\tau)} (\ite{i}w + C_R).
    \end{align*}
    We chose $\delta$ so that $f(\ite{x}\tau - \delta) = f(\ite{x}\tau)$ for all timesteps $\tau \in [2, t]$, meaning we can simplify the payoff of policy $\pi'$ as
    \begin{align*}
       V_{\pi'}(\tsv x0) &= \sum_{w=1}^{\infty} \para{\prod_{\tau=2}^w \gamma \tilde f(\ite{x}\tau)} (\ite{i}w + C_R) +\delta \para{1 - \para{\prod_{\tau \in [2, t]} \gamma \tilde f(\ite{x}\tau)}}  \\
        & = V_{\pi}(\tsv x0) +  \delta \para{1 - \para{\prod_{\tau \in [2, t]} \gamma \tilde f(\ite{x}\tau)}}  \\
        & > V_{\pi}(\tsv x0).
    \end{align*}
\end{proof}
In the next lemma, we confirm that there is no simple optimal policy that can result in an action trajectory that infinitely boosts user state.
\begin{lemma}
    \label{lemma:suboptimalmaximalup}
    If there is a simple policy $\pi$ and initial user state $x \in \reals$ such that the action trajectory of $\pi$ has a limit of maximally boosting user state, i.e. $\lim_{ t\to \infty} \ite{i}t = -K$, $\pi$ cannot be an optimal policy.
\end{lemma}
\begin{proof}
    By definition, for any $\delta > 0$, there is some finite time $T$ past which, for all timesteps $t \geq T$, the actions played are almost maximally boosting user state: $\ite{i}t \leq -K + \delta$.
    Since there are finite discontinuities, it follows that past some finite time $T'$, all future timesteps $t > T'$ result in user states $\ite{x}{t}$ exceeding the largest discontinuity: $\ite{x}{t} \geq \max D$.
    We thus have that past the timestep $\max \bset{T', T}$, the policy $\pi$ will always play an action $\ite{i}t$ where $\ite{i}t < C_E$, despite user state already resulting in the maximum engagement probability, i.e. $\ite{x}t \geq \max D$.
    Observe that $\ite{x}{\max \bset{T', T}}$ is a deterministic user state.
    
    We have that the payoff of the policy $\pi$ with initial user state $\ite{x}{\max \bset{T', T}}$ is strictly upper bounded by that of any policy $\pi'$ which keeps user state constant: $\pi'(\ite{x}{\max \bset{T', T}}) = C_E$.
    \begin{align*}
     V_\pi({\ite{x}{\max\bset{T', T}}})
     & \leq (C_R - K + \delta)\frac{1}{1 - \gamma \tilde{f}(\ite{x}{\max \bset{T', T}})} \\
     & <  (C_R - K)\frac{1}{1 - \gamma \tilde{f}(\ite{x}{\max \bset{T', T}})} \\
     & = 
     V_{\pi'}({\ite{x}{\max\bset{T', T}}}).
    \end{align*}
    Thus, $\pi$ is suboptimal.
\end{proof}

We can use the previous lemmas to then prove our main intermediate step, establishing the compactness of the set of optimal actions at every user state.
\begin{lemma}
    \label{lemma:compactactions}
At every level of user state $x \in \reals$, the set of optimal actions optimizing the Q function is finite. That is, $\argmax_{i \in [-K, K]} Q(i, x)$ is finite.
\end{lemma}
\begin{proof}
    We will fix any user state $x \in \reals$ and denote the set of optimal actions at $x$ with the shorthand $S \asseq \argmax_{i \in [-K, K]} Q(i, x)$.
    We argue that every action $i \in S$ must either maximally increase or decrease user state, i.e. $i = K$ or $i = -K$, or satisfy the following inequality for at least one choice of discontinuity $d \in D$: $0 = \abs{x + C_E + i - d} \pmod K$.
    Note that since there are finitely many discontinuities, the above assertion directly implies the finiteness of $S$.

    Suppose to the contrary that there is an $i \in S$ that does not satisfy any of the above conditions.
    Then the user state $x + C_E - i$ cannot be a discontinuity of $f$ and the content $i < K$ cannot be maximally decreasing user state.
    Our trajectory lemma (\cref{lemma:trajectory}) therefore says that there must exist an optimal policy $\pi$ where $\pi(x + C_E - i) = -K$.
    By recursive application of \cref{lemma:trajectory}, the action trajectory of the policy $\pi$ at an initial user state of $x + C_E - i$ must repeat the action $i=-K$ until user state trajectory happens to land on a discontinuity of $f$.
    However, since we have assumed that there is no $d \in D$ for which $0 = \abs{x + C_E + i - d} \pmod K$, the user state trajectory will never land on a discontinuity, and the action trajectory must be an infinite repetition of the action $-K$.
    We reach a contradiction by \cref{lemma:suboptimalmaximalup}, as $\pi$ therefore cannot be optimal.
\end{proof}

We now let $\pi^*$ denote the simple optimal policy that, for every user state, chooses the content that maximizes the Q function while tie-breaking in favor of actions that decreases user state, that is $\pi(x) = \max \argmax_{i \in [-K, K]} Q(i, x)$.
\cref{lemma:compactactions} ensures that $\pi^*$ is well-defined.

Next, we prove the final technical lemma of this proof, which states that, regardless of which pair of initial user states one chooses, the user state trajectories of $\pi^*$ will never ``cross''.
\begin{lemma}
    \label{lemma:crossing}
    Let $\bset{\ite{x}t}_{t \in \naturals}$ denote the user state trajectory of the policy $\pi^*$ starting at an initial user state of $\ite{x}0$ and let $\bset{\ite{\hat{x}}{t}}_{t \in \naturals}$ denote the user state trajectory of $\pi^*$ starting at an initial user state of $\ite{\hat{x}}{0}$.
    The two trajectories will never cross.
    That is, if $\ite{x}0 < \ite{\hat{x}}{0}$, then for all $t \in \naturals$, $\ite{x}t \leq \ite{\hat{x}}{t}$, and if $\ite{x}0 > \ite{\hat{x}}{0}$, then for all $t \in \naturals$, $\ite{x}t \geq \ite{\hat{x}}{t}$.
\end{lemma}
\begin{proof}
    Without loss of generality, assume that $\ite{x}0 < \ite{\hat{x}}{0}$.
    Suppose to the contrary that there exists a timestep $t \in \naturals$ at which $\ite{x}t > \ite{\hat{x}}{t}$.
    There must exist some initial timestep $t \in \naturals$ where the trajectories crossed; that is, where $\ite{x}{t-1} \leq \ite{x}{t-1}$ and $\ite{x}t > \ite{\hat x}{t}$.
    Since $\pi^*$ is constructed to be a simple policy, we can directly infer that $\ite{x}{t-1} < \ite{x}{t-1}$.
    For notational convenience and without loss of generality, we will assume $t=2$.

    We proceed by observing that the user state $\ite{\hat{x}}{2}$ must be reachable from $\ite{x}1$, i.e. $\ite{\hat{x}}{2} \in [C_E - \maxarm + \ite{\hat{x}}{1}, C_E + \maxarm + \ite{\hat{x}}{1}]$.
    To see this, recall that we can explicitly write the action $\ite{i}1$ that the policy $\pi^*$ takes to get from user state $\ite{x}1$ to $\ite{x}2$ is given by the equality $\ite{i}1 = \ite{x}1 - \ite{x}2 + C_E$.
    Since $\ite{x}1 = \ite{x}0 < \ite{\hat{x}}{0} = \ite{\hat{x}}{1}$, it follows that $\ite{x}1 - \ite{\hat{x}}{2} - C_E < \ite{\hat{x}}{1} - \ite{\hat{x}}{2} - C_E \leq K$.
    Similarly, we know that $\ite{x}1 - \ite{\hat{x}}{2} - C_E \geq -K$; otherwise, we would reach a contradiction due to $\ite{\hat{x}}{2} > \ite{x}1 + K - C_E$ and $\ite{x}2 > \ite{\hat{x}}{2}$ implying that $\ite{x}2 > \ite{x}1 + K - C_E$ despite $\ite{x}2$ being reachable from $\ite{x}1$.
    
    Similarly, the user state $\ite{x}2$ must be reachable from $\ite{\hat{x}}{1}$, i.e. $\ite{x}2 \in [C_E - \maxarm + \ite{x}1, C_E + \maxarm + \ite{x}1]$.
    This is because since $\ite{x}2$ is reachable from $\ite{x}1$, $\ite{\hat{x}}{1} - \ite{x}2 - C_E > \ite{x}1 - \ite{x}2 - C_E \geq -K$.
    Since $\ite{\hat{x}}{2}$ is reachable from $\ite{\hat{x}}{1}$, $\ite{\hat{x}}{1} - \ite{x}2 - C_E < \ite{\hat{x}}{1} - \ite{\hat{x}}{2} - C_E \leq K$.

    Since $\ite{x}1$ can reach both $\ite{x}2$ and $\ite{\hat{x}}{2}$, the optimality of the simple policy $\pi^*$ gives that
    \begin{align*}
       C_E + C_R + \ite{x}1 - \ite{x}2 + \gamma  \cdot \tilde{f}(\ite{x}2) \cdot V(\ite{x}2) \geq  C_E + C_R + \ite{x}1 - \ite{\hat{x}}{2} +   \gamma \cdot \tilde{f}(\ite{\hat x}2) \cdot V(\ite{\hat{x}}{2}).
    \end{align*}
	The optimality of $\pi^*$ also implies that
	\begin{align*}
		C_E + C_R + \ite{\hat{x}}{1} - \ite{\hat{x}}{2} + \gamma 
 \cdot \tilde{f}(\ite{\hat{x}}{2}) \cdot V(\ite{\hat{x}}{2}) \geq  C_E + C_R + \ite{\hat{x}}{1} - \ite{x}2 + \gamma  \cdot \tilde{f}(\ite{{x}}{2}) \cdot V(\ite{x}2),
	\end{align*}
	Thus, $\gamma  \cdot \tilde{f}(\ite{{x}}{2}) \cdot V(\ite{x}2) - \ite{x}2 = \gamma \cdot \tilde{f}(\ite{\hat{x}}{2}) \cdot   V(\ite{\hat x}2) - \ite{\hat{x}}{2}$.

        We now compare the optimality of following the action suggested by the policy $\pi^*$ at the user state $\ite{\hat{x}}{1}$ with the optimality of showing the content $i = \ite{x}2 - \ite{\hat{x}}{1} + C_E$, finding
	\begin{align*}
		Q(\pi^*(\ite{\hat{x}}{1}), \ite{\hat{x}}{1}) & = \ite{\hat{x}}{2} - \ite{\hat{x}}{1} + C_E + C_R + \gamma  \cdot \tilde{f}(\ite{\hat{x}}{2}) \cdot V(\ite{\hat{x}}{2})       \\
		             & = \ite{x}2 - \ite{\hat{x}}{1} + C_E + C_R +  \gamma  \cdot \tilde{f}(\ite{{x}}{2}) \cdot V(\ite{x}2)       \\
		             & = Q(\ite{x}2 - \ite{\hat{x}}{1} + C_E, \ite{\hat{x}}{1}).
	\end{align*}
        This implies that $\ite{x}2 - \ite{\hat{x}}{1} + C_E \in \argmax_{i \in [-K, K]} Q(i, \ite{\hat x}{1})$, which is a contradiction, since $x_2 - \ite{\hat x}{1} + C_E > \ite{\hat x}{2} - \ite{\hat x}{1} + C_E$ but $\pi^*(\ite{\hat x}{1})$ is defined to be the most exploitative of all of the optimal arms: $\pi^*(\ite{\hat x}{1}) = \max \argmax_{i \in [-K, K]} Q(i, \ite{\hat x}{1})$.
\end{proof}

We now prove the facts that compose \cref{lemma:piecewiseconstant}.
\begin{fact}
\label{fact:monotonictrajectories}
Any trajectory of the policy $\pi^*$ is either monotonically non-decreasing or monotonically non-increasing in the user's state.
\end{fact}
\begin{proof}
Suppose the policy $\pi^*$ unrolls a user state trajectory $\ite{x}1, \dots, \ite{x}t, \dots$ where $\pi^*(\ite{x}t) > C_E$ and $\pi^*(\ite{x}{t-1}) < C_E$ or $\pi^*(\ite{x}t) < C_E$ and $\pi^*(\ite{x}{t-1}) > C_E$.
Then, since $\pi^*$ is a simple policy, the user state trajectory of $\pi^*$ on the initial user state $\ite{x}{t-1}$ crosses the user state trajectory of $\pi^*$ on the initial user state $\ite{x}t$.
This contradicts \cref{lemma:crossing}.
\end{proof}

\begin{fact}
    \label{fact:fastup}
    In any trajectory of the policy $\pi^*$ where user states increase, there can be at most $k+1$ timesteps in which the step is neither a fixed point, i.e. $i= C_E$, nor a full step upwards, i.e. $i=-K$.
\end{fact}
\begin{proof}
    Suppose that the action trajectory of the policy $\pi^*$ simultaneously satisfies  $\ite{i}t \in (-K, C_E)$ and $\ite{i}{t'} \in (-K, C_E)$.
    Without loss of generality, suppose $t < t'$.
    \cref{lemma:trajectory} directly implies that, in at least one timestep, the user state trajectory subsequence $\ite{x}{t+1}, \dots, \ite{x}{t'}$ is a discontinuity; that is $\exists \tau \in [t+1, t']$ where $\ite{x}\tau \in D$.
    Since there are only $k$ discontinuities of the function $f$, there can only be $k+1$ timesteps with incomplete steps $i \in (-K, C_E)$.
\end{proof}

\begin{fact}
    \label{fact:fastdown}
    In any trajectory of the policy $\pi^*$ where user states decrease, there can be at most $k+1$ timesteps in which the step is neither a fixed point, i.e. $i= C_E$, nor a full step downwards, i.e. $i=K$.
\end{fact}
\begin{proof}
    As in the proof of \cref{fact:fastup}, suppose that the action trajectory of the policy $\pi^*$ simultaneously satisfies  $\ite{i}t \in (C_E, K)$ and $\ite{i}{t'} \in (C_E, K)$ where $t < t'$.
    \cref{lemma:trajectory} directly implies that, in at least one timestep, the user state trajectory subsequence $\ite{x}{t+1}, \dots, \ite{x}{t'}$ is a discontinuity.
    Since there are only $k$ discontinuities of the function $f$, there can only be $k+1$ timesteps with incomplete steps $i \in (C_E, K)$.
\end{proof}

\begin{fact}
    \label{fact:trajectorylimit}
    The action trajectory of the policy $\pi^*$ always has a limit which exists, i.e. $\lim_{t \to \infty} \ite{i}t$ exists, and the limit must either be maximal exploitation, i.e. $\lim_{t \to \infty} \ite{i}t = K$, or maintaining user state constant.
    In the latter case, the limit of the user state trajectory also exists and is one of the two neighboring discontinuities $d$ at which $\pi^*$ keeps user state with $\pi^*(d) =C_E$.
\end{fact}
\begin{proof}
    \cref{fact:monotonictrajectories} states that the user state trajectory of the policy $\pi^*$ is monotonically non-decreasing or monotonically non-increasing.
    Since $\pi^*$ is a simple policy, once it reaches a user state at which it keeps user state constant, i.e. $i=C_E$, it will do so perpetually.
    We also know that the policy $\pi^*$ will only play actions in the intervals $(C_E, K)$ and $(-K, C_E)$ a finite number of times.
    Thus, the limit of the action trajectory of policy $\pi^*$ must always either be $C_E$, $-K$ or $K$.

    We can rule out the limit of the policy being the action that maximally boosts user state, i.e. $\lim_{t \to \infty} \ite{i}t =-K$, by \cref{lemma:suboptimalmaximalup}.
    In the case that the limit is maintaining user state at a constant level, i.e. $\lim_{t \to \infty} \ite{i}t = C_E$, we recall that the limit must occur on a discontinuity of $f$.
    In this case, the user state trajectory must also have a limit, where either $\lim_{t \to \infty} \ite{x}t = \min \bset{d \in D \mid d > \ite{x}0}$ or $\lim_{t \to \infty} \ite{x}t = \max \bset{d \in D \mid d < \ite{x}0}$.
    This is because we know, by \cref{lemma:holdingdiscontinuity}, that user state can only be kept constant on a discontinuity, i.e. $\lim_{t \to \infty} \ite{x}t \in D$.
    Moreover, by \cref{lemma:crossing}, it is impossible for the limiting user state to satisfy $\lim_{t \to \infty} \ite{x}t > \min \bset{d \in D \mid d > \ite{x}0}$ as the trajectory of $\pi^*$ would cross the trajectory of $\pi^*$ starting at the user state $\lim_{t \to \infty} \ite{x}t$.
    Similarly, \cref{lemma:crossing} guarantees $\lim_{t \to \infty} \ite{x}t \geq \max \bset{d \in D \mid d < \ite{x}0}$.
\end{proof}
\end{proof}

\cref{lemma:trajectory} and \cref{lemma:piecewiseconstant} trivially extend to optimal policies for multiple users and where $K_1 \neq K_2$; we restate these lemmas for multiple users below.
\begin{restatable}{lemma}{piecewiseconstantmodified}
    \label{lemma:piecewiseconstantmodified}
    In the linear setting where a set of user demand functions $f_1, \dots, f_T$ each has a complexity of $k < \infty$, there is a simple optimal policy $\pi \in \argmax_{\pi^* \in \Pi} \sum_{t=1}^T J_{f_t}(\pi^*)$ for the app creator that satisfies all of the following characteristics:
    \begin{itemize}
        \item The sequence $x_1, x_2, \dots$ of user states is monotonic. %
        \item The limit $x_\infty = \lim_{t \to \infty} x_t$ exists and is either at a discontinuity of $f$ or negative infinity.
        \item (When $x_\infty = -\infty$) The app always shows the highest-revenue content, i.e. $i = K_2$. 
        \item (When $x_\infty$ is a discontinuity of $f$)
        The user state $x_\infty$ will be reached within $k + n + 1$ interactions ({in other words} $\tsv{x}{n+k+1} = x_\infty$) where $n$ is the smallest number of interactions in which $x_\infty$ can be reached {in some policy (that is, if}
        there exists an app policy for which $\tsv{x}{n} = x_\infty$).
        Moreover, if $n=1$, $x_2 = x_3 = \dots = x_\infty$.
    \end{itemize}
\end{restatable}

\begin{lemma}
    \label{lemma:trajectorymodified}
    Consider a simple policy $\pi$ that is optimal for a set of user demand functions $f_1, \dots, f_T$, i.e. $\pi \in \argmax_{\pi^* \in \Pi} \sum_{t=1}^T J_{f_t}(\pi^*)$.
    For any subsequence of the policy's user state trajectory, $\ite{x}1, \ite{x}2, \dots, \ite{x}T$, where $\ite{x}2, \dots, \ite{x}T$ are all not discontinuities for any demand function in $f_1, \dots, f_T$:
    \begin{enumerate}
        \item If the policy $\pi$ is not maximally increasing user state at the last step of the subsequence, i.e. $\pi(\ite{x}T) > -K_1$, then all previous actions in the subsequence should be maximally decreasing user state $\pi(\ite{x}1) = \pi(\ite{x}{T-1}) = K_2$.
        \item If the policy $\pi$ is not maximally decreasing user state at the first step, that is $\pi(\ite{x}1) < K_2$, then all later actions in the subsequence should be maximally increasing user state $\pi(\ite{x}2) = \pi(\ite{x}T) = -K_1$.
    \end{enumerate}
\end{lemma}

\if\litearxiv0
\subsection{Proofs of \cref{fact:algorithmleft} and \cref{fact:algorithmright}}
\label{subsection:algorithmproofs}
\algorithmleft*

\begin{proof}
Let $(\ite x1, \ite i1), (\ite x2, \ite i2), \dots$ denote the trajectory of $\pi^*$ starting at $d$.
We proceed inductively.
First consider the base case where $d = \min D$.
By \cref{lemma:piecewiseconstant}, $\ite xt$ must be monotonically non-increasing and, if $\ite it = C_E$, then $\ite xt \in D$.
By \cref{lemma:trajectory}, since there are no discontinuities below $D$, for all $t > 1$, either $\ite it = C_E$ for all $t \geq 1$ or $\ite it = K$ for all $t \geq 1$.
The payoff of the former trajectory is computed in \cref{line:evaluate_d'} and the payoff of the latter is computed in \cref{line:below_D} of \cref{alg:linear}, and the maximum is taken for $V[d]$ and $\text{Trajs}[d]$, tie-breaking in favor of \cref{line:below_D}.

For the inductive step, we fix a $d \in D$ where $d < 0$.
By \cref{lemma:piecewiseconstant}, $\ite xt$ must be monotonically non-increasing and, if $\ite it = C_E$, then $\ite xt \in D$.
Consider the set of discontinuities visited by the optimal policy $\pi^*$: $\bset{\ite xt \mid t \geq 2, \ite xt \in D}$.
If this set is empty, i.e. no discontinuities are visited, then by \cref{lemma:trajectory}, $\ite it = K$ for all $t \geq 1$ and \cref{line:below_D} computes the payoff of $\pi^*$.
If it is not, \cref{line:below_D} computes the payoff of a policy $\pi$ that does not visit any discontinuities.
If this set consists only of $d$, then $\ite it = C_E$ for all $t \geq 1$ and \cref{line:stay} computes the payoff of $\pi^*$.
If it does not, \cref{line:stay} computes the payoff of a policy $\pi$ that does stay at $d$.
If the smallest element of the set is $d' = \min \bset{\ite xt \mid t \geq 2, \ite xt \in D}$, suppose that $t'$ is the first timestep where $\ite x{t'} = d'$.
By \cref{lemma:piecewiseconstant}, $\ite it = \min \bset{K, x_{t-1} - x^* + C_E}$ for all $t < t'$.
Thus, $J_{d}(\pi^*) = \hat v + \hat \gamma \cdot J_{d'}(\pi^*)$ where $\hat v, \hat \gamma, \text{Traj} \leftarrow \mathrm{GetPayoff}(d, d', f)$. 
Since $d' \leq d$ and $\ite x{t'} \geq \ite x{t'+1}$, we have by inductive hypothesis that $J_{d'}(\pi^*) = V_{d'}^*$.
Thus, \cref{line:evaluate_d'} computes the payoff of $\pi^*$.

In any of the three possible cases, one of \cref{line:below_D}, \cref{line:evaluate_d'} or \cref{line:stay} must have computed $J_{d'}(\pi^*)$.
Moreover, they will do so before computing the payoffs of any other optimal policies.
Noting that \cref{line:below_D}, \cref{line:evaluate_d'} and \cref{line:stay} only ever compute the payoffs of valid policies, the optimality of $\pi^*$ implies that $\text{Trajs}[d]$ describes the action trajectory unrolled by $\pi^*$.
\end{proof}

\algorithmright*
\begin{proof}
Let $(\ite x1, \ite i1), (\ite x2, \ite i2), \dots$ denote the trajectory of $\pi^*$ starting at $d$.
We proceed inductively.
First consider the base case where $d = \max D$.
By \cref{lemma:piecewiseconstant}, $\ite xt$ must be monotonically non-increasing and, if $\ite it = C_E$, then $\ite xt \in D$.
Thus, $\ite it = C_E$ for all $t \geq 1$, the payoff of which is computed in \cref{line:pos_stay} of \cref{alg:linear}.

For the inductive step, we fix a $d \in D$ where $d > 0$.
By \cref{lemma:piecewiseconstant}, $\ite xt$ must be monotonically non-decreasing and, if $\ite it = C_E$, then $\ite xt \in D$.
Consider the set of discontinuities visited by the optimal policy $\pi^*$: $\bset{\ite xt \mid t \geq 2, \ite xt \in D}$.
This set cannot be empty, as by \cref{lemma:piecewiseconstant}, some discontinuity must be reached by $\pi^*$ in finite time.
If this set consists only of $d$, then $\ite it = C_E$ for all $t \geq 1$ and \cref{line:pos_stay} computes the payoff of $\pi^*$.
If it does not, \cref{line:pos_stay} computes the payoff of a policy $\pi$ that does stay at $d$.
If the largest element of the set is $d' = \max \bset{\ite xt \mid t \geq 2, \ite xt \in D}$, suppose that $t'$ is the first timestep where $\ite x{t'} = d'$.
By \cref{lemma:piecewiseconstant}, letting $\delta = (d' - d) \mathrm{\;mod\;} (K + C_E)$ for all $t < t'$, we know that action $i_t = \delta$ if $\delta > 0$ and $t = 1$, while $i_t = -K$ otherwise.
Thus, $J_{d}(\pi^*) = \hat v + \hat \gamma J_{d'}(\pi^*)$ where $\hat v, \hat \gamma, \text{Traj} \leftarrow \mathrm{GetPayoff}(d, d', f)$. 
Since $d' \geq d$ and $\ite x{t'} \geq \ite x{t'+1}$, we have by inductive hypothesis that $J_{d'}(\pi^*) = V[d']$.
Thus, \cref{line:pos_d'} computes the payoff of $\pi^*$.
In either of these cases, either \cref{line:pos_stay} or \cref{line:pos_d'} must have computed $J_{d'}(\pi^*)$.
Moreover, they will do so before computing the payoffs of any other optimal policies.
Noting that \cref{line:pos_stay} and \cref{line:pos_d'} only ever compute the payoffs of valid policies, the optimality of $\pi^*$ implies that $\text{Trajs}[d]$ describes the action trajectory unrolled by $\pi^*$.
\end{proof}
\fi

\subsection{Proof of \cref{theorem:trionlinelearning}}
\label{subsection:trionlinelearningproof}
\if\litearxiv0
\trionlinelearning*
\begin{proof}
For each $f_j$, we construct its \emph{rounded counterpart} $f'_j$ defined as follows:
\begin{equation}
f'_j(x)
\asseq
 f_j(\floor{x / (K - C_E)} (K - C_E)) %
\label{eq:round}\end{equation}
This rounding process transforms each demand function into a piecewise constant function.
Since each demand function $f_j$ is a constant function on the domain $(-\infty, -m) \cup (m, \infty)$, %
we know that the discontinuities of its rounded counterpart $f'_j$ lie in the set $[-m, m]$.
Moreover, by construction, $f'_j(x) \leq f_j(x)$ for all $x \in \reals$, meaning that the rounded demand functions correspond to a user who is less likely to engage and, therefore, provides less value to the app.

Importantly, we can verify that the optimal payoff is not affected by this rounding. To see why, observe that $(C_E - i) \mod (K - C_E) = 0$ for all $i \in \{2C_E - K, C_E, K\}$.
As a result, the user state always lies on a multiple of $K - C_E$: $\ite x t \mod (K - C_E) = 0$.
Rounding thus does not impact the optimal payoff:
\begin{equation*}
    \max_{\pi^* \in \Pi} \sum_{j=1}^N J_{f_j}(\pi^*) = \max_{\pi^* \in \Pi} \sum_{j=1}^N J_{f_j'}(\pi^*).
\end{equation*}

The key advantage of working with rounded demand functions is that they are piecewise constant with strategically spaced discontinuities, as summarized by the following lemma:
\begin{restatable}{lemma}{otwo}
\label{lemma:o2}
   Given a set of rounded demand functions $f_1, \dots, f_N$, define the set of simple policies $\Pi' = \bset{\pi_{i, v} \mid v \in \pm 1, i \in [0, \dots, 2 m / K] \cup \bset{-\infty}}$ where each policy $\pi_{i,v}$ is defined by setting $\pi_{i, v}(i \cdot K - m) = C_E$ and $\pi_{i, v}(x) = C_E + v (K - C_E)$ for all $x \neq i \cdot K - m$.
   There is an optimal policy in this set, i.e. \[
\max_{\pi^* \in \Pi} \sum_{j=1}^N J_{f_j'}(\pi^*)
= \max_{\pi^* \in \Pi'} \sum_{j=1}^N J_{f_j'}(\pi^*),
\]
\end{restatable}

Since our reduced policy space $\Pi'$ from \cref{lemma:o2} is small with $\setsize{\Pi'} \in O(\frac m K)$, we can directly apply a bandit learning algorithm to select among the policies in $\Pi'$.
Moreover, since each user's episodic feedback is bounded by $U$, we can apply a folklore stochastic approximation argument~\cite{nemirovski2009robust} to the high-probability regret bound of Exp3-IX \cite{neu15}, as summarized by the following lemma.

\begin{lemma}
\label{lemma:exp3}
Consider an agent who repeatedly uses the (random) Exp3-IX algorithm \cite{neu15} to select an action $i_t$ from a finite set $\cA$.
An adversary, who observes $i_1, \dots, i_t$, chooses a loss $\loss_t \in [0, 1]^{\cA}$. However, the agent receives only a noisy, unbiased estimate $\hat \loss_t \in [0, 1]^{\cA}$, satisfying $\mathbb{E}[\hat \loss_t] = \loss_t$.
The agent then chooses its next action $i_{t+1}$ according to the Exp3-IX algorithm, which is computed from the history $\bsetf{(i_\tau, \hat \loss_\tau(i_\tau))}_{\tau \in [t]})$.
Then, with probability at least $1 - \delta$ in the randomness of Exp3-IX and the observed losses $\hat \loss$, the cumulative loss incurred by the agent satisfies
\[
\sum_{t=1}^T \hat \loss_t(i_t) \geq \max_{i^* \in \cA} \sum_{t=1}^T \loss_t(i^*) - O(\sqrt{T \setsize{A} \log( \setsize{A} /\delta)}).
\]
\end{lemma}
Applying Exp3-IX results in the regret bound
\begin{align*}
\sum_{j=1}^N \hat J_{f_j}(\pi_j)
&\geq \max_{\pi^* \in \Pi'} \sum_{j=1}^N J_{f_j}(\pi^*) - O(U \cdot \sqrt{({Nm }/K)  \log(m /K\delta)}) \\
&\geq \max_{\pi^* \in \Pi'} \sum_{j=1}^N J_{f_j'}(\pi^*) - O(U \cdot \sqrt{({Nm }/K) \log(m /K\delta)}) \\
&= \max_{\pi^* \in \Pi} \sum_{j=1}^N J_{f_j}(\pi^*) - O(U \cdot \sqrt{({Nm }/K) \log(m /K\delta)}).
\end{align*}
\end{proof}
\fi

\otwo*
\begin{proof}
\cref{lemma:piecewiseconstantmodified} (\cref{lemma:piecewiseconstant}) establishes that the user state trajectory of an optimal policy must be either stationary, monotonically increasing or monotonically decreasing.
We will consider this optimal policy and the three possible types of trajectories it can take, and prove that in all three situations, the optimal policy must exist in $\Pi'$.

The stationary case is trivial as there exists a policy $\pi_{1, m/K} \in \Pi'$ that always plays $\ite it = C_E$ and thus keeps user state constant.

Now consider the monotonically decreasing case.
Suppose, for the sake of contradiction, that the optimal policy is not always fully decreasing user state, i.e. it plays $\ite it < K$ at some timestep $t$.
This means that $\ite xt \; \mathrm{mod} \; (K - C_E) \neq 0$ at some timestep $t$.
Then \cref{lemma:trajectorymodified} states that the user will keep fully decreasing user state until it reaches a discontinuity of the demand function.
Every discontinuity $x$ of the demand function satisfies $x \; \mathrm{mod} \; (K - C_E) = 0$; this means that $\ite x{t'} \; \mathrm{mod} \; (K - C_E) = 0$ at some timestep $t' > t$; $t'$ must be finite as all discontinuities of $f'_t$ lie in $[-m, m]$.
However, fully decreasing user state decreases user state by $K - C_E$, meaning that there must exist some finite $j$ such that $\ite xt - j (K - C_E) \; \mathrm{mod} \; (K - C_E) = \ite xt \; \mathrm{mod} \; (K - C_E) =  0$, which is a contradiction.
Since the demand function is constant below $-m$ and above $m$ and thus all discontinuities must lie in the set $\bset{iK -m \mid i \in [0, \dots, 2m/K]}$, we have that $\pi_{1, i}$ must be optimal for some choice $i \in [0, \dots, 2m/K] \cup \bset{-\infty}$.

The monotonically increasing case follows similarly.
\cref{lemma:trajectorymodified} gives that the user will fully increase user state until it reaches a discontinuity of the demand function.
Since all discontinuities lie a factor of $K-C_E$ above origin and fully increasing user state increases user state by $K - C_E$, the optimal policy must fully increase user state before stopping on a discontinuity, which must lie in the set $\bset{iK -m \mid i \in [0, \dots, 2m/K]}$. Thus $\pi_{-1, i}$ must be optimal for some choice $i \in [0, \dots, 2m/K]$.
\end{proof}

\subsection{Proof of \cref{theorem:onlinelearning}}
\label{subsection:onlinelearningproof}
\if\litearxiv0
\onlinelearning*
\begin{proof}
For each $f_j$, we construct its rounded counterpart $f'_j$ as in Equation~\eqref{eq:round}.
We can verify that the optimal payoff is not significantly affected by this rounding. 
\begin{restatable}{lemma}{oone}
\label{lemma:o1}
   The optimal payoff that can be realized with the rounded demand functions $f'_1, \dots, f'_N$ is at least a constant fraction of that of the original demand functions $f_1, \dots, f_N$:
   \[\max_{\pi^* \in \Pi} \sum_{j=1}^N J_{f_j}(\pi^*) \leq  \left(1 + \tfrac{K - C_E}{C_R + K}\right) \max_{\pi^* \in \Pi} \sum_{j=1}^N J_{f_j'}(\pi^*).\]
\end{restatable}

\cref{lemma:o1} and \cref{lemma:o2} give
\begin{align*}
\max_{\pi^* \in \Pi} \sum_{j=1}^N J_{f_j}(\pi^*)
\leq \left(1 + \tfrac{K - C_E}{C_R + K}\right) \max_{\pi^* \in \Pi'} \sum_{j=1}^N J_{f_j'}(\pi^*) \leq \left(1 + \tfrac{ K-C_E}{C_R + K}\right) \max_{\pi^* \in \Pi'} \sum_{j=1}^N J_{f_j}(\pi^*),
\end{align*}
where the second inequality follows from the fact that $f_j'(x) \leq f_j(x)$ for all $x \in \reals$.

As in our proof of \cref{theorem:trionlinelearning}, we can again apply a bandit learning algorithm to choose from $\Pi'$.
We can apply a standard stochastic approximation argument~\cite{nemirovski2009robust} to the high-probability regret bound of Exp3-IX \cite{neu15} on the reduced policy space $\Pi'$ to get:
\begin{align*}
\sum_{j=1}^N \hat J_{f_j}(\pi_j)
&\geq \max_{\pi^* \in \Pi'} \sum_{j=1}^N J_{f_j}(\pi^*) - O\left(U \cdot \sqrt{\tfrac{Nm \gamma}{K}  \log\tfrac {m }{K\delta}}\right) \\
&\geq \max_{\pi^* \in \Pi'} \sum_{j=1}^N J_{f_j'}(\pi^*) - O\left(U \cdot \sqrt{\tfrac{Nm \gamma}{K}  \log\tfrac {m }{K\delta}}\right) \\
&\geq \tfrac {C_R + K}{C_R - C_E + 2K} \max_{\pi^* \in \Pi} \sum_{j=1}^N J_{f_j}(\pi^*) - O\left(U \cdot \sqrt{\tfrac{Nm \gamma}{K}  \log\tfrac {m }{K\delta}}\right).
\end{align*}
Finally, since $C_E \geq 0$ and $C_R \geq 0$, $ \tfrac {C_R + K}{C_R - C_E + 2K} \geq \tfrac 12$.
\end{proof}
\fi

\oone*
\begin{proof}
    Now consider a $\pi^*$ belonging to $\argmax_{\pi^* \in \Pi} \sum_{j=1}^N J_{f_j}(\pi^*)$.
    By the variable discount rate formula (\cref{theorem:cputime}), we know that
    \[
    \sum_{j=1}^N J_{f_j}(\pi^*)  = N (C_R + i) + \sum_{j=1}^N \gamma \tilde f_j(C_E - i) J_{f_j}(\pi^*, C_E - i)
    \]
    where we use $J_f(\pi, x_0)$ to denote the app creator's objective value under policy $\pi$ for demand function $f$ when user state is initialized at $x_0$.
    By \cref{lemma:o4}, we can upper bound
    \begin{align*}
    \sum_{j=1}^N \gamma \tilde f_j(C_E -  \pi^*(0)) J_{f_j}(\pi^*, C_E -  \pi^*(0))
    & \leq \sum_{j=1}^N \gamma \tilde f_j'(K - C_E) J_{f_j}(\pi^*, C_E -  \pi^*(0)) \\
    & \leq \sum_{j=1}^N \gamma \tilde f_j'(K - C_E) (C_E -  \pi^*(0) + J_{f_j'}(\pi^*, K - C_E)) \\
    & \leq N(C_E - \pi^*(0)) + \sum_{j=1}^N \gamma \tilde f_j'(K - C_E) (J_{f_j'}(\pi^*, K - C_E)).
    \end{align*}
    We then have that
    \begin{align*}
    \sum_{j=1}^N J_{f_j}(\pi^*) & \leq N (C_R + C_E) + \sum_{j=1}^N \gamma \tilde f_j'(K - C_E) (J_{f_j'}(\pi^*, K - C_E))  \\
    & = N (C_R + C_E + K - 2 C_E - C_R) + \sum_{j=1}^N (2C_E + C_R - K) + \gamma \tilde f_j'(K - C_E) (J_{f_j'}(\pi^*, K - C_E))  \\
    & \leq N(K - C_E) + \max_{\pi^* \in \Pi} \sum_{j=1}^N J_{f_j'}(\pi^*)  \\
    & \leq (1 + \tfrac{K - C_E}{C_R + K})\max_{\pi^* \in \Pi} \sum_{j=1}^N J_{f_j'}(\pi^*).
    \end{align*}
    The second inequality follows from executing a policy that chooses $i_1 = 2C_E-K$ and then executes the \cref{lemma:o4} policy starting from user state $K - C_E$.
    The inequality then follows from  $\max_{\pi^* \in \Pi} \sum_{j=1}^N J_{f_j'}(\pi^*) \geq N (K + C_R)$ as we can set $\pi^*(0) = K$. 
\end{proof}

\begin{restatable}{lemma}{oonee}
\label{lemma:o4}
   The optimal payoff that can be realized with a rounded demand functions $f'_1, \dots, f'_N$ at some initial state $x'$ is within $K-C_E-x'$ that which can be realized with the original demand functions $f_1, \dots, f_N$ at an initial state of $0$:
   \[\max_{\pi^* \in \Pi} \sum_{j=1}^N J_{f_j}(\pi^*, 0) \leq N(K - C_E - x') + \max_{\pi^* \in \Pi} \sum_{j=1}^N J_{f_j'}(\pi^*, x')\]
   where we use $J_j(\pi, x_0)$ to denote the app creator's objective value under policy $\pi$ for demand function $f$ when user state is initialized at $x_0$.
\end{restatable}
\begin{proof}
    Let $\pi^* = \argmax_{\pi^* \in \Pi} \sum_{j=1}^N J_{f_j}(\pi^*)$ denote the optimal policy for the original demand functions.
    Note that a user's demand function does not influence the per-interaction trajectory of policy $\pi^*$, only the amount of time between each interaction.
    Thus, its (deterministic) trajectory $(\ite x1, \ite i1), (\ite x2, \ite i2), \dots$ on every user will be the same.

    We can now consider the non-stationary but Markov policy $\pi'$ that executes the action trajectory \[\ite {\tilde i}t = \max\bsetf{2C_E-K, \ite it - \max \bsetf{0, K - C_E - x' - (\sum_{\tau < t} \ite i{\tau} - \ite {\tilde i} {\tau})}}.\]
    That is, $\pi'$ executes the same action trajectory as the optimal policy $\pi^*$ except that, initially, $\pi'$ maximally invests in boosting the user state until the user state is $K - C_E - x'$ higher than what it would have been under the policy $\pi^*$.
    In doing so, $\pi'$ guarantees that when started from $x'$ it produces as much demand on the rounded demand functions as $\pi^*$ does when started from $0$ on the original demand functions, i.e. \[f'_j(\ite {\tilde x} t) \geq f'_j(\ceil{\ite {x} t / (K - C_E)} (K - C_E)) \geq f_j(\ite x t)\]
    for all timesteps $t$.
    Moreover, the only difference in the actions taken by $\pi'$ and $\pi^*$ is that $\pi'$ invests $K - C_E - x'$ more into user state.  
    It follows that \[J_{f_j}(\pi^*) - J_{f_j}(\pi') \leq \sum_{t=1}^T \ite it - \ite {\tilde i} t \leq K - C_E - x'.\]
    Since there must exist a stationary Markov policy with as high a payoff as $\pi'$, we can upper bound the left-hand side of our claim by \[N(K-C_E - x') + \sum_{j=1}^N J_{f_j}(\pi') \leq N(K-C_E - x') + \max_{\pi^* \in \Pi} \sum_{j=1}^N J_{f_j}(\pi^*).\]
\end{proof}

\section{Omitted Proofs and Additional Results for \cref{section:primitives}}

\subsection{Proof of \cref{prop:regime_1}}
\label{subsection:regimeproof}

\regime*
\begin{proof}
Let us set $a = 0$ and $b = \frac {\gamma}{1-\gamma} (\epsilon+p_2-p_1)$ for some sufficiently small choice of $\epsilon > 0$.
With these choices, we can define the user demand function as
\begin{align*}
    f(x) = \begin{cases}
        0 & x < 0 \\
        p_1 & 0 \leq x < \frac {\gamma}{1-\gamma} (\epsilon+p_2-p_1) \\
        p_2 & \frac {\gamma}{1-\gamma} (\epsilon+p_2-p_1) \leq x.
    \end{cases}
\end{align*}
By \cref{lemma:piecewiseconstant}, the optimal app policy must fall into one of three categories:
\begin{enumerate}
    \item The app maximally decreases the user state at each interaction by repeatedly showing content $\tsv it = b$.
    \item The app maintains the user state at $x = 0$ by always showing content $i = 0$.
    \item The app initially shows content $\tsv i1 = -b$ and the consistently shows content $\tsv it = 0$ to stabilize the user state at $x = b$.
\end{enumerate}
In the first case, the user immediately disengages, yielding a utility of $$1 + b = 1+ \frac{\gamma}{1-\gamma}(\epsilon+p_2-p_1).$$

When there is no friction, the payoff of the second policy is
$\tfrac{\gamma}{1-\gamma} p_1 + 1$
and the payoff of the third policy is
$\tfrac{\gamma}{1-\gamma} (p_1 - \epsilon) + 1.$
By construction, the app prefers the second policy over the third in the frictionless scenario.

In contrast, when there is complete friction, the payoff of the second policy is
$\tfrac{1}{1 - \gamma p_1}$
and the payoff of the third policy is
${\tfrac{1}{1 - \gamma p_2} - \tfrac {\gamma}{1-\gamma} (\epsilon + p_2 - p_1)}.$
By criterion (2), we have
\begin{align*}
    \tfrac 1 {1-\gamma p_2} - \tfrac \gamma {1-\gamma} p_2 > \tfrac 1 {1-\gamma p_1} - \tfrac \gamma {1-\gamma} p_1.
\end{align*}
Rearranging this inequality shows that, for sufficiently small $\epsilon$,
\begin{align*}
    {\tfrac{1}{1 - \gamma p_2} - \tfrac {\gamma}{1-\gamma} (\epsilon + p_2 - p_1)} > \tfrac{1}{1 - \gamma p_1}.
\end{align*}
Thus, under complete friction, the third policy outperforms the second.

Finally, it remains to show that the first policy is always suboptimal.
By assumption, $p_2 - p_1 < \frac{1-\gamma}{\gamma(1-\gamma p_1)}$ and $p_2 - p_1 < \frac{1-\gamma}{2\gamma(1-\gamma p_2)}$.
Let $\epsilon$ be sufficiently small that $p_2 - p_1 \leq \tfrac{1-\gamma}{\gamma(1-\gamma p_1)} - \epsilon$ and $p_2 - p_1 \leq \tfrac{1-\gamma}{2\gamma(1-\gamma p_2)} - \epsilon$.
With this choice, one can show that
\begin{align*}
    {\tfrac{1}{1 - \gamma p_2}  > 1 + 2 \tfrac {\gamma}{1-\gamma} (\epsilon + p_2 - p_1)}.
\end{align*}
Therefore, under complete friction, the payoff of the third policy always exceeds the friction-independent payoff of the first policy, establishing the suboptimality of the first policy.
\end{proof}
\subsection{Proof of \cref{prop:basic_1}}

\basiceng*
\begin{proof}
We will write $\tilde f$ with subscripts $\tilde f_c$ and $\tilde f_{c'}$ to denote the operative friction constant.
Let us fix $a = 0$.
We can write our user's demand function as $f(x) = 0.6 \cdot 1[0 \leq x < b] + 0.99 \cdot 1[x \geq b]$.
It has two discontinuities: one at $x = 0$ and one at $x = 4$.
Both are reachable with a single interaction by showing content $i = 0$ and $i = b$ respectively.
We now set $b$ as follows.
Consider the function
\begin{align*}
    g(c^*) &= \tfrac 1 {1 - \gamma \tilde f_{c^*}(b)} - \tfrac 1 {1 - \gamma \tilde f_{c^*}(0)} \\
    &= \frac{1}{1 - 0.9  (0.99 + \tfrac{1-0.99}{1-0.9  (1-(1-c^*) 0.99)}(1-c^*) \cdot 0.99 \cdot 0.9)} \\
    & \quad \quad - \frac{1}{1 - 0.9  (0.6 + \tfrac{1-0.6}{1-0.9  (1-(1-c^*) 0.6)}(1-c^*) \cdot 0.6 \cdot 0.9)}.
\end{align*}
$g$ is positive and monotonically increasing for $c^* \in [0, 1]$.
We can therefore let $$\epsilon \asseq \min \bset{\tfrac 12 (g(c') - g(c)), 0.1} > 0$$ and $b = g(c) + \epsilon$.
By \cref{lemma:piecewiseconstant}, we know that the optimal app policy must be one of three:
\begin{enumerate}
    \item Maximally decrease user state at each interaction by repeatedly showing content $\tsv it = b$.
    \item Show content $\tsv it = 0$ in perpetuity to maintain the user state at $x = 0$.
    \item Show content $\tsv i1 = -b$ and subsequently show content $\tsv it = 0$ in perpetuity to maintain the user state at $x = b$.
\end{enumerate}
The payoff of the first policy is simply $1 + b$, as the user will stop interacting immediately.

Now suppose the friction constant is $c$.
By \cref{theorem:cputime}, the payoff of the second  policy is $\tfrac{1}{1 - \gamma \tilde{f}_{c}(0)}$ while the payoff of the third policy is $\frac{1}{1 - \gamma \tilde{f}_{c}(b)} - b
=
\frac{1}{1 - \gamma \tilde{f}_{c}(0)} - \epsilon$.
This means the app will always prefer the second policy over the third.
If instead the friction constant is $c'$, by \cref{theorem:cputime}, the payoff of the second  policy is $\tfrac{1}{1 - \gamma \tilde{f}_{c'}(0)}$ while the payoff of the third policy is 
\begin{align*}
\frac{1}{1 - \gamma \tilde{f}_{c'}(b)} - b
&=
\frac{1}{1 - \gamma \tilde{f}_{c'}(0)} +
g(c') - g(c) - \epsilon > \frac{1}{1 - \gamma \tilde{f}_{c'}(0)},
\end{align*}
where we use that $g(c) < g(c')$.
In this case, we have that the app prefers the third policy over the second.

Since the utility of the first policy is friction-independent and the other two policy utilities are decreasing in friction, e can observe that the first policy is always suboptimal as
$$
\frac{1}{1 - \gamma \tilde{f}_{c'}(b)} - g(c)
\geq 1 + g(c) + 0.5$$
for all $c \leq 0.3$
and so
$$
\frac{1}{1 - \gamma \tilde{f}_{c'}(b)} - b
\geq 1 + b.$$

Thus, under friction $c'$, the app's optimal policy is policy 3, resulting in the user state trajectory $x_t = b$ for all $t \geq 2$.
Similarly, under friction $c$, the app's optimal policy is policy 2, resulting in the user state trajectory $x_t = 0$ for all $t \geq 2$.
It follows that $x_t^c < x_t^{c'}$ for all $t \geq 2$.
\end{proof}

\subsection{Analog of \cref{prop:basic_1} for Absent/Complete Friction}
The app's increased investment in user engagement makes it possible for usage of an app to increase when friction is increased.
If the user gains utility from its interactions with the app, the user's discounted utility may also be higher.
This holds true even in the extreme case where we are comparing the absence of friction ($c = 0$) with the complete friction ($c=1$) of a user being entirely banned from rejoining apps.
\begin{restatable}{proposition}{basicutility}
    \label{prop:basic_2}
    In \cref{ex:first_example}, the number of expected app-user interactions in the first $N=100$ timesteps is lower when there is less friction than when there is more friction.
    Similarly, if the user gains utility upon every app interaction, its cumulative utility (for discount factor $\gamma' \in (0, 0.98)$) is strictly greater when there is more friction.
\end{restatable}
By replacing \cref{ex:first_example}'s demand levels of 60\% and 99\% with appropriate substitutes, \cref{prop:basic_2} can hold for arbitrary choices of $N \in \integers$ and $\gamma' \in (0, 1)$.
We can also extend \cref{ex:first_example} to generalizations of our basic model of user-app interactions.
For example, the preceding \cref{prop:basic_1} and \cref{prop:basic_2} both hold if instead of the user's satisfaction being the sum of previous experiences, the user's satisfaction corresponds to the average or arbitrarily discounted sum of their previous experiences.
\begin{proof}[Proof of \cref{prop:basic_2}]
    Let us resume our construction from the proof of \cref{prop:basic_1}.
    In this construction, when there is less friction, the app's policy results in a user state trajectory such that $f(x_t) = 0.6$ for all $t \geq 2$.
    When there is more friction, the app's policy results in $f(x_t) = 0.99$ for all $t \geq 2$.

    Let us first upper bound the expected number of interactions that occur in the first $N=100$ timesteps under less friction.
    For this upper bound, we can assume a frictionless setting, in which case the expected number of times the user engages the app is $1 + 0.6 (N-1) = 60.4$; here, the plus-1 reflects the guaranteed interaction at the first timestep.
    To lower bound the expected number of interactions that occur under more friction, we can consider the opposite extreme and suppose complete friction: in this case, the expected number of interactions is $\sum_{i=1}^N b^{n-1} = \tfrac{1-0.99^N}{1-0.99} = 63.4$.
    Thus, there is always a greater expected number of interactions in the first $N=100$ timesteps when there is more friction.

    Now suppose the user receives a reward of $R > 0$ if it interacts with the app and no reward otherwise.
    To upper bound the user's utility when there is less friction, assume a frictionless setting and observe that the user's utility is then $\sum_{t=1}^{\infty} (\gamma')^{t-1} \para{0.6 R} = \tfrac{0.6 R}{1-\gamma'} = 30$.
    To lower bound the user's utility when there is more friction, assume a full friction setting and observe the user's utility is  $\sum_{t=1}^\infty (\gamma' 0.99)^{n-1} R = \tfrac{R}{1-\gamma' 0.99} = 33.5$.
    Thus, the user's utility is also higher under more friction.
\end{proof}

\subsection{A Second Example of Friction}

 \begin{example}\label{ex:second_example}
Suppose a user is repeatedly choosing an app to use from an ecosystem of competitors, which we will treat as a mean-field.
We zoom in on the point-of-view of a specific app, which we will say is Instagram.

\paragraph{Original scenario:} Suppose the user's interest in Instagram can be represented as a numerical score that falls into one of three levels:
\begin{itemize}
    \item The user is entirely disinterested if their interest in Instagram is below a threshold of $0$. In this case, the user will stop interacting with Instagram.
    \item The user exercises healthy usage of Instagram if their interest in Instagram falls within the interval $[0, 6)$. In this case, the user has a 60\% chance of using Instagram on any given day.
    \item The user is addicted to Instagram if their interest in Instagram is larger $6$. In this case, the user has a 90\% chance of using Instagram on any given day.
\end{itemize}
\noindent
Further suppose Instagram creator has a $\gamma=0.95$ discount factor and a linear content landscape parameterized by $i \in [-6, 6]$ such that displaying content $i$ yields $R_i = 1 + i$ revenue for Instagram and a $E_i = -i$ effect on user interest.

\paragraph{Alternate scenario: }
Let us consider again the original scenario, but imagine Instagram's competitors begin serving increasingly addictive content.
To reflect that Instagram's users might get addicted to a competitor while they're not using Instagram, if a user does not open Instagram for a day, the probability they use it the next day is reduced by 50\%.
Let us also suppose the competitors' new use of addictive content has a side-effect of disinclining users from starting new sessions with them, increasing the retention rate on Instagram by some $w\%$ (can set $w=0$).
As a result, the competition has gotten stronger at increasing session length but also weaker at attracting new sessions.
The engagement probabilities in this alternate scenario can thus be summarized as
\begin{itemize}
    \item Interest below $0$: users interact with 0\% chance.
    \item Interest in $[0, 6)$: The user has a 60\% + w\% chance of staying on Instagram if they are already on it, but a 30\% chance of using Instagram if they are not already.
    \item Interest above $6$: The user has a 90\% + w\% chance of staying on Instagram if they are already on it, but a 45\% chance of using Instagram if they are not already.
\end{itemize}

If we compare Instagram's optimal policy and the resulting engagement frequency, we see that the alternate scenario (where competitors become more addictive) changes the optimal policy of Instagram to be more addictive and ends up increasing the number of user-app interactions.

\begin{restatable}{proposition}{altengs}
\label{prop:basic_3}
In \cref{ex:second_example}, for all $w \in [0\%, 5\%]$, the user has a higher interest in Instagram in the alternate scenario when competitors are addictive. 
Formally, let $x_t$ and $x'_t$ denote the user's interest in Instagram at time step $t$ in the original and alternate scenario respectively. Then, for all $t$, $x_t \leq x'_t$ and this inequality is strict for $t \geq 2$.
Moreover, for any time period, the expected number of days that the user spends on Instagram is strictly higher in the alternate scenario.
\end{restatable}
\begin{proof}
First, we can observe that both scenarios correspond to an instance of our model where the app creator's content decision problem is linear and the user demand function has two discontinuities.
By \cref{lemma:piecewiseconstant}, we know the optimal app policy in either scenario must be one of three possibilities:
\begin{enumerate}
    \item $\pi_1$: Maximally decrease user interest by showing content $\ite i1 = 6$. No interactions occur after.
    \item $\pi_2$: Show content $\ite it = 0$ at all timesteps $t$, keeping interest at $\ite xt = 0$.
    \item $\pi_3$: First show content $\ite i1 = -6$ then show $\ite it = 0$ thereafter, keeping interest at $\ite xt = 6$.
\end{enumerate}
The first policy only obtains revenue in the first timestep, giving a payoff of $7$ in both scenarios: $J(\pi_1) = J'(\pi_1) = 7$, where we use $J'$ to denote the payoff in the alternate scenario.
We can manually compute the payoff of the remaining two policies using \cref{theorem:cputime}.

In the original scenario, the payoffs of the second policy and third policy are, respectively,
\begin{align*}
    J(\pi_2) &= \sum_{t=1}^\infty \para{\gamma \cdot 0.6 +  \gamma^2\frac{1 - 0.6}{1 - \gamma (1 - 0.6)}\cdot 0.6}^{t-1} = 12.4, \\
    J(\pi_3) &= -6 + \sum_{t=1}^\infty \para{\gamma \cdot 0.9 +  \gamma^2\frac{1 - 0.9}{1 - \gamma (1 - 0.9)}\cdot 0.9}^{t-1} = 12.1,
\end{align*}
meaning that the second policy is optimal.
Thus, the user states---and equivalently user interest in Instagram---are $\tsv xt = 0$ for all $t \geq 2$ under the original scenario.

In the alternate scenario, the payoffs of the second policy and third policy are, respectively,
\begin{align*}
    J'(\pi_2) &= \sum_{t=1}^\infty \para{\gamma (0.6 + w) +  \gamma^2\frac{1 - (0.6+w)}{1 - \gamma (1 - 0.3)} \cdot 0.3}^{t-1} &&\approx \frac 1{0.1067-0.1418w},\\
    J'(\pi_3) &= -6 +  \sum_{t=1}^\infty \para{\gamma (0.9 + w) +  \gamma^2\frac{1 - (0.9+w)}{1 - \gamma (1 - 0.45)} \cdot 0.45}^{t-1}  &&\approx \frac{1}{0.05995 -  0.09948w} - 6.
\end{align*}
Since $\tfrac{1}{0.05995 -  0.09948w} - 6 > \tfrac 1{0.1067-0.1418w} > 7$ for all $w \in [0, 0.05]$, the third policy is optimal.
Thus, the user states---and equivalently user interest in Instagram---are $\tsv xt = 6$ for all $t \geq 2$ under the original scenario.

We now turn to the second claim.
Fix any timestep $t \geq 2$.
In Scenario 1, the probability that the user interacts with Instagram is 60\%, i.e. $\Pr(s_t = 1) = f(0) = 0.6$.
In Scenario 2, we will lower bound the probability of interaction by 80\%, i.e. $\Pr(s_t = 1) > 0.8$.
To show this, suppose to the contrary that there is a timestep $t' \geq 2$ where $\Pr(s_{t'} = 1) \leq 0.8$.
Letting $t'$ be the smallest such timestep, the observation that $\Pr(s_{t'-1} = 1) > 0.8$ directly leads to a contradiction:
$$\Pr(s_{t'}=1) = (0.9 + w) \cdot \Pr(s_{t'-1} = 1) + 0.45 \cdot \Pr(s_{t'-1} = 0) \geq 0.9 \cdot 0.8 + 0.45 \cdot 0.2 = 0.81,$$
confirming $\Pr(s_t = 1) > 0.8$ for all $t \geq 2$.
The second claim then follows by linearity of expectation.
\end{proof}

As an illustration of what this example demonstrates, suppose Instagram and a competitor app both implement a recommendation algorithm that prioritizes interesting but not addictive content.
If the competitor app suddenly improves the quality of its recommendation algorithm, standard models of competition tell us that Instagram may need to show fewer ads to become more competitive for user engagement.

However, suppose instead that the competitor app does not improve its algorithm.
Rather, it switches to recommending extremely addictive content that results in longer user sessions but also {repels} users, resulting in fewer new user sessions.
Our model provides a formal treatment of this scenario, which is not as easily captured by standard models of competition.
\cref{prop:basic_3} says that even though the competitor has become less effective at attracting users, by becoming more addictive, the competitor might still incentivize Instagram to show fewer ads to compete for user engagement.
In fact, \cref{prop:basic_3} says something stronger: Instagram may end up sacrificing so much of its profit to increase user engagement that it sees more usage than it did prior.
\end{example}

\subsection{Proof of \cref{theorem:simplealignment}}

We will first define the following notions for convenience.
\begin{definition}
\label{definition:demand_range}
    Given a simple policy $\pi$ and some $\delta > 0$, we say that the user state is $(\sat^*, \delta)$ bounded if there exists a constant $T$ such that for all $T' \geq T$, $\frac 1 {T'} \sum_{t=1}^{T'} 1[\ite{\sat}t \leq \sat^*] \geq \delta$ where $\ite{\sat}t$ is the user state  in the $t^{th}$ interaction.
\end{definition}

\begin{definition}\label{def:satrange}
	Given a simple policy $\pi$, and some $\delta > 0$, we say that user state \emph{lies in range $\satrange \subseteq \R$ for a $\delta$-fraction of time} if there exists a constant $T$ such that for all $T' \geq T$, $\frac 1 {T'} \sum_{t=1}^{T'} 1[\ite{\sat}t \in \satrange] \geq \delta$.
\end{definition}

The following is a more general statement of \cref{theorem:simplealignment}.
\begin{restatable}{theorem}{alignment}
    \label{theorem:alignment}
    If an app sees in its user an achievable payoff of at least $$\max_{\pi} J(\pi) \in \Omega\para{ \min_{\epsilon > 0}\left\{{\frac{\epsilon \cdot \log(1 / \gamma) }{\delta(\log\tilde f(x + \epsilon) - \log \tilde f(x))}+ \frac{1}{1-\tilde f(x)^{\delta / 2} \cdot \gamma}}\right\}},$$
    any app policy where user state is $(x, \delta)$ bounded is suboptimal.
    Here, $\Omega$ treats content attributes (i.e., $\bset{E_i}_{i \in \arms}, \bset{R_i}_{i \in \arms}$) as constants.
\end{restatable}

Due to the constants in the $\Omega$, this sufficient condition is non-trivial when there is a user state $x^*$ that results in very low user demand and is a small constant distance below the user state $x$.
Please see \eqref{eq:non-trivial} for explicit constants.

\begin{proof}[Proof of \cref{theorem:alignment}]
    In this proof, we consider all policies that result in users not having a user state of at least $x$ a constant $\delta$ fraction of the time and examine whether they can be optimal.
    Our proof is roughly as follows.
    We can immediately rule out policies that drive user state to such a low level that---even if the app maximizes revenue at every interaction---user demand is so low that the policies must be suboptimal (\cref{lemma:minimum_alignment}).
    All remaining policies must keep user state within a bounded range for a constant fraction of the time.
    Suppose to the contrary that one of these policies $\pi$ is optimal and that the optimal payoff---which $\pi$ must realize---is sufficiently large.
    Then the large payoff of $\pi$ can be further increased by first raising user state to a higher level (\cref{lemma:movingup}), contradicting the optimality of $\pi$.

\begin{lemma}
    \label{lemma:minimum_alignment}
    Suppose that for some finite level of user state $\sat$ and $\delta > 0$$$\max_\pi J(\pi) \in \Omega \para{\frac{\max_{i \in \arms} \EE{R_i}}{1 - \tilde f(x)^\delta \cdot \gamma} }.$$
Then, any policy where user state is less than $\sat$ a $\delta$-fraction of the time is suboptimal.
\end{lemma}
\begin{proof}
    Let $K \asseq \max_{i \in \arms} \EE{R_i}$, which exists by compactness of $\bset{\EE{R_i}}_{i \in \arms}$.
	Fix a policy where users are less than $\sat^*$ satisfied for at least a $\delta$-fraction of the time.
	By definition, there exists a constant $T$ such that for all $T' \geq T$, $\mathrm{Avg}(\curlybrackets{1[\ite{\sat}t \leq \sat^*]}_{t \leq T'}) \geq \delta$.
	We can thus write $J(\pi)$ as
	\begin{align*}
		J(\pi)
		&= \sum_{t=1}^{T-1} \para{\prod_{\tau \in [2, t]} \sdep(\ite{\sat}{\tau}) \gamma} r_t + \sum_{t=T}^{\infty} \para{\prod_{\tau \in [2, t]} \sdep(\ite{\sat}{\tau}) \gamma} r_t\\
&		\leq \sum_{t=1}^{T-1} \para{\prod_{\tau \in [2, t]} \sdep(\ite{\sat}{\tau}) \gamma} K + \sum_{t=T}^{\infty} \para{\prod_{\tau \in [2, t]} \sdep(\ite{\sat}{\tau}) \gamma} K.
	\end{align*}
	Let $S_t = \bset{\tau \in [2, t] : \ite{\sat}{\tau} \leq \sat}$ denote the subset of the first $t$ timesteps where the user state is no more than $\sat$.
	Since $\sdep \leq 1$ and $\sdep$ is monotonically non-decreasing, we can further simplify
	\begin{align*}
		J(\pi) & \leq   \sum_{t=1}^{T-1} \para{\prod_{\tau \in [2, t]} \sdep(\ite{\sat}{\tau}) \gamma} \variancebound  +  \sum_{t=T}^{\infty} \para{\prod_{\tau \in S_t} \sdep(\sat^*)} \gamma^{t-1} \variancebound\; & \text{($\sdep$ is monotonically non-decreasing)}                                                                                      \\
		       & =  \sum_{t=1}^{T-1} \para{\prod_{\tau \in [2, t]} \sdep(\ite{\sat}{\tau}) \gamma} \variancebound  +  \sum_{t=T}^{\infty} \sdep(\sat^*)^{|S_t|} \gamma^{t-1} \variancebound \;                        & \text{($\delta$-fraction user state assumption)}                                                                                    \\
		       & \leq   \sum_{t=1}^{T-1} \para{\prod_{\tau \in [2, t]} \sdep(\ite{\sat}{\tau}) \gamma} \variancebound  +  \sum_{t=T}^{\infty} \sdep(\sat^*)^{(\delta + o(1)) (t-1)} \gamma^{t-1} \variancebound \;        & \para{\frac 1t \sum_{\tau=1}^t 1[\ite{\sat}{\tau} \leq \sat^*] \geq \delta + o(1) } \\
		       & \leq \bigO{T\variancebound} + \sum_{t=1}^{\infty} \sdep(\sat^*)^{\delta (t-1)} \gamma^{t-1} \variancebound.
	\end{align*}
	Thus, if $$\max_\pi J(\pi) \in \Omega \para{T\variancebound + \frac{\variancebound}{1 - \tilde f(x)^\delta \cdot \gamma}} = \Omega \para{\frac{\variancebound}{1 - \tilde f(x)^\delta \cdot \gamma}},$$ any policy where users are less than $\sat^*$ satisfied a $\delta$ of the time cannot be optimal.
\end{proof}

Before proceeding to the next step of the proof, we first characterize the payoff of modifying a policy to first pre-emptively increase user state.
\begin{lemma}
	\label{lemma:ads}
    Given a policy $\pi$, let $\pi'$ denote the policy where the app repeatedly shows a content $i^* \in \cI$ where $\EE{\effect_{i^*}} > 0$ until reaching some fixed user state $\sattarget$ at which it switches to running the policy $\pi$.
	Letting $C_1, C_2$ denote positive constants, the payoff of $\pi'$ is
	\begin{equation}
		\label{eq:reputation_lb}
		J(\pi') \geq (C_1 \gamma)^{ \frac{\sattarget }{\EE{\effect_{i^*}}}} \cdot V_\pi(\sattarget) -  C_2 \frac{\sattarget}{\EE{\effect_{i^*}}}.
	\end{equation}
\end{lemma}

\begin{proof}
    Given a transcript of app-user interactions $H$, let $T_H \asseq \min_{t \in \naturals} \bset{\ite{x}t + \ite{e}t \geq \sattarget}$ denote the timestep at which user state first passes the level of $\sattarget$.
    We then define a function $g$ that edits transcripts by setting $g(H) = \emptyset$ if $T = \emptyset$ and  $g(H) = \bset{(1, 0, \sum_{t=1}^{T-1} \ite et, \emptyset)} + \bset{(\ite st, \ite rt, \ite et, \ite it )}_{t \geq T}$ otherwise.
    We then define the policy $\pi'$ as $\pi'(H) = i^*$ if $H = \emptyset$ otherwise $\pi'(H) = \pi(g(H))$.

    We can decompose the payoff of the policy $\pi'$ into 
    \begin{align*}
        J(\pi') &= {\sum_{t=1}^{T-1} \EE{\para{\prod_{\tau=2}^t \gamma \tilde f(\ite{x}\tau)} r_t} } + \EE{{\para{\prod_{\tau=2}^T \gamma \tilde f(\ite{x}\tau)}} {\sum_{t=T}^\infty \para{\prod_{\tau=T}^t \gamma \tilde f(\ite{x}\tau)}r_t}}\\
        &= \underbrace{\sum_{t=1}^{T-1} \EE{\para{\prod_{\tau=2}^t \gamma \tilde f(\ite{x}\tau)} r_t} }_A + \underbrace{\EE{\para{\prod_{\tau=2}^T \gamma \tilde f(\ite{x}\tau)}}}_{B} \EE{\sum_{t=T}^\infty \para{\prod_{\tau=T}^t \gamma \tilde f(\ite{x}\tau)}r_t}\\
        &= {\sum_{t=1}^{T-1} \EE{\para{\prod_{\tau=2}^t \gamma \tilde f(\ite{x}\tau)} r_t} } + \EE{\para{\prod_{\tau=2}^T \gamma \tilde f(\ite{x}\tau)}} {V_\pi({\ite{x}T})}\\
        &\geq {\sum_{t=1}^{T-1} \EE{\para{\prod_{\tau=2}^t \gamma \tilde f(\ite{x}\tau)} r_t} } + \EE{\para{\prod_{\tau=2}^T \gamma \tilde f(\ite{x}\tau)}} V_\pi({\sattarget}),
    \end{align*}
    where the inequality follows because the value function is monotonically non-decreasing.
    We can interpret $A$ as the cost incurred from showing content $i^*$ in the initial demand-building phase and $B$ as the discount penalty incurred from spending time building user demand.

        We first observe that $T$---the timestep at which user state reaches the level of $\sattarget$---corresponds to the stopping time of a random walk.
        In particular, we know that $T$ is a stopping time for the filtration $\curlybrackets{\mathcal{F}_t}$ generated by the realized user experiences $\bset{\ite{e}t}_{t \in \naturals}$.
	We also know that the difference sequence $$Y_t \asseq (\ite{\sat}t - \ite{\sat}0) - t \cdot \EE{\effect_{i^*}}$$ is a martingale with respect to $\mathcal{F}_t$ from $t = 1, \dots, T$.
	Moreover, since the variance of user experiences are bounded by assumption, we know that $$\EE{\abs{\ite{\sat}{t} - \ite{\sat}{t-1} - \EE{\effect_{i^*}}}} = \EE{\abs{\ite{e}{t} - \EE{\effect_{i^*}}}} < \infty.$$
	If we cap the stopping time $T$ by some constant $n \in \integers$, which we will denote by $T \wedge n \asseq \min\bset{T, n}$, then we can also observe that $T \wedge n$ is also a stopping time for $\mathcal{F}$.
	Since $n$ is finite and thus $\EE{T \wedge n} < \infty$, we appeal to the optional stopping theorem to observe that $$\ite{\sat}0 = \EE{\ite{\sat}{T \wedge n}} - \EE{(T \wedge n) \cdot \effect_{i^*}}.$$
	  We can upper bound how much we are expected to overshoot the target user state   $\sattarget$ by $$\EE{\ite{\sat}{T \wedge n}} \leq \sattarget,$$
        which directly implies that $$\EE{T \wedge n} \leq  \sattarget  / \EE{\effect_{i^*}}.$$
	The monotone convergence theorem then implies that the expected stopping time is $\EE{T} = \frac{\sattarget}{\EE{\effect_{i^*}}}$.

	We now lower bound the $A$ summand, which corresponds to the revenue realized prior to timestep $T$.
	If the content $i^*$ does not result in negative revenue for the app creator, i.e., $\EE{\reward_{i^*}} \geq 0$, we can bound $A \geq 0$.
	Otherwise, if the content does cause negative revenue, i.e. $\EE{\reward_{i^*}} < 0$, then we can lower bound $A \geq \EE{T} \EE{\reward_{i^*}}$.
	Thus, we can lower bound, $A \geq \frac{\sattarget }{\EE{\effect_{i^*}}} \min\bset{0, \EE{\reward_{i^*}}}$, meaning the constant $C_2$ in the lemma statement can therefore be understood as $C_2 = \min \bset{0, \EE{\reward_{i^*}}}$.

        To bound the term $B$, we note that the probability of user engagement is bounded away from zero by some constant.
        Formally, there exists some $C > 0$ such that for all $x \in \reals$, $f(x) \geq C$ and thus also a $C_1 > 0$ such that for all $x \in \reals$, $\tilde f(x) \geq C_1$.
 We can thus bound $\prod_{\tau=2}^T \gamma \tilde f(\ite{x}\tau) \geq (\gamma C_1)^T$.
	By Jensen's inequality, $${B} = \EE{(\gamma C_1)^{T}} \geq (\gamma C_1)^{\EE{T}} \geq (\gamma C_1)^{\frac{\sattarget }{\EE{\effect_{i^*}}}}.$$
    That $J(\pi) \geq 0$ concludes our proof.
\end{proof}

We can always improve policies that keep user state within a small range and attain a high payoff.
\begin{lemma}
	\label{lemma:movingup}
	Consider any closed range of user states $\satrange$ (where $\under{\satrange} = \min \satrange$ and $\ov{\satrange} = \max \satrange$) and any policy $\pi$ with $J(\pi) > 0$, where user state lies within $\satrange$ for a $\delta > 0$ fraction of the time.
    The policy $\pi$ is suboptimal if
	\begin{align*}
		J(\pi) \in \Omega \para{\min_{\satimp \geq \max \satrange} 
			(\satimp - \under{\satrange}) \para{C + \frac{\log(1/\gamma)}{\delta \log(\tilde f(\satimp) / \tilde f(\ov{\satrange}))}}}.
	\end{align*}
\end{lemma}
\begin{proof}
	We will improve upon policy $\pi$ by defining a new policy $\pi'$.
 Fix any $\satimp \geq \max \satrange$.
	By assumption, there exists an content $i^*$ with positive user effect: $\EE{\effect_{i^*}} > 0$.
	The policy will show this content $i^*$ until user state is raised to a target level of user state of $\sattarget = \satimp - \under{\satrange}$.
	We will then execute the original policy $\pi$ as if we had never raised user state to $\sattarget$.
        
        Formally, given a transcript $H$, let $T_H \asseq \min_{t \in \naturals} \bset{\ite{x}t + \ite{e}t \geq \sattarget}$ denote the timestep at which user state first passes the level of $\sattarget$.
        We then define the transcript editing function $g$ where $g(H) = \emptyset$ if $T = \emptyset$ and otherwise $g(H) = \bset{(1, 0, 0, \emptyset)} + \bset{(\ite st, \ite rt, \ite et, \ite it)}_{t \geq T}$. 
        We then define the policy $\pi'$ as $\pi'(H) = i^*$ if $H = \emptyset$ otherwise $\pi'(H) = \pi(g(H))$.
        For the remainder of this proof, we use $\hat{\sat}_{\text{target}} = \ite xT$ to denote the user state that policy $\pi'$ achieves before switching to simulating $\pi$; note that $\hat{\sat}_{\text{target}} \geq \sattarget$.
        By \cref{lemma:ads}, we can lower bound the payoff of $\pi'$ by
	\begin{align*}
		\underbrace{C_1 \frac{\satimp - \under{\satrange}}{\EE{\effect_{i^*}}} }_{A}
		+ \underbrace{\para{K \gamma}^{\frac{\satimp - \under{\satrange}}{\EE{\effect_{i^*}}}}}_{B} \underbrace{ \EEs{\bset{(\ite st, \ite rt, \ite et, \ite it)}_t \sim \pi}{\sum_{t=2}^\infty \parantheses{\prod_{\tau=2}^t \sdep(\ite{\sat}{\tau} + \sattarget) \gamma }  r_t}}_{C},
	\end{align*}
        where 
        $K > 0$ is a constant that lower bounds the range of $\tilde f$.
        We now turn to lower bounding $C$.
        We introduce a sequence of constants $\bset{\ite{C}{\tau}}_{\tau \in \naturals}$, where we guarantee that $\ite{C}{\tau} \geq 1$ for all $\tau \in \naturals$.
        We can observe that, for all $k \in \naturals$,
        \begin{align*}
         &   \EEs{\bset{(\ite st, \ite rt, \ite et, \ite it)}_t \sim \pi}{\sum_{t=1}^\infty  \para{\prod_{\tau=2}^t \ite{C}{\tau} \sdep(\ite{\sat}{\tau}) \gamma} r_t}
          \\&  = \EEs{\bset{(\ite st, \ite rt, \ite et, \ite it)}_t \sim \pi}{\sum_{t=1}^{k-1}  \para{\prod_{\tau=2}^t \ite{C}{\tau} \sdep(\ite{\sat}{\tau}) \gamma} r_t + \para{\prod_{\tau=2}^{k-1} \ite{C}{\tau} \sdep(\ite{\sat}{\tau}) \gamma} C_k J_{\ite{x}k}(\pi)}
        \end{align*}
        is non-decreasing in $C_k$. %
        Now, we let $\ite{C}{\tau} = \frac{\tilde f(\ite{x}\tau + \hat{\sat}_{\text{target}} )}{\tilde f(\ite{x}\tau)}$ and will resort to two lower bounds.
        Since $f$ is monotonically non-decreasing, we immediately know that $\ite{C}{\tau} \geq 1$ for all $\tau \in \naturals$.
        Using the shorthand $\eta = \tilde{f}(\satimp) / \tilde f(\ov{\satrange})$ and recalling that $\tilde f$ is monotonically non-decreasing, we will also use the following lower bound for every $x \in \satrange$,
        \begin{align*}
            \tilde \engagement(\sat + \ite{x}{T-1} + \ite{e}{T-1})
            & \geq \tilde \engagement(\ite{\sat}{\tau} + \sattarget) & \text{($\ite{x}{T-1} + \ite{e}{T-1} \geq \satimp$)} \\
            & \geq \tilde \engagement(\satimp) & \text{($x \geq \under{\satrange}$)} \\
            & \geq \eta \tilde \engagement(\sat). & \text{($\ov{\satrange} \geq x$)}
        \end{align*}
        Recall that, by assumption, for every $\xi > 0$, there is a timestep $T'$ where all $k \geq T'$ satisfy ${\frac 1 k \sum_{t=1}^k 1[\ite{\sat}t \in \satrange]} \geq \delta$.
        Putting everything together, we conclude that for all $k \geq T'$,
        \begin{align*}
            \prod_{\tau=1}^k \ite{C}\tau& = \para{\prod_{\tau=1}^k (\ite{C}\tau)^{1[\ite{C}\tau \in \satrange]}} \para{\prod_{\tau=1}^k (\ite{C}\tau)^{1[\ite{C}\tau \notin \satrange]}}\\
          &  \geq \para{\prod_{\tau=1}^k \eta^{1[\ite{C}\tau \in \satrange]}} \para{\prod_{\tau=1}^k 1^{1[\ite{C}{\tau} \notin \satrange]}}
            \geq \eta^{k \delta}.
        \end{align*}
        We thus have that for any $k \geq T'$, taking an expectation over the trajectory of policy $\pi$,
        \begin{align*}
  C
  &          =
            \EEs{\bset{(\ite st, \ite rt, \ite et, \ite it)}_t}{\sum_{t=1}^\infty  \para{\prod_{\tau=2}^t \ite{C}{\tau} \sdep(\ite{\sat}{\tau}) \gamma} r_t}
\\&            \geq
            \EEs{\bset{(\ite st, \ite rt, \ite et, \ite it)}_t}{\sum_{t=1}^{k-1}  \para{\prod_{\tau=2}^t \sdep(\ite{\sat}{\tau}) \gamma} r_t}
            + \eta^{\delta k} \EEs{\bset{(\ite st, \ite rt, \ite et, \ite it)}_t}{\sum_{t=k}^\infty  \para{\prod_{\tau=2}^t \sdep(\ite{\sat}{\tau}) \gamma} r_t}
\\&            \geq
\eta^{\delta k} \EEs{\bset{(\ite st, \ite rt, \ite et, \ite it)}_t}{\sum_{t=1}^\infty  \para{\prod_{\tau=2}^t \sdep(\ite{\sat}{\tau}) \gamma} r_t} - k \max_{i \in \arms} \EE{R_i}
\\&            \geq
\eta^{\delta k} J(\pi) - k \max_{i \in \arms} \EE{R_i}.
        \end{align*}
        Simplifying our initial equality, we have
        \begin{align*}
            J(\pi') - J(\pi) & \geq 	{C_1 \frac{\satimp - \under{\satrange}}{\EE{\effect_{i^*}}} }
		+ \para{{\para{C_2 \gamma}^{\frac{\satimp - \under{\satrange}}{\EE{\effect_{i^*}}}}} \eta^{\delta k} - 1} J(\pi) - {\para{K \gamma}^{\parantheses{1 + \frac{\satimp - \under{\satrange}}{\EE{\effect_{i^*}}}}}} k \max_{i \in \arms} \EE{R_i} \\ & \geq C_1 \frac{\satimp - \under{\satrange}}{\EE{\effect_{i^*}}}
		+ \para{{\para{K \gamma}^{\frac{\satimp - \under{\satrange}}{\EE{\effect_{i^*}}}}} \eta^{\delta k} - 1} J(\pi) - k \max_{i \in \arms} \EE{R_i}
        \end{align*}
        where the second inequality uses the fact that $B \leq 1$.
        Choosing $$k \assignequals \frac1{ \delta \log(\eta)} \log\para{2 \para{K \gamma}^{-\frac{\satimp - \under{\satrange}}{\EE{\effect_{i^*}}}}} + T' = \frac 1 \delta \frac{\satimp - \under{\satrange}}{\EE{\effect_{i^*}}} \log(1 / K \gamma) + \frac 1 \delta \log(2) + T',$$
        we have for some appropriate constant $C$
        \begin{align*}
            J(\pi') - J(\pi)& \geq 	{C_1 \frac{\satimp - \under{\satrange}}{\EE{\effect_{i^*}}}}
		+  J(\pi) \\
            &+ \para{ \frac 1 { \delta \log(1/\eta)}  \frac{\satimp - \under{\satrange}}{\EE{\effect_{i^*}}} \log(1 / K \gamma) + \frac 1 \delta \log(2) + T'}  \max_{i \in \arms} \EE{R_i} \\
            & = J(\pi) + C (\satimp - \under{\satrange}) \para{C + \frac{\log(1/K\gamma)}{\delta \log(\eta)}}.
        \end{align*}
\end{proof}

The following fact says that we can combine the results of \cref{lemma:minimum_alignment} and \cref{lemma:ads} to assert our claim.
\begin{fact}
Every policy that is less than $x$ satisfied a $\delta$ fraction of the time is either within $[x', x]$ satisfied a $\delta/2$ fraction of the time or less than $x'$ satisfied a $\delta/2$ fraction of the time.
\end{fact}
\begin{proof}
Fix any two choices of user states $x$ and $x'$ where $x > x'$ and $\delta > 0$.
Given a policy and $k \in \naturals$, let $T_{1,k}$ denote the fraction of the first $k$ timesteps where user state is at most $x$, let $T_{2,k}$ denote the fraction of the first $k$ timesteps where user state is less than $x'$, and let $T_{3,k}$ denote the fraction of the first $k$ timesteps where user state is within $[x', x]$.
We have that for all $k \in \naturals$, $T_{1,k} = T_{2, k} + T_{3, k}$.
Now, consider the set $S$ of all policies where user state is less than $x$ a $\delta$ fraction of the time and the set $S'$ of policies where user state lies in $[x', x]$ a $\delta/2$ fraction of the time.
That is, where after some constant $T'$, for all $k \geq T'$, $T_{1, k} \geq \delta$ and $T_{3, k} \leq  \delta/2$ and thus $T_{2, k} \geq \delta / 2$.
\end{proof}

We thus have that, for any choice of $\delta$ and $x > x'$, if
\begin{align}
\label{eq:non-trivial}
J(\pi) \in \Omega \para{\min_{\satimp \geq x} 
    (C + \satimp - x') \para{\frac{2 \log(1/K\gamma)}{\delta \log(\tilde f(\satimp) / \tilde f(x))}} + \frac{\max_{i \in \arms} \EE{R_i}}{1 - \tilde f(x')^{\delta/2} \cdot \gamma} },
\end{align}
then all policies where users are less than $x$ satisfied a $\delta$ fraction of the time are suboptimal.
\end{proof}

\end{document}